%% file: main.tex
\documentclass{article}

\PassOptionsToPackage{numbers,compress}{natbib}

\usepackage[main, final]{neurips_2026}

\usepackage[utf8]{inputenc} %
\usepackage[T1]{fontenc}    %
\usepackage{hyperref}       %
\usepackage{url}            %
\usepackage{booktabs}       %
\usepackage{amsfonts}       %
\usepackage{amsmath,amssymb,amsthm,mathtools}
\usepackage{nicefrac}       %
\usepackage{microtype}      %
\usepackage{xcolor}         %
\usepackage{float}          %
\usepackage{enumitem}
\usepackage{wrapfig}

\definecolor{graphinput}{RGB}{25,25,112}    %
\definecolor{graphinter}{RGB}{34,139,34}    %
\definecolor{graphoutput}{RGB}{178,34,52}   %
\definecolor{graphscalar}{RGB}{253,188,66}  %

\usepackage{algorithm}
\usepackage{algorithmicx}
\usepackage{algpseudocode}
\usepackage{multirow}
\usepackage{graphicx}
\usepackage{tikz}
\usetikzlibrary{arrows.meta, positioning, calc, fit, backgrounds}

\newtheorem{definition}{Definition}[section]
\newtheorem{proposition}[definition]{Proposition}
\newtheorem{theorem}[definition]{Theorem}
\newtheorem{corollary}[definition]{Corollary}
\newtheorem{lemma}[definition]{Lemma}
\theoremstyle{remark}

\newcommand{\R}{\mathbb{R}}
\newcommand{\Tr}{\mathsf{T}}

\newcommand{\rowsum}{\operatorname{rowsum}}

\newcommand{\softmax}{\operatorname{softmax}}

\newcommand{\td}[1]{\widetilde{d}(#1)}

\algrenewcommand\algorithmicrequire{\textbf{Input:}}
\algrenewcommand\algorithmicensure{\textbf{Output:}}
\algrenewcommand\algorithmiccomment[1]
{\hfill\textcolor{gray!70}{\scriptsize$\triangleright$~#1}}

\definecolor{fbInk}      {HTML}{0F172A} %
\definecolor{fbMuted}    {HTML}{334155} %

\definecolor{fbBlue}     {HTML}{174EA6} %
\definecolor{fbRed}      {HTML}{B42318} %

\definecolor{hbmFill}    {HTML}{CBD5E1} %
\definecolor{hbmStroke}  {HTML}{475569} %
\definecolor{hbmInk}     {HTML}{1E293B} %

\definecolor{sramFill}   {HTML}{B7E0C5} %
\definecolor{sramStroke} {HTML}{2F855A} %
\definecolor{sramInk}    {HTML}{14532D} %

\definecolor{roleFixed}  {HTML}{1F4FA3}
\definecolor{roleStream} {HTML}{C46A1C}

\algrenewcommand\algorithmiccomment[1]{\hfill\textcolor{gray!70}{\scriptsize$\triangleright$~#1}}

\title{FlashBoB: I/O-Efficient Exact Backward-over-Backward for Softmax Attention}

\author{%
  Anthony Givans, Michael Crawshaw, Mingrui Liu\thanks{Corresponding Author. The implementation is at \url{https://github.com/MingruiLiu-ML-Lab/FlashBoB}} \\
  Department of Computer Science\\
  George Mason University, Fairfax, VA 22030 \\
  \texttt{agivans2@gmu.edu, mcrawsha@gmu.edu, mingruil@gmu.edu} \\
}

\begin{document}

\maketitle

\begin{abstract}
Transformer models built on the attention mechanism have become a central building block in modern deep learning, yet softmax attention remains a major bottleneck for long-context workloads. While FlashAttention makes the forward and first backward passes I/O-efficient, it does not support \emph{backward-over-backward} (BoB), which enables exact differentiation through the backward pass for applications such as second-order optimization, test-time training, gradient-written memory, and meta-learning. Existing BoB implementations either materialize large intermediate tensors or exhaust GPU memory at long sequence lengths.

We present \textsc{FlashBoB}, an exact, I/O-efficient algorithm for BoB in softmax attention that keeps computation within on-chip tiles and avoids all $N \times N$ intermediate tensors, where $N$ is the sequence length. The key insight is a hierarchical affine structure in the softmax double backward: two row-wise scalars determine all outputs through affine transformations. This yields a two-pass schedule with bounded on-chip static random-access memory (SRAM) usage and minimal off-chip high-bandwidth memory (HBM) traffic. \textsc{FlashBoB} achieves $\Theta(N^2 d^2/M)$ HBM traffic ($d$ is the head dimension and $M$ is the on-chip SRAM size) and, within the standard FlashAttention-style score-recomputation model, matches the inherited large-cache lower bound for exact forward attention. Empirically, in an isolated attention-layer benchmark on a single A100 80GB GPU, it computes exact attention BoB at $N=262\text{K}$, where prior PyTorch exact baselines exhaust memory by $N=16\text{K}$, and it is up to $9.4\times$ faster than FlashBack. Across complete second-order steps on GPT-2, Pythia, and five native GQA and MoE models, \textsc{FlashBoB} is $1.6$--$3.8\times$ faster than the fastest exact baseline that fits on chip, supports $2$--$8\times$ longer contexts, and cuts 1B-token Sophia-H training time and peak memory on FineWeb by $2.2\times$. These results make exact second-order attention practical at long-context sequence lengths where prior implementations cannot run efficiently.
\end{abstract}

\vspace*{-0.1in}

\section{Introduction}
\vspace*{-0.05in}

\label{sec:intro}

Transformers are the dominant architecture for sequence modeling~\citep{vaswani2017attention}, and attention is their core operation. The cost of attention grows quadratically with sequence length and, at scale, is dominated by data movement between high-bandwidth memory (HBM) and on-chip SRAM. Two lines of work address this at long context. One modifies the operator: recurrent memories, sparse patterns, low-rank or kernelized attention, and state-space models reduce or replace the full pairwise interaction~\citep{dai2019transformerxl,beltagy2020longformer,zaheer2020bigbird,kitaev2020reformer,wang2020linformer,choromanski2021rethinking,katharopoulos2020transformers,poli2023hyena,peng2023rwkv,gu2024mamba}. The other, FlashAttention~\citep{dao2022flashattention,dao2023flashattention2,shah2024flashattention3,zadouri2026flashattention4}, preserves exact softmax attention and redesigns the memory schedule to make it I/O-efficient: tiles are recomputed from saved statistics, and the full attention matrix is never materialized in HBM. The key design principle is that for exact attention, the memory schedule is part of the algorithm.

This design principle has so far been limited to the forward and first backward passes of softmax attention. A growing class of methods places gradients inside the model's computation and must differentiate through the attention backward pass itself; we call this primitive \emph{backward-over-backward} (BoB). Such methods arise across modern training paradigms, including second-order optimization via Hessian-vector products for curvature estimation~\citep{liu2024sophia,pearlmutter1994fast,martens2010deep,martens2015optimizing,gupta2018shampoo,yao2021adahessian}, test-time training that differentiates through inner-loop updates~\citep{tandon2025endtoend}, gradient-written memory that stores context via per-example gradient descent~\citep{kuratov2026gradmem}, and meta-learning that differentiates through inner optimization~\citep{finn2017maml,nichol2018firstorder,rajeswaran2019meta}. At the kernel level, all of these rely on the same BoB primitive, yet at long context none of the existing implementations scale. Generic automatic differentiation materializes intermediate tensors of quadratic (sequence-by-sequence) size and exhausts memory at sequence length $N = 8{,}192$ for an isolated GPT-2 Small attention layer (batch 4, 12 heads, $d=64$, causal, BF16 inputs with FP32 reductions, no GQA/MQA) on a single A100 80GB GPU. GradMem~\citep{kuratov2026gradmem} ships custom Hessian-vector product kernels that fuse parts of the second backward and avoid the worst materialization, doubling the feasible context to $N = 8{,}192$ on the same hardware before exhausting memory at $N = 16{,}384$. The closest prior non-materializing kernel, FlashBack~\citep{engstrom2024flashback}, fuses the entire BoB into a single GPU kernel parallelized over rows, with several outputs accumulated by concurrent threads writing to the same memory locations through synchronized hardware updates that serialize and stall under contention. Crucially, none of these approaches extend FlashAttention's memory discipline to backward-over-backward.

Extending FlashAttention-style scheduling to backward-over-backward faces two key obstacles. First, the computation requires combining information in both row-wise and column-wise directions, which prevents a single consistent pass over the data. Second, the softmax double backward introduces global dependencies that must be resolved before downstream results can be finalized. Together, these constraints break the simple, single-pass execution pattern that keeps intermediate data in fast on-chip memory: a naive implementation either requires multiple passes or stores large intermediate tensors. Both options sacrifice the memory and I/O efficiency that makes FlashAttention effective.

We resolve these challenges with \textsc{FlashBoB}, an exact, I/O-efficient algorithm for backward-over-backward in softmax attention. The key idea is to expose a hierarchical affine structure in the softmax double backward, which allows the computation to be reorganized into a small number of passes that avoid materializing $N \times N$ tensors and keep intermediate data in fast on-chip memory. The resulting schedule closely follows the FlashAttention execution model, incurring only minimal overhead beyond the first backward pass while preserving efficient memory access patterns. As a result, \textsc{FlashBoB} scales to sequence lengths far beyond prior exact implementations~\citep{kuratov2026gradmem,engstrom2024flashback} and removes the attention-BoB memory bottleneck for methods that require exact second derivatives through softmax attention. Theoretically, it matches a large-cache I/O lower bound under the standard
FlashAttention-style score-recomputation model, where score and probability tiles are recomputed from saved row statistics rather than stored as \(N\times N\) HBM tensors. Our contributions can be summarized as follows.

\begin{itemize}[leftmargin=*, itemsep=1pt, topsep=2pt]

\item \textbf{FlashBoB algorithm.}
We introduce an exact two-pass algorithm for softmax-attention BoB that
preserves the FlashAttention execution model: no $N \times N$ HBM
intermediates, bounded on-chip workspace, and one write per finalized
output block. The key identity is a hierarchical affine structure in the
softmax double backward (Proposition~\ref{prop:affine} and
Figure~\ref{fig:overview}) that collapses the natural four-pass schedule
to two passes.

\item \textbf{I/O analysis and optimality.}
We establish a $\Theta(N^2 d^2 / M)$ HBM upper bound for \textsc{FlashBoB} (Theorem~\ref{thm:io}) and a matching large-cache lower bound under the standard FlashAttention-style score-recomputation model (Theorem~\ref{thm:bob-lb}). The lower bound is inherited from exact forward attention by embedding it inside a special BoB instance.

\item \textbf{Kernel and model-level impact.}
On an isolated attention layer (A100 80GB), \textsc{FlashBoB} reaches $N=262{,}144$, where PyTorch materialization fails at $N=8{,}192$ and GradMem HVP baselines fail by $N=16{,}384$. It is up to $10\times$ faster and $28\times$ more memory-efficient than materializing PyTorch, $2.1$--$2.4\times$ faster than HVP baselines at $d=128$ (up to $3.8\times$ at $d=64$), and up to $9.4\times$ faster than FlashBack. Across ten complete models (GPT-2, Pythia, and five native GQA/MoE architectures), it yields $1.6$--$3.8\times$ speedups and fits $2$--$8\times$ longer sequences. On 1B-token FineWeb training with Sophia-H, it cuts wall time and peak memory by $2.2\times$ ($1.22$--$1.70\times$ in a matched hybrid baseline isolating the BoB kernel) while matching the loss trajectory.
\end{itemize}

\vspace*{-0.1in}
\section{Related Work}
\label{sec:related}
\vspace*{-0.1in}

\textbf{GPU memory hierarchy and kernel fusion.}
The performance principle behind FlashAttention is part of a broader
systems literature: communication between memory levels can dominate
arithmetic, recomputation can trade extra work for lower memory
pressure, and combining multiple operations into one GPU kernel can
avoid expensive reads and writes~\citep{hong1981redblue,ballard2011minimizing,chen2016training,jain2020checkmate,kirisame2021dtr,ragankelley2013halide,chen2018tvm}.
GPU programming systems such as CUDA and Triton expose the memory
hierarchy well enough to keep small working sets in fast on-chip memory
while using specialized hardware for matrix multiplication~\citep{nickolls2008cuda,tillet2019triton,markidis2018tensorcore}.
\textsc{FlashBoB} applies this principle to a higher-order derivative:
rather than changing the attention operator or introducing an
approximation, it changes the execution schedule so that the exact BoB
computation avoids unnecessary HBM traffic.

\textbf{Attention I/O and exact streaming.}
Long-context attention has been pursued along two lines. Approximate
or structured-attention methods change the operator: recurrent
memories, sparse patterns, low-rank projections, kernelized attention,
and state-space models reduce the explicit quadratic interaction by
modifying the computation~\citep{dai2019transformerxl,beltagy2020longformer,zaheer2020bigbird,kitaev2020reformer,wang2020linformer,choromanski2021rethinking,katharopoulos2020transformers,poli2023hyena,peng2023rwkv,gu2024mamba}.
Building on online softmax normalization and exact memory-efficient
attention~\citep{milakov2018online,rabe2021selfattention}, the
FlashAttention family~\citep{dao2022flashattention,dao2023flashattention2,shah2024flashattention3,zadouri2026flashattention4}
keeps exact softmax attention and changes the memory schedule: it
recomputes attention blocks from saved softmax statistics so the full
$N \times N$ attention matrix is never written to HBM. Lower-bound work
formalizes these gains in two-level memory and red-blue pebble
models~\citep{hong1981redblue,ballard2011minimizing,saha2024ioattention,li2024finegrainedbackward}.
\textsc{FlashBoB} belongs to this second line: it preserves exact
softmax attention and extends FlashAttention's streaming memory
schedule to the second derivative.

\textbf{Implementations of attention backward-over-backward.}
Existing implementations of exact attention BoB, and more general
higher-order derivative tooling~\citep{dangel2020backpack}, fall into
three broad categories. Generic reverse-mode automatic differentiation~\citep{baydin2018automatic}
through a PyTorch-style tensor graph~\citep{paszke2019pytorch} provides a
natural correctness oracle, but stores quadratic-size intermediate
tensors and therefore scales poorly at long context. GradMem~\citep{kuratov2026gradmem} provides custom
Hessian-vector product kernels (\textsc{hvp-m}, \textsc{hvp-s}) that
avoid the worst materialization, but they do not follow the
FlashAttention-style streaming schedule. FlashBack~\citep{engstrom2024flashback}
is the closest prior non-materializing kernel we are aware of: it fuses
the BoB computation into a single GPU kernel, but relies on shared
updates to global memory for some outputs and repeatedly recomputes
attention blocks. In contrast, \textsc{FlashBoB} separates the work into two streaming
passes that match the natural structure of the BoB outputs, avoid shared global updates, and expose affine fusion that a single-kernel design cannot exploit without additional synchronization or recomputation.

\textbf{Applications of attention backward-over-backward.}
Backward-over-backward through attention appears in several long-context training settings. Second-order methods use Hessian-vector products to obtain curvature information without forming the Hessian explicitly, from classical Hessian-free and preconditioned optimization to modern stochastic curvature estimators~\citep{pearlmutter1994fast,martens2010deep,martens2015optimizing,gupta2018shampoo,yao2021adahessian,liu2024sophia}.
Meta-learning differentiates through inner gradient
steps~\citep{finn2017maml,nichol2018firstorder,rajeswaran2019meta}, while recent test-time training and gradient-written memory methods differentiate through inner-loop adaptation or per-example writing procedures~\citep{tandon2025endtoend,kuratov2026gradmem}.
These settings motivate exact BoB implementations for softmax attention.

\vspace*{-0.1in}
\section{Background and Preliminaries}
\label{sec:prelim}
\vspace*{-0.1in}

\textbf{Single-Head Attention in Transformer.} For a single attention head in a Transformer, the inputs
$Q, K, V \in \R^{N \times d}$ are the query, key, and value matrices
across $N$ tokens of the sequence~\citep{vaswani2017attention}. With
$\tau = 1/\sqrt d$, the forward pass is
\begin{equation}
\label{eq:attn-forward}
S = \tau Q K^\Tr + B_{\text{mask}}, \qquad
P = \softmax(S), \qquad
O = P V,
\end{equation}
where $B_{\text{mask}} \in \{0, -\infty\}^{N \times N}$ encodes any
attention mask (e.g., causal, sliding-window, or block patterns), and the
softmax is applied row-wise. For row $i \in [N]$, let
$\mathcal{V}_i = \{ j \in [N] : (B_{\text{mask}})_{ij} \neq -\infty \}$
denote the set of valid columns; all row sums below are over $\mathcal{V}_i$.
Lowercase letters denote row slices:
$q_i, k_i, v_i, o_i, do_i \in \R^d$, where $do_i$ denotes the gradient
of the loss with respect to $o_i$ and $i=1,\ldots,N$. Table~\ref{tab:notation} in
Appendix~\ref{app:notation} collects the full notation, including the
BoB-specific objects introduced in \S\ref{sec:bob-def}. For simplicity, we consider only single-head attention in our derivations, and the analysis extends naturally to multiple heads and to grouped-query or multi-query attention, where $H_Q/H_{KV}$ query heads share one key-value head (\S\ref{sec:benchmarks}). 

\textbf{Two-level GPU Hierarchy.} We work in the standard two-level memory
model~\citep{hong1981redblue,ballard2011minimizing}: HBM is large and
slow, SRAM is small and fast, and I/O counts elements moved between
them. SRAM has capacity $M$ elements. We process attention using a tiled
(blocked) algorithm: let $B_r$ and $B_c$ denote the row and column tile
sizes, respectively, so that blocks of size $B_r \times d$ and $B_c \times d$
from $Q, K, V$ (and associated workspace) are materialized in SRAM during
computation. We work in the large-cache regime $M = \Omega(d^2)$ so that
these tiles and the necessary local workspaces fit on chip~\citep{dao2022flashattention,saha2024ioattention}.
The goal throughout the paper is to avoid HBM traffic for $N \times N$
intermediates while keeping the arithmetic exact.

\textbf{Flash Attention.}
FlashAttention~\citep{dao2022flashattention,dao2023flashattention2}
processes attention using a tiled algorithm, partitioning rows
into blocks of size $B_r$ and columns into blocks of size $B_c$. The
forward pass stores only the row-wise log-normalizer
$L_i = \log\!\sum_{j \in \mathcal{V}_i} \exp(S_{ij})$ together with the
output $O$, where $\mathcal{V}_i$ denotes the set of valid columns for
row $i$. Later passes recompute probability tiles
$P_{ij} = \exp(S_{ij} - L_i)$ inside SRAM, so the full $N \times N$
matrix $P$ is never stored in HBM.

Given $dO=\partial \mathcal{L}/\partial O$, the attention backward computes $dP,dV,dS,dQ,dK$ as
\begin{equation}
\label{eq:first-bwd}
dP = dO\, V^\Tr, \qquad
dV = P^\Tr dO, \qquad
dS_{ij} = P_{ij}(dP_{ij} - D_i),
\end{equation}
\begin{equation}
\label{eq:first-bwd-2}
dQ = \tau\, dS\, K, \qquad
dK = \tau\, dS^\Tr Q,
\end{equation}
where the row-wise scalar
$D_i = \sum_{j \in \mathcal{V}_i} P_{ij}\, dP_{ij}
    = do_i^\Tr o_i$
is the first-backward row summary.

FlashAttention keeps only the row statistics $L$ (from the forward) and $D$ (from the first backward), together with $O$; probability tiles are recomputed from $L$ inside SRAM, and $P,dP,dS$ are never stored as full $N\times N$ matrices. We refer to this schedule as the \emph{FlashAttention execution model}: it maintains bounded $O(M)$ on-chip working sets, avoids global atomic adds on outputs, writes no $N\times N$ intermediates to HBM, and recomputes probability tiles from saved row state. \textsc{FlashBoB} reuses this state and adds only two BoB row-state vectors, \((\alpha,E)\), which are introduced in Section~\ref{sec:flashbob}.

\textbf{Motivating applications of backward-over-backward.}
The forward and first-backward kernels are insufficient for methods that
differentiate through attention-containing gradient steps, including
Sophia-H Hessian-vector products~\citep{liu2024sophia,pearlmutter1994fast},
end-to-end test-time training~\citep{tandon2025endtoend}, and
gradient-written memory~\citep{kuratov2026gradmem}. At the attention
boundary, all of these supply upstream matrices
$U_Q,U_K,U_V \in \mathbb{R}^{N \times d}$ at the first-backward outputs
$(dQ,dK,dV)$. We define this BoB primitive formally in \S\ref{sec:bob-def}
and use Sophia-H and GradMem-derived kernels in the experiments.

\vspace*{-0.1in}
\section{FlashBoB: Algorithm and Analysis}
\label{sec:flashbob}
\vspace*{-0.1in}

This section develops the \textsc{FlashBoB} schedule in five steps:
we first define the BoB primitive (\S\ref{sec:bob-def}); derive
closed-form identities (\S\ref{sec:affine}); %
present two row-scalar reductions that reduce the natural
FlashAttention-style tiled schedule from four sweeps to two passes (\S\ref{sec:affine}); describe the resulting algorithm (\S\ref{sec:algo}); and finally analyze its I/O complexity together with a matching score-recomputation lower bound in the large-cache regime (\S\ref{sec:io}).

\vspace*{-0.05in}
\subsection{Backward-over-backward}
\label{sec:bob-def}
\vspace*{-0.05in}

Let $g(Q, K, V, dO) = (dQ, dK, dV)$ denote the first-backward map of
\eqref{eq:first-bwd}--\eqref{eq:first-bwd-2}. Given matrices $U_Q, U_K, U_V \in \R^{N \times d}$, each
application in \S\ref{sec:prelim} requires the
vector--Jacobian product of $g$ with $(U_Q, U_K, U_V)$, which is obtained by reverse-mode automatic differentiation applied to the scalar
\begin{equation}
\label{eq:bob-obj}
\Phi
 = \langle dQ, U_Q\rangle
 + \langle dK, U_K\rangle
 + \langle dV, U_V\rangle,
\end{equation}
where $(dQ,dK,dV)=g(Q,K,V,dO)$ and
$\langle A,B\rangle=\operatorname{tr}(A^\Tr B)$ denotes the Frobenius inner product.
Backward-over-backward (BoB) computes the four resulting gradients
\begin{equation}
\label{eq:bob-defn}
\bigl(\td{Q},\, \td{K},\, \td{V},\, \td{dO}\bigr)
 \;:=\; \nabla_{Q, K, V, dO}\, \Phi.
\end{equation}

For any intermediate variable $X$ in the first-backward graph,
$\td{X} := \partial \Phi / \partial X$ is its \emph{BoB gradient};
when $g$ is the gradient of a scalar loss, $\nabla \Phi$ is the
HVP applied to the attention block~\citep{pearlmutter1994fast}.

\vspace*{-0.05in}
\subsection{Affine Fusion}
\label{sec:affine}
\label{sec:bob-formulation}
\vspace*{-0.05in}

Reverse-mode automatic differentiation applied to~\eqref{eq:bob-obj}
(full derivation in Appendix~\ref{app:derivation}) yields the four BoB
outputs as
\begin{align}
  \td{Q} &= \tau\bigl(\td{S}\, K + dS\, U_K\bigr), &
  \td{K} &= \tau\bigl(\td{S}^\Tr Q + dS^\Tr U_Q\bigr), \label{eq:QKtilde}\\
  \td{V} &= (P \odot h)^\Tr dO, &
  \td{dO} &= P\, U_V + (P \odot h)\, V, \label{eq:VOtilde}
\end{align}
where $\odot$ denotes elementwise multiplication; $dP$, $dS$, and
$D_i$ are the first-backward intermediates from
\eqref{eq:first-bwd}--\eqref{eq:first-bwd-2}; and
$h_{ij}=F_{ij}-\alpha_i$, $F=\tau(U_QK^\Tr+QU_K^\Tr)$,
$C=dO\,U_V^\Tr$,
$\alpha_i=\sum_{j\in\mathcal V_i}P_{ij}F_{ij}$,
$\td{P}_{ij}=C_{ij}+F_{ij}(dP_{ij}-D_i)-\alpha_i dP_{ij}$,
$E_i=\sum_{j\in\mathcal V_i}P_{ij}\td{P}_{ij}$, and
$\td{S}_{ij}=P_{ij}(\td{P}_{ij}-E_i)$. We write
$\alpha=(\alpha_1,\ldots,\alpha_N)$ and $E=(E_1,\ldots,E_N)$.
Thus the second derivative introduces two row scalars, $\alpha_i$ and
$E_i$, which are the BoB analogs of the first-backward row scalar $D_i$.

The four outputs split according to how their sums are formed:
$(\td Q,\td{dO})$ and the row scalars $(\alpha,E)$ sum over columns for each fixed
row, while $(\td K,\td V)$ sum over rows for each fixed column. Under the
FlashAttention execution model, this requires a row-wise and a
column-wise pass. In addition, the row-wise computations appear to have the
sequential dependence $\alpha_i \to \td{P}_{ij} \to E_i \to \td{S}_{ij}$, which
suggests four passes: one for $\alpha$, one for $E$, one for
$(\td Q,\td{dO})$, and one column-wise pass for $(\td K,\td V)$. Both row scalars, however, enter later quantities only \emph{affinely}: $\alpha_i$ enters $\td{P}_{ij}$ through the term $-\alpha_i dP_{ij}$, and $E_i$ enters $\td{S}_{ij}$ through a broadcast
subtraction. These affine forms allow the row-wise computations to be fused: one can accumulate $\alpha$-free row quantities\footnote{We describe a quantity as
$\alpha$-free if it can be computed without first computing $\alpha$.} first, and then apply the affine
corrections once the row scalars are known, without revisiting the row. Proposition~\ref{prop:affine} makes this explicit by naming six row accumulators, five of which are $\alpha$-free, and the corrections that finalize the row outputs within each row block (Figure~\ref{fig:overview}).

\begin{proposition}[Affine fusion]
\label{prop:affine}
There exist six row accumulators consisting of $\alpha_i$ together with five $\alpha$-free quantities $E^{\circ}_i, B_i, \td{dO}^{\circ}_i, R_i, \Omega_i$ (defined in Appendix~\ref{app:affine-proof}), all with tilewise computable summands, such that
\begin{align}
\label{eq:fusion}
E_i = E^{\circ}_i - \alpha_i D_i, \quad
\td{dO}_i = \td{dO}^{\circ}_i - \alpha_i o_i, \quad
\td{Q}_i = \tau\bigl(\Omega_i - \alpha_i R_i - E_i B_i\bigr).
\end{align}
\end{proposition}

The proof is by direct substitution; see Appendix~\ref{app:affine-proof}.
A single-pass schedule for both row-owned and column-owned outputs would
need to pay one of three costs: materializing an \(N\times N\) tensor in
HBM, accumulating column-owned outputs through global atomic adds, or
holding \(\Theta(Nd)\) partial column accumulators on chip. FlashBack~\citep{engstrom2024flashback}
pays the second cost; \textsc{FlashBoB} uses two passes to avoid all
three. The schedule-level argument is given in Appendix~\ref{app:two-pass-proof}.

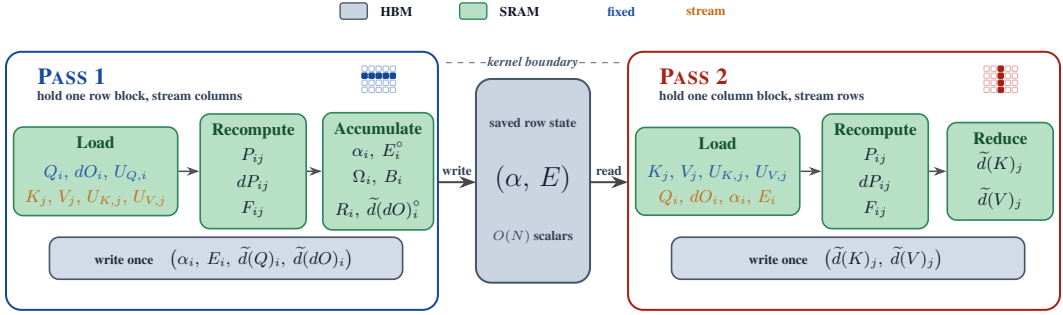
\begin{figure}[t]
\centering
\resizebox{\textwidth}{!}{%
\begin{tikzpicture}[
    x=1cm, y=1cm,
    font=\small,
    >={Stealth[length=2.4mm,width=1.8mm]},
    panel/.style={
        rounded corners=8pt,
        line width=1.05pt
    },
    stage/.style={
        rounded corners=5pt,
        line width=0.85pt,
        align=center,
        inner xsep=6pt,
        inner ysep=5pt,
        minimum height=1.55cm,
        text=fbInk
    },
    writebar/.style={
        rounded corners=5pt,
        line width=0.85pt,
        align=center,
        inner xsep=8pt,
        inner ysep=5pt,
        minimum height=0.70cm,
        text=fbInk
    },
    flowarrow/.style={
        -{Stealth[length=3.1mm,width=2.25mm]},
        line width=1.05pt,
        draw=black!72
    },
    stagearrow/.style={
        -{Stealth[length=2.35mm,width=1.65mm]},
        line width=0.8pt,
        draw=black!68
    },
    legendsq/.style={
        rounded corners=1.5pt,
        line width=0.7pt,
        inner sep=0pt,
        minimum width=0.5cm,
        minimum height=0.30cm
    }
]

\coordinate (Lsw) at (0.00,0.00);
\coordinate (Lne) at (7.95,4.75);
\coordinate (Bsw) at (8.65,0.50);
\coordinate (Bne) at (10.75,4.25);
\coordinate (Rsw) at (11.45,0.00);
\coordinate (Rne) at (19.40,4.75);

\node[legendsq, draw=hbmStroke, fill=hbmFill] (legHBM) at (6.40,5.50) {};
\node[anchor=west, font=\scriptsize\bfseries, text=fbInk]
    at ([xshift=3pt]legHBM.east) {HBM};

\node[legendsq, draw=sramStroke, fill=sramFill] (legSRAM) at (8.60,5.50) {};
\node[anchor=west, font=\scriptsize\bfseries, text=fbInk]
    at ([xshift=3pt]legSRAM.east) {SRAM};

\node[anchor=west, font=\scriptsize\bfseries, text=roleFixed]
    at (10.95,5.50) {fixed};
\node[anchor=west, font=\scriptsize\bfseries, text=roleStream]
    at (12.40,5.50) {stream};

\path[panel, draw=fbBlue!92!black] (Lsw) rectangle (Lne);
\path[panel, draw=hbmStroke, fill=hbmFill] (Bsw) rectangle (Bne);
\path[panel, draw=fbRed!92!black] (Rsw) rectangle (Rne);

\draw[dashed, line width=0.55pt, draw=hbmStroke!85]
    (8.10,4.55) -- (8.85,4.55);
\node[font=\scriptsize\bfseries\itshape, text=hbmStroke] at (9.70,4.55)
    {kernel boundary};
\draw[dashed, line width=0.55pt, draw=hbmStroke!85]
    (10.55,4.55) -- (11.30,4.55);

\node[anchor=west, font=\bfseries\large, text=fbBlue!95!black]
    at ($(Lsw)+(0.45,4.30)$) {\textsc{Pass 1}};
\node[anchor=west, font=\scriptsize\bfseries, text=fbMuted]
    at ($(Lsw)+(0.45,3.93)$) {hold one row block, stream columns};

\begin{scope}[shift={(6.55,4.00)}, x=0.13cm, y=0.13cm]
    \foreach \gx in {0,1,2,3,4} {
        \foreach \gy in {0,1,2,3} {
            \path[
                draw=fbBlue!42,
                fill=white,
                line width=0.32pt,
                rounded corners=0.6pt
            ]
                (\gx,\gy) rectangle ++(0.76,0.76);
        }
    }
    \foreach \gx in {0,1,2,3,4} {
        \path[
            draw=fbBlue!95!black,
            fill=fbBlue!90!black,
            line width=0.35pt,
            rounded corners=0.6pt
        ]
            (\gx,2) rectangle ++(0.76,0.76);
    }
\end{scope}

\node[stage, draw=sramStroke, fill=sramFill, minimum width=2.95cm]
    (Lload) at ($(Lsw)+(1.65,2.55)$) {%
    {\bfseries\small\color{sramInk} Load}\\[5pt]
    {\small\color{roleFixed}$Q_i,\, dO_i,\, U_{Q,i}$}\\[3pt]
    {\small\color{roleStream}$K_j,\, V_j,\, U_{K,j},\, U_{V,j}$}
};

\node[stage, draw=sramStroke, fill=sramFill, minimum width=1.78cm]
    (Lrec) at ($(Lsw)+(4.55,2.55)$) {%
    {\bfseries\small\color{sramInk} Recompute}\\[3pt]
    {\small\color{fbInk}$
    \begin{gathered}
        P_{ij}\\[1pt]
        dP_{ij}\\[1pt]
        F_{ij}
    \end{gathered}
    $}
};

\node[stage, draw=sramStroke, fill=sramFill, minimum width=1.95cm]
    (Lacc) at ($(Lsw)+(6.85,2.55)$) {%
    {\bfseries\small\color{sramInk} Accumulate}\\[3pt]
    {\small\color{fbInk}$
    \begin{gathered}
        \alpha_i,\ E^{\circ}_i\\[1pt]
        \Omega_i,\ B_i\\[1pt]
        R_i,\ \td{dO}^{\circ}_i
    \end{gathered}
    $}
};

\draw[stagearrow] (Lload.east) -- (Lrec.west);
\draw[stagearrow] (Lrec.east)  -- (Lacc.west);

\node[writebar, draw=hbmStroke, fill=hbmFill!85, minimum width=6.50cm]
    (Lwrite) at ($(Lsw)+(4.05,0.95)$) {%
    {\bfseries\scriptsize\color{hbmInk} write once}\hspace{6pt}
    {\small\color{fbInk}$\bigl(\alpha_i,\ E_i,\ \td{Q}_i,\ \td{dO}_i\bigr)$}
};

\node[font=\scriptsize\bfseries, text=fbMuted] at (9.70,3.45)
    {saved row state};
\node[font=\bfseries\Large, text=fbInk] at (9.70,2.40)
    {$(\alpha,\,E)$};
\node[font=\scriptsize\bfseries, text=fbMuted] at (9.70,1.35)
    {$O(N)$ scalars};

\node[anchor=west, font=\bfseries\large, text=fbRed!95!black]
    at ($(Rsw)+(0.45,4.30)$) {\textsc{Pass 2}};
\node[anchor=west, font=\scriptsize\bfseries, text=fbMuted]
    at ($(Rsw)+(0.45,3.93)$) {hold one column block, stream rows};

\begin{scope}[shift={(18.00,4.00)}, x=0.13cm, y=0.13cm]
    \foreach \gx in {0,1,2,3,4} {
        \foreach \gy in {0,1,2,3} {
            \path[
                draw=fbRed!38,
                fill=white,
                line width=0.32pt,
                rounded corners=0.6pt
            ]
                (\gx,\gy) rectangle ++(0.76,0.76);
        }
    }
    \foreach \gy in {0,1,2,3} {
        \path[
            draw=fbRed!95!black,
            fill=fbRed!90!black,
            line width=0.35pt,
            rounded corners=0.6pt
        ]
            (2,\gy) rectangle ++(0.76,0.76);
    }
\end{scope}

\node[stage, draw=sramStroke, fill=sramFill, minimum width=2.95cm]
    (Rload) at ($(Rsw)+(1.65,2.55)$) {%
    {\bfseries\small\color{sramInk} Load}\\[5pt]
    {\small\color{roleFixed}$K_j,\, V_j,\, U_{K,j},\, U_{V,j}$}\\[3pt]
    {\small\color{roleStream}$Q_i,\, dO_i,\, \alpha_i,\, E_i$}
};

\node[stage, draw=sramStroke, fill=sramFill, minimum width=1.78cm]
    (Rrec) at ($(Rsw)+(4.55,2.55)$) {%
    {\bfseries\small\color{sramInk} Recompute}\\[3pt]
    {\small\color{fbInk}$
    \begin{gathered}
        P_{ij}\\[1pt]
        dP_{ij}\\[1pt]
        F_{ij}
    \end{gathered}
    $}
};

\node[stage, draw=sramStroke, fill=sramFill, minimum width=1.95cm]
    (Rred) at ($(Rsw)+(6.85,2.55)$) {%
    {\bfseries\small\color{sramInk} Reduce}\\[3pt]
    {\small\color{fbInk}$
    \begin{gathered}
        \td{K}_j\\[3pt]
        \td{V}_j
    \end{gathered}
    $}
};

\draw[stagearrow] (Rload.east) -- (Rrec.west);
\draw[stagearrow] (Rrec.east)  -- (Rred.west);

\node[writebar, draw=hbmStroke, fill=hbmFill!85, minimum width=6.50cm]
    (Rwrite) at ($(Rsw)+(4.05,0.95)$) {%
    {\bfseries\scriptsize\color{hbmInk} write once}\hspace{6pt}
    {\small\color{fbInk}$\bigl(\td{K}_j,\ \td{V}_j\bigr)$}
};

\draw[flowarrow]
    ($(Lne)+(0.00,-2.40)$) --
    node[above, font=\scriptsize\bfseries, text=fbMuted] {write}
    ($(Bsw)+(0.00,1.85)$);

\draw[flowarrow]
    ($(Bne)+(0.00,-1.90)$) --
    node[above, font=\scriptsize\bfseries, text=fbMuted] {read}
    ($(Rsw)+(0.00,2.35)$);

\end{tikzpicture}%
}
\caption{\textsc{FlashBoB} overview. The legend at top encodes memory tier (HBM, SRAM) and tile role; saturated text marks variables held \textit{fixed} in registers, lighter text marks variables \textit{streamed}
through. Pass~1 holds a row block and streams column blocks to finalize the
row-side outputs and write $(\alpha,E)$ to HBM. Across the kernel boundary,
Pass~2 holds a column block, streams row blocks, reads $(\alpha,E)$, and
writes the remaining outputs. Only two length-$N$ vectors cross the boundary,
and no $N{\times}N$ tensor is ever materialized in HBM.}
\label{fig:overview}
\end{figure}

\vspace*{-0.05in}
\subsection{Algorithm}
\label{sec:algo}
\vspace*{-0.05in}

Algorithm~\ref{alg:flashbob} realizes the two-pass schedule in the block-streaming style of FlashAttention (Figure~\ref{fig:overview}). HBM-resident inputs are
$Q, K, V, dO, U_Q, U_K, U_V, O \in \R^{N \times d}$ and the saved
row state $L, D \in \R^N$. Pass~1 keeps a query-row block in SRAM, streams over
key-value blocks, accumulates the six row accumulators of
Proposition~\ref{prop:affine}, applies
the corrections~\eqref{eq:fusion}, and writes the auxiliary row quantities
$(\alpha_i, E_i)$ together with $(\td{Q}_i, \td{dO}_i)$ once per row block.
Pass~2 keeps a key-value block in SRAM, streams over row blocks together with
the saved $(\alpha_i, E_i)$, and writes $(\td{K}_j, \td{V}_j)$ once per
column block. No $N \times N$ intermediate is ever written to HBM. Block sizes follow the
FlashAttention-style SRAM budget. With
$b = \lfloor M / (c_{\textsc{BoB}} d) \rfloor$, where
$c_{\textsc{BoB}}>0$ accounts for all arrays and scalar buffers that must be stored in SRAM while processing one block, Pass~1 keeps a row block in SRAM and streams smaller column blocks,
$B_r^{(1)} = b$ and $B_c^{(1)} = \min(b,d)$; Pass~2 reverses the
 orientation, using $B_c^{(2)} = b$ and $B_r^{(2)} = \min(b,d)$.
Under grouped-query or multi-query attention ($H_Q > H_{KV}$), each query head in Pass~1 reads its group's shared $(K,V,U_K,U_V)$ head, while Pass~2 accumulates $(\td{K}_j, \td{V}_j)$ across both row blocks and the $H_Q/H_{KV}$ query heads in the group before writing once to HBM.
In our Triton implementation at $d=128$, autotuning selects $(B_r,B_c)\in\{(16,16),(32,16)\}$ with two stages for Pass~1 and $(B_r,B_c)\in\{(16,16),(16,32)\}$ with two or three stages for Pass~2, all with four warps (Appendix~\ref{app:tile-configs}); Appendix~\ref{app:io-ledger} gives the SRAM budget derivation.
 
\begin{algorithm}[t]
\caption{\textsc{FlashBoB}: two-pass exact attention BoB. A fully
annotated version with explicit tile shapes appears as
Algorithm~\ref{alg:flashbob-full} in Appendix~\ref{app:full-algo}.}
\label{alg:flashbob}
\small
\begin{algorithmic}[1]
\Require Matrices $Q, K, V, dO, U_Q, U_K, U_V, O \in \R^{N \times d}$; row state $L, D \in \R^N$, all in HBM; SRAM size $M$
\Ensure BoB outputs $\td{Q}, \td{K}, \td{V}, \td{dO}$ and row scalars $\alpha, E$
\State $b \gets \lfloor M / (c_{\textsc{BoB}} d) \rfloor$;\;
       $B_r^{(1)}, B_c^{(2)} \gets b$;\;
       $B_c^{(1)}, B_r^{(2)} \gets \min(b, d)$
\Statex \colorbox{green!10}{\parbox{0.97\linewidth}{\textsc{Pass 1 (row-major):} Accum. row summaries $\rightarrow$ apply scalar updates $\rightarrow$ write $(\alpha, E, \td{Q}, \td{dO})$}}
\For{$i = 1$ \textbf{to} $\lceil N / B_r^{(1)} \rceil$}
    \State Load active row tiles $(Q_i, dO_i, U_{Q,i}, O_i)$ and saved row scalars $(L_i, D_i)$\;
    \State Zero the six row accumulators \Comment{Prop.~\ref{prop:affine}}
    \For{$j = 1$ \textbf{to} $\lceil N / B_c^{(1)} \rceil$}
        \State Load $(K_j, V_j, U_{K,j}, U_{V,j})$; recompute $P_{ij}$
               \Comment{from saved $L$}
        \State Form tile intermediates $dP_{ij}, F_{ij}, C_{ij}, \td{P_{ij}}^{\circ}$ on chip
        \State Accumulate $\alpha_i,\, E^{\circ}_i,\, \td{dO}^{\circ}_i,\, \Omega_i,\, B_i,\, R_i$
    \EndFor
    \State Apply scalar updates to form $E_i, \td{dO}_i, \td{Q}_i$
           \Comment{Eq.~\eqref{eq:fusion}}
    \State Write $\alpha_i, E_i, \td{Q}_i, \td{dO}_i$ to HBM
\EndFor
\Statex \colorbox{orange!12}{\parbox{0.97\linewidth}{\textsc{Pass 2 (column-major):} Load saved $(\alpha, E)$ $\rightarrow$ accumulate and write $(\td{K}, \td{V})$}}
\For{$j = 1$ \textbf{to} $\lceil N / B_c^{(2)} \rceil$}
    \State Load active column tiles $(K_j, V_j, U_{K,j}, U_{V,j})$;
           zero $\td{K}_j, \td{V}_j$
    \For{$i = 1$ \textbf{to} $\lceil N / B_r^{(2)} \rceil$}
        \State Load streamed row tiles and saved row scalars $(L_i, D_i, \alpha_i, E_i)$;\;recompute $P_{ij}$ and tile intermediates
        \State $\td{V}_j \mathrel{+}= (P_{ij} \odot h_{ij})^\Tr dO_i$;\;
               $\td{K}_j \mathrel{+}= \tau(\td{S}_{ij}^\Tr Q_i + dS_{ij}^\Tr U_{Q,i})$
    \EndFor
    \State Write $\td{K}_j, \td{V}_j$ to HBM
\EndFor
\end{algorithmic}
\end{algorithm}
 
\begin{theorem}[Correctness, FLOPs, and memory]
\label{thm:correct}
Algorithm~\ref{alg:flashbob} computes exactly the BoB outputs defined
by Equations~\eqref{eq:QKtilde}--\eqref{eq:VOtilde}, requires
$\Theta(N^2 d)$ FLOPs, uses $O(N)$ extra HBM beyond inputs, outputs,
the saved row state $(L,D)$, and the forward output $O$, and fits
within the fast-memory budget $M$ for block sizes $B_r, B_c$ satisfying
$B_r B_c + (B_r + B_c) d = O(M)$.
\end{theorem}
\vspace*{-0.05in}
The proof is in Appendix~\ref{app:correctness-proof}; the only auxiliary HBM beyond inputs, outputs, the saved row state $(L,D)$, and the forward output $O$ is the two length-\(N\) vectors \(\alpha,E\).

\vspace*{-0.05in}
\subsection{I/O complexity and large-cache optimality}
\label{sec:io}
\label{sec:optimality}
\vspace*{-0.05in}

We analyze HBM traffic in the standard two-level memory model from
\S\ref{sec:prelim}. SRAM has capacity \(M\) elements, and HBM traffic
counts elements moved between HBM and SRAM. Let
\(B_r^{(1)},B_c^{(1)}\) denote the row and column block sizes in Pass~1,
and \(B_r^{(2)},B_c^{(2)}\) those in Pass~2. In this subsection, we focus on the large-cache regime where $M=\Omega(d^2)$.
 
\begin{theorem}[HBM traffic]
\label{thm:io}
For attention with \(N\) rows and head dimension \(d\),
Algorithm~\ref{alg:flashbob} performs
$Q_{\textsc{FlashBoB}}
=
12Nd+4N
+
\frac{4dN^2}{B_r^{(1)}}
+
\frac{(3d+4)N^2}{B_c^{(2)}}$
HBM transfers, up to boundary effects.
\end{theorem}

\emph{Remark}. Only the reused block sizes \(B_r^{(1)}\) and \(B_c^{(2)}\) appear in
Theorem~\ref{thm:io}: the streamed dimension cancels inside each pass.
The \(12Nd+4N\) term collects input loads, row-metadata reads, and output
writes. Every \(N\times N\) quantity is recomputed and consumed within
SRAM before moving to the next block. The per-pass ledger is given in
Appendix~\ref{app:io-ledger}.

\begin{corollary}[Large-cache upper bound]
\label{cor:io-upper}
Assume \(d^2 \lesssim M \lesssim Nd\) and \(N\ge d\). With
\(B_r^{(1)} = B_c^{(2)} = b\),
\(B_c^{(1)} = B_r^{(2)} = \min(b,d)\), and
\(b = \Theta(M/d)\), \textsc{FlashBoB} has HBM traffic
\(Q_{\textsc{FlashBoB}} = \Theta(N^2d^2/M)\).
\end{corollary}

\emph{Remark}. Indeed, substituting \(b=\Theta(M/d)\) into Theorem~\ref{thm:io} gives
\(Q_{\textsc{FlashBoB}} =
12Nd+4N+(7d+4)N^2/b =
12Nd+4N+\Theta(N^2d^2/M)\). Under \(M\lesssim Nd\), the \(\Theta(N^2d^2/M)\) term dominates the stationary \(O(Nd)\) traffic.

\begin{theorem}[Large-cache lower bound under score recomputation]
\label{thm:bob-lb}
Assume \(M=\Omega(d^2)\), \(N\ge d\), \(M\lesssim Nd\), and
\(\Theta(N^2)\) valid query-key interactions. This includes dense
attention and causal attention, since causal masking leaves
\(N(N+1)/2=\Theta(N^2)\) valid interactions. In the standard
FlashAttention-style score-recomputation model formalized in
Appendix~\ref{app:lb-proof}, every exact attention-BoB algorithm in this model
requires \(\Omega(N^2d^2/M)\) HBM transfers.
\end{theorem}

The proof is given in Appendix~\ref{app:lb-proof}. The reduction sets
\(V=0\), \(dO=0\), \(U_Q=U_K=0\), and \(U_V=W\), giving
\(F=0\), \(\alpha=0\), \(h=0\), and \(\td{dO}=PW=\softmax(\tau QK^\Tr)W\).
Thus the lower bound is inherited from the large-cache exact-attention lower bound of
\citep[Lemma~3.4]{saha2024ioattention}. This is a conditional optimality statement: it applies to exact
FlashAttention-style score-recomputation algorithms, not to arbitrary
exact algorithms that use a different computational model.

Combining Corollary~\ref{cor:io-upper} with
Theorem~\ref{thm:bob-lb}, \textsc{FlashBoB} has
\(\Theta(N^2d^2/M)\) HBM traffic in the regime
\(d^2\lesssim M\lesssim Nd\), matching the score-recomputation lower
bound up to stationary \(O(Nd)\) input, output, and row-state traffic.

\vspace*{-0.1in}
\section{Experiments}
\label{sec:experiments}
\vspace*{-0.1in}

We evaluate \textsc{FlashBoB} at three levels: isolated attention BoB
on A100 and B300 (\S\ref{sec:attn-kernel}), full-model second-order steps
across GPT-2, Pythia, and five native GQA and MoE models on B200 and B300 (\S\ref{sec:benchmarks}), and 1B-token Sophia-H pretraining
runs on \(8{\times}\)A6000 and \(8{\times}\)A100 (\S\ref{sec:fineweb}). For the main causal-attention experiments, we use PyTorch's \texttt{nn.scaled\_dot\_product\_attention} fused backend for the forward and first backward, and replace only the attention double backward with our two-pass Triton kernel. Experiments requiring features not covered by PyTorch, such as the sliding-window experiments in Appendix~\ref{app:swa}, use the \texttt{flash-attn} package. Correctness, additional FlashBack details, and software environment configurations are in
Appendices~\ref{app:errors}, \ref{app:flashback-comparison},
and~\ref{app:env}.

\vspace*{-0.05in}
\subsection{Attention-layer kernel scaling}
\label{sec:attn-kernel}
\vspace*{-0.05in}

We compare four exact attention-BoB implementations: \textsc{math}, the PyTorch materializing reference; \textsc{hvp-m} and \textsc{hvp-s}, the GradMem-derived hand-written HVP kernels~\citep{kuratov2026gradmem}; and \textsc{FlashBoB}, the two-pass schedule of Algorithm~\ref{alg:flashbob}. The HVP baselines avoid full materialization but do not preserve FlashAttention-style tile streaming at the BoB level. Table~\ref{tab:attn-kernel} reports the sequence-length sweeps at $d=64$ (A100 80GB) and $d=128$ (B300); the full $d=64$ sweep from $N{=}256$ through $N{=}262{,}144$ appears in Appendix~\ref{app:extended-kernel-scaling}.

\begin{table}[t]
\centering\scriptsize
\setlength{\tabcolsep}{2.5pt}
\caption{Isolated attention-layer BoB scaling at $d=64$ (NVIDIA A100 80GB) and $d=128$ (NVIDIA B300), with batch 4, 12 heads, causal mask, and BF16 inputs with FP32 reductions. Peak memory reports the lowest-memory baseline that fits (\textsc{hvp-m}) versus \textsc{FlashBoB}; full four-method memory tables through $N{=}262{,}144$ are in Appendix~\ref{app:extended-kernel-scaling}. OOM denotes out of memory; best values are bold.}
\label{tab:attn-kernel}
\label{tab:d128-isolated}
\begin{tabular}{@{}r rrrr rr rrrr rr@{}}
\toprule
& \multicolumn{6}{c}{Head dimension $d=64$ (NVIDIA A100 80GB)} & \multicolumn{6}{c}{Head dimension $d=128$ (NVIDIA B300)} \\
\cmidrule(lr){2-7}\cmidrule(lr){8-13}
& \multicolumn{4}{c}{Wall-clock time (ms)} & \multicolumn{2}{c}{Peak mem.\ (MiB)} & \multicolumn{4}{c}{Wall-clock time (ms)} & \multicolumn{2}{c}{Peak mem.\ (MiB)} \\
\cmidrule(lr){2-5}\cmidrule(lr){6-7}\cmidrule(lr){8-11}\cmidrule(lr){12-13}
$N$ & \textsc{math} & \textsc{hvp-m} & \textsc{hvp-s} & \textsc{FlashBoB} & Base. & \textsc{FlashBoB} & \textsc{math} & \textsc{hvp-m} & \textsc{hvp-s} & \textsc{FlashBoB} & Base. & \textsc{FlashBoB} \\
\midrule
   256 &   1.82 &   1.58 &   1.66 & \textbf{0.94} &    114 &    \textbf{95} &   1.3 &   1.1 &   1.7 &    \textbf{0.7} &     160 &    \textbf{133} \\
   512 &   4.95 &   2.31 &   2.33 & \textbf{1.29} &    315 &   \textbf{171} &   2.1 &   1.2 &   1.6 &    \textbf{0.9} &     401 &    \textbf{250} \\
1\,024 &  17.66 &   7.39 &   6.41 & \textbf{2.57} & 1\,045 &   \textbf{323} &   6.8 &   2.9 &   3.3 &    \textbf{1.8} &  1\,207 &    \textbf{485} \\
2\,048 &  68.80 &  28.34 &  22.73 & \textbf{7.53} & 3\,805 &   \textbf{630} &  24.9 &  11.6 &  10.8 &    \textbf{5.2} &  3\,745 &    \textbf{553} \\
4\,096 & 265.67 & 114.89 &  89.14 & \textbf{25.80} & 14\,521 & \textbf{1\,243} &  94.2 &  41.5 &  39.1 &   \textbf{18.2} & 14\,401 & \textbf{1\,106} \\
8\,192 &    OOM & 468.93 & 356.84 & \textbf{95.11} & 55\,971 & \textbf{1\,702} & 366.3 & 178.6 & 160.7 &   \textbf{68.2} & 56\,498 & \textbf{2\,211} \\
16\,384 &   OOM &    OOM &    OOM & \textbf{365.73} &   OOM & \textbf{3\,372} &   OOM & 620.0 & 602.6 &  \textbf{263.6} & 223\,829 & \textbf{4\,422} \\
32\,768 &   OOM &    OOM &    OOM & \textbf{1\,449.06} & OOM & \textbf{6\,744} &   OOM &   OOM &   OOM & \textbf{1\,037.0} &   OOM & \textbf{8\,844} \\
65\,536 &   OOM &    OOM &    OOM & \textbf{5\,843.14} & OOM & \textbf{13\,488} &  OOM &   OOM &   OOM & \textbf{4\,124.2} &   OOM & \textbf{17\,688} \\
\bottomrule
\end{tabular}
\vspace*{-0.12in}
\end{table}

At $d=64$ on A100, the speedup of \textsc{FlashBoB} over \textsc{math} grows with $N$
(from $1.9\times$ at $N{=}256$ to $10.3\times$ at $N{=}4{,}096$), and
the memory ratio grows from $2.2\times$ to $28.0\times$.
\textsc{FlashBoB} is fastest at every tested length, reaching $3.8\times$ over \textsc{hvp-s} at $N{=}8{,}192$.
\textsc{math} runs out of memory at $N{=}8{,}192$ and both HVP kernels at $N{=}16{,}384$, whereas \textsc{FlashBoB} continues to $N{=}262{,}144$ at $52.7$~GiB peak (53,952~MiB; Appendix~\ref{app:extended-kernel-scaling}). This $N{=}262{,}144$ measurement is an isolated attention-layer result rather than a full-model training length: at $N{=}32{,}768$, the isolated \textsc{FlashBoB} kernel uses 6,744~MiB on A100, whereas a complete GPT-2 Small second-order step uses 128,392~MiB on B200 because non-attention activations dominate full-step memory (\S\ref{sec:benchmarks}). At $d=128$ on B300 (Table~\ref{tab:d128-isolated}, right), \textsc{FlashBoB} is $2.08$--$2.36\times$ faster than the fastest surviving baseline across $N{=}2{,}048$--$16{,}384$, reduces peak memory by up to $50.6\times$ ($4{,}422$~MiB vs.\ $223{,}829$~MiB at $N{=}16{,}384$), and alone runs at $N{=}32{,}768$ and $65{,}536$. A fixed-$N$ sweep across $d\in\{32,64,128\}$ on A100 appears in Appendix~\ref{app:head-dim-sweep}.

\begin{table}[t]
\centering\scriptsize
\setlength{\tabcolsep}{5.0pt}
\caption{Matched isolated-kernel comparison with FlashBack at $d=64$ (NVIDIA A6000 48GB) and $d=128$ (NVIDIA B300), using batch 4, 12 heads, causal mask, and BF16 inputs with FP32 reductions. Both kernels are autotuned at each $N$ (at $d=128, N{=}8{,}192$, FlashBack takes $619.00$~ms vs.\ $68.23$~ms for \textsc{FlashBoB}, $9.07\times$; See Appendix~\ref{app:flashback-comparison}).}
\label{tab:flashback-main}
\begin{tabular}{@{}r rrr rrr@{}}
\toprule
& \multicolumn{3}{c}{Head dimension $d=64$ (NVIDIA A6000 48GB)} & \multicolumn{3}{c}{Head dimension $d=128$ (NVIDIA B300)} \\
\cmidrule(lr){2-4}\cmidrule(lr){5-7}
$N$ & FlashBack (ms) & \textsc{FlashBoB} (ms) & Speedup & FlashBack (ms) & \textsc{FlashBoB} (ms) & Speedup \\
\midrule
   256 &      1.94 &       \textbf{0.94} & $2.07\times$ &      1.22 &      \textbf{0.75} & $1.63\times$ \\
   512 &      6.22 &       \textbf{1.52} & $4.09\times$ &      3.24 &      \textbf{0.92} & $3.52\times$ \\
 1\,024 &     22.16 &       \textbf{3.80} & $5.83\times$ &     10.69 &      \textbf{1.79} & $5.98\times$ \\
 2\,048 &     67.25 &      \textbf{11.73} & $5.74\times$ &     39.03 &      \textbf{5.19} & $7.53\times$ \\
 4\,096 &    262.47 &      \textbf{41.51} & $6.32\times$ &    156.03 &     \textbf{18.23} & $8.56\times$ \\
16\,384 &  2\,421.36 &     \textbf{651.28} & $3.72\times$ &  2\,451.53 &    \textbf{263.65} & $9.30\times$ \\
32\,768 &  9\,659.78 &  \textbf{2\,682.60} & $3.60\times$ &  9\,754.58 & \textbf{1\,037.03} & $9.41\times$ \\
65\,536 & 38\,442.11 & \textbf{10\,729.20} & $3.58\times$ & 38\,902.59 & \textbf{4\,124.16} & $9.43\times$ \\
\bottomrule
\end{tabular}
\vspace*{-0.12in}
\end{table}

\textbf{FlashBack comparison.}
FlashBack~\citep{engstrom2024flashback} is the closest prior
non-materializing exact BoB attention kernel. Table~\ref{tab:flashback-main} compares both kernels on A6000 ($d=64$) and B300 ($d=128$, with both kernels autotuned over block sizes, stages, and warp counts). \textsc{FlashBoB} is faster at every length, achieving $2.07$--$6.32\times$ speedups at $d=64$ and $1.63$--$9.43\times$ at $d=128$ across $N{=}256$--$65{,}536$ (Appendix~\ref{app:flashback-comparison}). In a matched MQA benchmark on B300 ($H_Q/H_{KV}=8/1$, batch 1, $N{=}2{,}048$--$131{,}072$), \textsc{FlashBoB} is $3.96$--$5.04\times$ faster at $d=64$ and $5.87$--$8.90\times$ faster at $d=128$.
Whereas FlashBack's row-parallel single-kernel design accumulates $(\td{K},\td{V})$ through contended global atomic adds, \textsc{FlashBoB} assigns these outputs to a column-major pass and writes each output tile once.

\vspace*{-0.05in}
\subsection{Full-model benchmarks}
\label{sec:benchmarks}
\vspace*{-0.05in}

We replace the attention double backward inside complete language models across three architecture groups: GPT-2 Small, Medium, and Large ($d=64$, MHA on B200), Pythia-160M and Pythia-410M~\citep{biderman2023pythia} (GPT-NeoX architecture, $d=128$, MHA on B300), and five released native GQA and MoE models on B300 (Llama 3.2 1B, SmolLM2-135M, SmolLM2-360M, Qwen2.5-1.5B, and Granite 3.1 1B-A400M). Non-attention components are unchanged, so the isolated kernel gains upper-bound the whole-model gains.
Tables~\ref{tab:gpt2-bench} and~\ref{tab:gqa-main} report representative configurations; full per-$N$ sweeps for all ten models appear in Appendix~\ref{app:gpt2-full}.

\begin{table}[t]
\centering\scriptsize
\caption{End-to-end second-order model steps on MHA architectures: \textbf{(a)} GPT-2 ($d=64$, NVIDIA B200 192~GiB) and \textbf{(b)} Pythia (GPT-NeoX, $d=128$, NVIDIA B300), with batch 4 and BF16 inputs. For GPT-2, Base.\ reports the lowest-memory baseline that fits (\textsc{hvp-m}); full four-method sweeps are in Appendix~\ref{app:gpt2-full}.}
\label{tab:gpt2-bench}
\begin{minipage}[t]{0.52\textwidth}
\centering
\textbf{(a) GPT-2 ($d=64$, NVIDIA B200 192~GiB)}\\[2pt]
\setlength{\tabcolsep}{1.5pt}
\begin{tabular}{@{}ll rrrr rr@{}}
\toprule
& & \multicolumn{4}{c}{Wall-clock time (ms)} & \multicolumn{2}{c}{Peak mem.\ (MiB)} \\
\cmidrule(lr){3-6}\cmidrule(lr){7-8}
Model & $N$ & \textsc{math} & \textsc{hvp-m} & \textsc{hvp-s} & \textsc{FlashBoB} & Base. & \textsc{FlashBoB} \\
\midrule
\multirow{3}{*}{Small}
  & 4\,096  & 397.7 & 417.5 & 304.0 & \textbf{190.9}  & \textbf{17\,649} & 17\,804 \\
  & 8\,192  &   OOM & 1537.5 & 1017.3 & \textbf{509.1} & 46\,965 & \textbf{33\,604} \\
  & 32\,768 &   OOM &   OOM &   OOM & \textbf{5368.3} &    OOM & \textbf{128\,392} \\
\midrule
\multirow{3}{*}{Med.}
  & 2\,048  & 312.6 & 343.0 & 255.9 & \textbf{177.2}  & \textbf{16\,612} & 16\,810 \\
  & 4\,096  &   OOM & 1028.5 & 726.1 & \textbf{420.9} & 30\,480 & \textbf{28\,677} \\
  & 16\,384 &   OOM &   OOM &   OOM & \textbf{3878.9} &    OOM & \textbf{99\,912} \\
\midrule
\multirow{3}{*}{Large}
  & 2\,048  & 562.2 & 614.6 & 450.4 & \textbf{304.3}  & \textbf{27\,666} & 28\,045 \\
  & 4\,096  &   OOM & 1879.6 & 1313.4 & \textbf{740.3} & 49\,205 & \textbf{45\,457} \\
  & 16\,384 &   OOM &   OOM &   OOM & \textbf{7106.5} &    OOM & \textbf{150\,019} \\
\bottomrule
\end{tabular}
\end{minipage}\hfill
\begin{minipage}[t]{0.46\textwidth}
\centering
\textbf{(b) Pythia ($d=128$, NVIDIA B300)}\\[2pt]
\setlength{\tabcolsep}{1.5pt}
\begin{tabular}{@{}ll rrrr r@{}}
\toprule
& & \multicolumn{4}{c}{Wall-clock time (ms)} & \multicolumn{1}{c}{Mem.\ (MiB)} \\
\cmidrule(lr){3-6}\cmidrule(lr){7-7}
Model & $N$ & \textsc{math} & \textsc{hvp-m} & \textsc{hvp-s} & \textsc{FlashBoB} & \textsc{FlashBoB} \\
\midrule
\multirow{4}{*}{160M}
  & 2\,048  & 152.4 & 161.4 & 142.7 & \textbf{115.7}  & \textbf{16\,702} \\
  & 4\,096  & 420.8 & 437.0 & 361.7 & \textbf{239.8}  & \textbf{32\,222} \\
  & 16\,384 &   OOM & 5069.5 & 4088.9 & \textbf{2064.2} & \textbf{126\,273} \\
  & 32\,768 &   OOM &   OOM &   OOM & \textbf{7114.6} & \textbf{250\,435} \\
\midrule
\multirow{4}{*}{410M}
  & 2\,048  & 368.0 & 393.8 & 328.7 & \textbf{242.9}  & \textbf{25\,992} \\
  & 4\,096  & 1074.1 & 1125.0 & 918.4 & \textbf{583.7} & \textbf{48\,974} \\
  & 8\,192  &   OOM & 4121.5 & 3121.7 & \textbf{1652.8} & \textbf{94\,937} \\
  & 16\,384 &   OOM &   OOM &   OOM & \textbf{5280.2} & \textbf{186\,863} \\
\bottomrule
\end{tabular}
\end{minipage}
\vspace{6pt}
\caption{Complete second-order steps on native GQA and MoE models at the longest sequence length completed by both methods (batch 1, BF16, one NVIDIA B300). Max-$N$ gain is from power-of-two capacity sweeps; full per-$N$ sweeps are in Appendix~\ref{app:gpt2-full}.}
\label{tab:gqa-main}
\setlength{\tabcolsep}{2.4pt}
\begin{tabular}{@{}ll r r rrr rrr c@{}}
\toprule
& & & & \multicolumn{3}{c}{Wall-clock time (ms)} & \multicolumn{3}{c}{Peak memory (GiB)} & \\
\cmidrule(lr){5-7}\cmidrule(lr){8-10}
Model & Attention ($H_Q/H_{KV}$) & $d$ & $N$ & \textsc{math} & \textsc{FlashBoB} & Speedup & \textsc{math} & \textsc{FlashBoB} & Red. & Max-$N$ gain \\
\midrule
Llama 3.2 1B         & Dense, GQA ($32/8$) &  64 & 4\,096 &   653.3 & \textbf{247.1} & $2.64\times$ & 152.5 & \textbf{39.8} & $73.9\%$ & $4\times$ \\
SmolLM2-135M         & Dense, GQA ($9/3$)  &  64 & 8\,192 & 1\,208.1 & \textbf{319.4} & $3.78\times$ & 249.2 & \textbf{27.9} & $88.8\%$ & $8\times$ \\
SmolLM2-360M         & Dense, GQA ($15/5$) &  64 & 4\,096 &   599.3 & \textbf{212.3} & $2.82\times$ & 122.7 & \textbf{24.5} & $80.1\%$ & $8\times$ \\
Qwen2.5-1.5B         & Dense, GQA ($12/2$) & 128 & 4\,096 &   568.5 & \textbf{365.4} & $1.56\times$ & 128.6 & \textbf{59.9} & $53.4\%$ & $4\times$ \\
Granite 3.1 1B-A400M & MoE, GQA ($16/8$)   &  64 & 4\,096 &   593.2 & \textbf{272.4} & $2.18\times$ & 118.5 & \textbf{38.7} & $67.3\%$ & $8\times$ \\
\bottomrule
\end{tabular}
\vspace*{-0.12in}
\end{table}

\textsc{FlashBoB} is fastest at every tested configuration. At short $N$, peak memory is dominated by non-attention activations, so \textsc{FlashBoB} and the HVP baselines use similar memory; the gap opens once quadratic attention dominates the step. On GPT-2 at the longest sequence lengths where a baseline still fits ($N{=}16{,}384$ for Small and $N{=}8{,}192$ for Medium and Large; Appendix~\ref{app:gpt2-full}), \textsc{FlashBoB} is $2.14$--$2.18\times$ faster than \textsc{hvp-s} and reduces peak memory by $28.6$--$56.0\%$, while continuing to $N{=}32{,}768$ (Small) and $N{=}16{,}384$ (Medium/Large) after all baselines exhaust memory. A sliding-window sweep appears in Appendix~\ref{app:swa}.

\textbf{Beyond GPT-2: Pythia ($d=128$) and native GQA/MoE models.}
On Pythia-160M and Pythia-410M at $d=128$ on B300 (Table~\ref{tab:gpt2-bench}, bottom blocks), \textsc{FlashBoB} achieves $1.98\times$ speedup at $N{=}16{,}384$ (160M) and $1.89\times$ at $N{=}8{,}192$ (410M) over \textsc{hvp-s}, and alone reaches $N{=}32{,}768$ and $N{=}16{,}384$. On the five native GQA and MoE models on B300 (Table~\ref{tab:gqa-main}), \textsc{FlashBoB} operates directly on the native $H_Q/H_{KV}$ layout without repeating key/value heads (whereas \textsc{hvp-m/s} require head replication and FlashBack supports only MQA with $H_{KV}=1$). At the largest shared sequence length for each model, \textsc{FlashBoB} is $1.56$--$3.78\times$ faster than PyTorch \textsc{math}, reduces peak memory by $53.4$--$88.8\%$, and reaches $4$--$8\times$ longer maximum sequence lengths in power-of-two capacity sweeps (SmolLM2 runs past $8{,}192$ tokens use RoPE extrapolation and test hardware capacity only).

\vspace*{-0.05in}
\subsection{Training on FineWeb}
\label{sec:fineweb}
\vspace*{-0.05in}

We train GPT-2 Small on 1B tokens from FineWeb at $N{=}2048$ on
$8{\times}$NVIDIA A6000 48GB, with per-GPU batch size $2$ and gradient
accumulation $16$ (effective batch $256$, $1{,}908$ optimizer steps).
Following~\citep{liu2024sophia}, we estimate the Hessian diagonal once
every ten steps. We compare AdamW, Sophia-H with the PyTorch math
reference, and Sophia-H with \textsc{FlashBoB}. The AdamW curve is the
best-performing AdamW reference from a small learning-rate sweep, while
the two Sophia-H runs use matched optimizer, model, data-streaming,
precision, seed, batch, and Hessian-estimation settings; the only
intended difference between them is the attention BoB backend used
during Hessian estimation. Full training hyperparameters are given in
Appendix~\ref{app:fineweb-training-repro}.

PyTorch's \texttt{scaled\_dot\_product\_attention} does not support
double backward in any accelerated backend, so an exact second-order
pipeline using the standard PyTorch path must fall back to the
\texttt{math} backend for the attention path. Because the graph must remain compatible with
double backward, the math-reference Sophia-H run loses the optimized
attention path even on ordinary gradient steps. In contrast,
\textsc{FlashBoB} preserves the optimized forward and first backward and
replaces only the attention double backward, reducing per-step non-Hessian time from ${\approx}4.9$\,s to ${\approx}2.3$\,s.

\begin{table}[t]
\centering\scriptsize
\caption{GPT-2 Small training on 1B FineWeb tokens at $N=2048$ ($1{,}908$ steps, effective batch 256). \textbf{(a)} Standard PyTorch math-reference versus \textsc{FlashBoB} Sophia-H with Hessian estimation every 10 steps on $8{\times}$A6000 48GB. \textbf{(b)} Matched hybrid baseline (FlashAttention forward and first backward in both arms; only the attention BoB differs) across Hessian frequencies on $8{\times}$A100 80GB.}
\label{tab:fineweb-step}
\label{tab:hybrid}
\begin{minipage}[t]{0.54\textwidth}
\centering
\textbf{(a) Standard vs.\ \textsc{FlashBoB} Sophia-H ($8{\times}$A6000)}\\[2pt]
\setlength{\tabcolsep}{3.5pt}
\begin{tabular}{@{}l rr r@{}}
\toprule
Metric & Math ref. & \textsc{FlashBoB} & Ratio \\
\midrule
Hessian estimation / event (ms) & 13{,}500 & \textbf{5{,}500}  & $2.45\times$ \\
Total Hessian time (min)        &     42.8 & \textbf{17.4}     & $2.45\times$ \\
Total non-Hessian time (hr)     &     2.59 & \textbf{1.23}     & $2.10\times$ \\
Full training wall time (hr)    &     3.30 & \textbf{1.52}     & $2.17\times$ \\
Peak memory per GPU (MiB)       & 30{,}644 & \textbf{13{,}830} & $2.22\times$ \\
\bottomrule
\end{tabular}
\end{minipage}\hfill
\begin{minipage}[t]{0.44\textwidth}
\centering
\textbf{(b) Matched hybrid baseline ($8{\times}$A100 80GB)}\\[2pt]
\setlength{\tabcolsep}{3.5pt}
\begin{tabular}{@{}l rr r@{}}
\toprule
Hessian interval & Hybrid (hr) & \textsc{FlashBoB} (hr) & Speedup \\
\midrule
Every 10 steps & 0.77 & \textbf{0.63} & $1.22\times$ \\
Every 5 steps  & 1.12 & \textbf{0.79} & $1.42\times$ \\
Every 2 steps  & 2.11 & \textbf{1.24} & $1.70\times$ \\
\bottomrule
\end{tabular}
\end{minipage}
\vspace*{-0.12in}
\end{table}

\begin{figure}[t]
\centering
\includegraphics[width=0.92\textwidth]{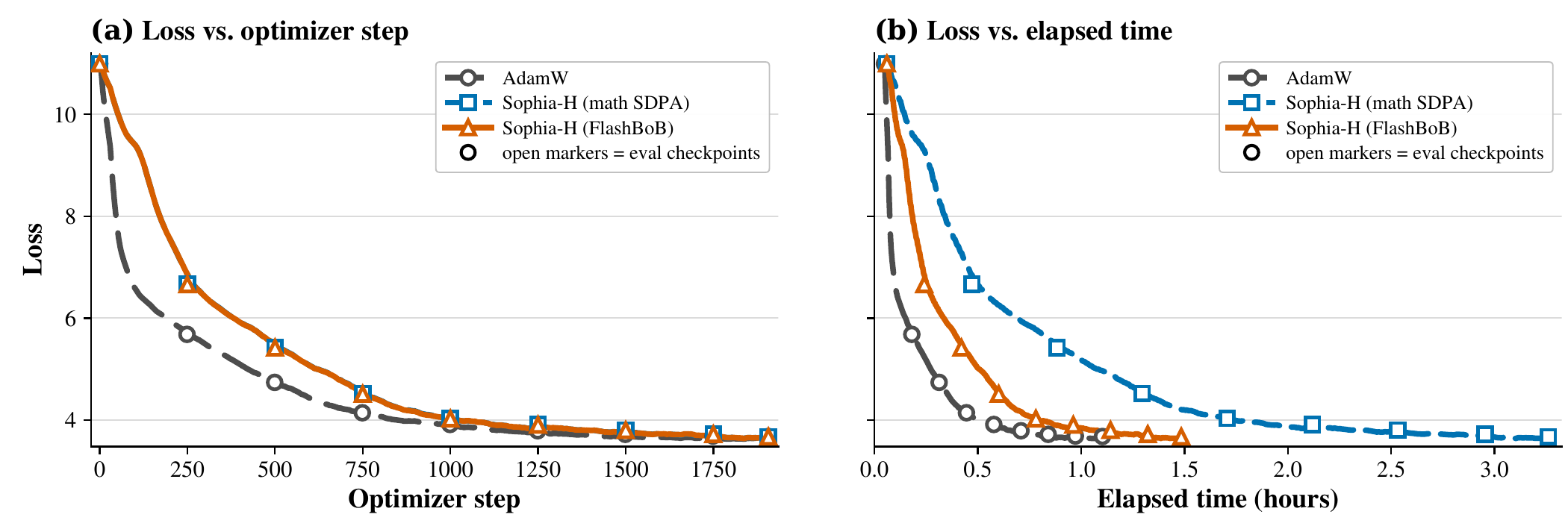}
\vspace*{-0.18in}
\caption{Single-seed systems parity check. Final losses are $3.6759$ for
AdamW, $3.6674$ for Sophia-H with the PyTorch math backend, and $3.6662$
for Sophia-H with \textsc{FlashBoB}; the elapsed-time panel shows that
\textsc{FlashBoB} preserves the Sophia-H trajectory while moving exact
second-order training closer to AdamW in time-to-loss.}
\label{fig:fineweb_parity}
\vspace*{-0.15in}
\end{figure}

\textsc{FlashBoB} accelerates the Hessian-estimation event by
$2.45\times$ and the full training loop by $2.17\times$ end to
end, while reducing peak memory by $2.22\times$ (from 30\,GB to 14\,GB per GPU). The two
Sophia-H loss curves nearly overlap throughout training
(Figure~\ref{fig:fineweb_parity}), with matched final losses of $3.6674$ (PyTorch math) and
$3.6662$ (\textsc{FlashBoB}), compared with $3.6759$ for AdamW. In elapsed time, \textsc{FlashBoB} shifts the Sophia-H curve substantially left toward the AdamW time-to-loss curve while preserving the final Sophia-H loss.

\textbf{Separating kernel and integration gains.}
The $2.17\times$ full-loop speedup combines a faster BoB kernel on Hessian steps with retaining FlashAttention on ordinary gradient steps. To isolate the kernel contribution, we implemented a matched hybrid baseline that uses FlashAttention for the forward and first backward on all steps and differs only in the attention BoB implementation used during Hessian estimation (PyTorch \texttt{math} versus \textsc{FlashBoB}). Across six matched 1B-token GPT-2 Small runs on $8{\times}$A100 80GB with Hessian estimation every 10, 5, and 2 steps (Table~\ref{tab:hybrid}, right), the hybrid baseline requires $0.77$, $1.12$, and $2.11$~hours, whereas \textsc{FlashBoB} requires $0.63$, $0.79$, and $1.24$~hours ($1.22\times$, $1.42\times$, and $1.70\times$ end-to-end speedup; See Appendix~\ref{app:fineweb-training-repro}).

\vspace*{-0.1in}
\section{Discussion}
\label{sec:discussion}
\vspace*{-0.06in}

Long-context learning methods such as second-order optimization,
test-time training, gradient-written memory, and meta-learning
differentiate through gradient computations. When these computations
contain softmax attention, they require backward-over-backward (BoB).
\textsc{FlashBoB} extends FlashAttention's tile-streaming execution
model to this operation, computing exact BoB without materializing
$N\times N$ intermediates. Across complete-model second-order steps,
it achieves $1.6$--$3.8\times$ speedups over the available exact
baselines and supports $2$--$8\times$ longer sequences in the tested
memory-capacity sweeps. The $N=262{,}144$ result is an isolated
attention-layer measurement on one A100 80GB GPU, rather than a full-model training context. On FineWeb, \textsc{FlashBoB} preserves
the Sophia-H loss trajectory while reducing training wall time by
$2.17\times$ and peak memory by $2.22\times$ relative to the PyTorch
math-backend pipeline. The matched hybrid baseline shows
$1.22$--$1.70\times$ end-to-end gains when only the attention BoB
implementation changes. Evaluating these benefits within test-time
training and gradient-written memory remains future work. It also
makes a new design point practical: attention-aware post-training
quantization based on exact Hessian-vector
products~\citep{qbert,hawqv2}, where existing methods use
approximations or layer-local surrogates~\citep{adaround,aptq,boa};
we have not evaluated this direction.

\textbf{Limitations.}
Our Triton implementation does not yet use architecture-specific
asynchronous pipelines, warp specialization, or tensor-core scheduling
for Hopper and Blackwell. The released kernels support standard
multi-head attention and native grouped-query and multi-query attention
(GQA/MQA). The GQA/MQA path currently supports BF16,
$N_Q=N_{KV}$, $H_Q/H_{KV}\le 32$, and causal masking without sliding
windows; rectangular attention is supported for multi-head attention.
Extensions to flexible attention APIs and other attention
normalizations~\citep{sinkformers} require additional implementation or derivation.
We view the two-pass algorithm as a foundation for these
specializations.

\vspace*{-0.1in}
\begin{ack}
We are especially grateful to Modal Labs for their generous provision of GPU compute resources and support throughout this research. Their contribution was invaluable in enabling the extensive experimental evaluation of \textsc{FlashBoB}, including
the additional benchmarks and validation conducted during the review process. We sincerely appreciate Modal's generosity and commitment to supporting academic research. This work is supported by NSF grants \#2436217, \#2425687, \#2601681.
\end{ack}

\bibliographystyle{plainnat}
\bibliography{neurips_2026}

\appendix

\section{Notation}
\label{app:notation}

We collect the main symbols used throughout the paper in
Table~\ref{tab:notation}. Shapes use the square, single-head attention
setting of Section~\ref{sec:prelim}; batch and head dimensions are omitted.
Rectangular attention and grouped-query attention are described in
Sections~\ref{sec:prelim} and~\ref{sec:algo}. We use $t_i$ for the
$i$-th row of a matrix $T$; row scalars $L$, $D$, $\alpha$, and $E$
are indexed by query row. For an input or intermediate $X$,
$\td{X}$ denotes $\partial\Phi/\partial X$, where $\Phi$ is the BoB
scalar defined in~\eqref{eq:bob-obj}.

\begin{table}[h]
\centering\small
\caption{Notation used throughout the paper. Tile-level intermediates
($F, C, \td{P}, \td{S}$) are introduced in Section~\ref{sec:flashbob}.}
\label{tab:notation}
\begin{tabular}{@{}lll@{}}
\toprule
Symbol & Shape & Meaning \\
\midrule
$Q, K, V, dO$ & $\R^{N \times d}$ & Forward inputs, first-backward upstream \\
$B_{\text{mask}}$ & $\{0,-\infty\}^{N \times N}$ & Additive attention mask \\
$O$ & $\R^{N \times d}$ & Forward output, $O = PV$ \\
$L$ & $\R^{N}$ & Saved row log-normalizer, $L_i = \log \sum_{j \in \mathcal{V}_i} \exp(S_{ij})$ \\
$dQ, dK, dV$ & $\R^{N \times d}$ & First-backward outputs \\
$D$ & $\R^{N}$ & Saved first-backward row summary, $D_i = do_i^\Tr o_i$ \\
$U_Q, U_K, U_V$ & $\R^{N \times d}$ & BoB input matrices \\
$\td{Q}, \td{K}, \td{V}, \td{dO}$ & $\R^{N \times d}$ & BoB outputs \\
$\alpha, E$ & $\R^{N}$ & BoB row scalars (Section~\ref{sec:flashbob}) \\
$M$ & scalar & On-chip SRAM capacity in scalar elements \\
$B_r, B_c$, $T_r, T_c$ & scalar & Row/column tile sizes and tile counts \\
\bottomrule
\end{tabular}
\end{table}

\section{Software environment}
\label{app:env}

Table~\ref{tab:software-stack} reports the software stack used for the
experiments. We report the CUDA version from the PyTorch/kernel stack
rather than the maximum CUDA compatibility version reported by
\texttt{nvidia-smi}. We separately report the \texttt{nvcc} compiler
release and the \texttt{nvidia-smi} CUDA compatibility version because
these quantities can differ from \texttt{torch.version.cuda}. The A6000
experiments used Python 3.10.20, while the A100, B200, and B300 experiments used
Python 3.10.13.

\textbf{Measurement protocol.}
Unless noted otherwise, kernel and model benchmarks run 25 warmup iterations followed by 100 CUDA-synchronized timed iterations and report the median CUDA-event latency ($p_{50}$). Sliding-window experiments use \texttt{triton.testing.do\_bench} with \texttt{warmup=25} and \texttt{rep=100}, where both arguments specify millisecond budgets rather than fixed iteration counts; entries written as \(x\pm\delta\) report the median and half the 20th-to-80th percentile range,
\[
x = p_{50},
\qquad
\delta = \frac{1}{2}(p_{80}-p_{20}),
\]
rather than a standard deviation. Multi-GPU training results on $8{\times}$A6000 and $8{\times}$A100 report full end-to-end wall-clock time; full-model and training timings are single-run measurements and do not estimate run-to-run variation.

Peak memory is measured via \texttt{torch.cuda.max\_memory\_allocated} and therefore includes all tensors held live during the benchmark, not only kernel workspace. In particular, the $d=64$ A100 sweep in Table~\ref{tab:attn-kernel} retains the \textsc{math} reference outputs for $N \le 4{,}096$, whereas the head-dimension sweep in Appendix~\ref{app:head-dim-sweep} releases them before timing.

\begin{table}[t]
\centering
\small
\caption{Software stack used for experiments. Under CUDA versions, Runtime denotes \texttt{torch.version.cuda}, \texttt{nvcc} denotes the compiler release used for custom kernels, and \texttt{smi} denotes the maximum CUDA version reported by \texttt{nvidia-smi}. N/A indicates that standalone \texttt{flash-attn} was not required.}
\label{tab:software-stack}
\setlength{\tabcolsep}{4.5pt}
\begin{tabular}{@{}l l c c c c c c@{}}
\toprule
 & \multicolumn{3}{c}{Python packages} & \multicolumn{3}{c}{CUDA versions} & \\
\cmidrule(lr){2-4} \cmidrule(lr){5-7}
Hardware & PyTorch & Triton & \texttt{flash-attn} & Runtime & \texttt{nvcc} & \texttt{smi} & Driver \\
\midrule
A6000 ($1\times$ / $8\times$)     & 2.9.1+cu128        & 3.5.1 & 2.8.4 & 12.8 & 12.8 & 13.0 & 580.76.05 \\
\addlinespace[2pt]
A100 80GB ($1\times$ / $8\times$) & 2.11.0.dev20260123 & 3.6.0 & 2.8.3 & 12.6 & 12.6 & 12.6 & 560.35.05 \\
                                  & \quad +cu126       &       &       &      &      &      &           \\
\addlinespace[2pt]
B200 192GB                        & 2.11.0+cu130       & 3.6.0 & N/A   & 13.0 & 13.1 & 13.2 & 595.58.03 \\
B300                              & 2.11.0+cu130       & 3.6.0 & N/A   & 13.0 & 13.1 & 13.2 & 595.58.03 \\
\bottomrule
\end{tabular}
\end{table}

FlashBack additionally used JAX 0.6.2 and \texttt{jaxlib} 0.6.2;
we do not compare its JAX/XLA peak-memory measurements directly with
PyTorch's allocated-memory statistic
(Appendix~\ref{app:flashback-comparison}). The main causal-attention
experiments use PyTorch's fused scaled-dot-product attention (SDPA)
for the forward and first backward; \textsc{FlashBoB} replaces only
the double backward. SDPA need not select the same underlying fused
kernel for every GPU and input shape. The standalone \texttt{flash-attn}
package was used for the sliding-window experiments on A100
(Appendix~\ref{app:swa}).

The GradMem reference implementations, \textsc{hvp-m} and
\textsc{hvp-s}, were evaluated at commit \texttt{3f40f11}. FlashBack was evaluated at commit \texttt{62a0e0e} with its own autotuned
configurations. The Sophia-H baseline was evaluated at commit
\texttt{738c32b}, the last commit in which Sophia-H was present. Random seeds were fixed per benchmark. Between configurations, we called \texttt{torch.cuda.empty\_cache()} and reset peak memory statistics before each trial. Measurements used the same PyTorch CUDA allocator configuration across compared methods unless otherwise noted.

\subsection{Example: Hessian-vector products through attention}
\label{app:usage}

A common backward-over-backward workload computes a Hessian-vector
product (HVP) of a model loss $\mathcal L$ with respect to its parameters
$\theta$. Given a direction $z$ of the same shape as $\theta$, we first
compute $g=\nabla_\theta\mathcal L$ with an autograd graph and then
compute $H_\theta z=\nabla_\theta\langle g,z\rangle$. This differentiates
through the attention backward without constructing the Hessian.
Within the model, replace the scaled-dot-product attention call with
\texttt{sdpa\_bob}; all other model operations remain unchanged:
\begin{verbatim}
import torch
from bob import sdpa_bob

# Inside the model's attention block (q, k, v: [B, H, N, d]):
y = sdpa_bob(q, k, v, is_causal=True)

# Inside a curvature-estimation step for a model using this block:
params = tuple(p for p in model.parameters() if p.requires_grad)
loss = loss_fn(model(inputs), targets)
g = torch.autograd.grad(loss, params, create_graph=True)
z = [torch.randn_like(p) for p in params]
phi = sum((gi.float() * zi.float()).sum() for gi, zi in zip(g, z))
Hv = torch.autograd.grad(phi, params)
\end{verbatim}
Here \texttt{model}, \texttt{inputs}, \texttt{targets}, and
\texttt{loss\_fn} are supplied by the training program; the model's
attention modules call \texttt{sdpa\_bob}. The returned \texttt{Hv}
contains one Hessian-vector product for each parameter tensor. A
Hutchinson-style diagonal estimate can further use the elementwise
product $z\odot(H_\theta z)$ when $z$ has independent, zero-mean,
unit-variance entries. This is the curvature primitive underlying
second-order optimizer examples, not a complete Sophia-H training step.
Our repository provides a runnable model-parameter HVP example at
\url{https://github.com/MingruiLiu-ML-Lab/FlashBoB/blob/main/examples/hessian_vector_product.py}.

\subsection{Software and dataset licenses}
\label{app:licenses}

We acknowledge the following software, models, and datasets used in
our experiments or as baselines. The listed licenses describe the
referenced distributions; code, model weights, and datasets may have
different terms. Consult the upstream releases for their applicable
licenses and conditions.
\begin{itemize}[leftmargin=*, itemsep=2pt, topsep=2pt]
\item \textbf{PyTorch} (BSD-3-Clause), \texttt{https://github.com/pytorch/pytorch}.
\item \textbf{Triton} (MIT), \texttt{https://github.com/triton-lang/triton}.
\item \textbf{flash-attn} (BSD-3-Clause), \texttt{https://github.com/Dao-AILab/flash-attention}.
\item \textbf{JAX} and \textbf{jaxlib} (Apache-2.0), \texttt{https://github.com/jax-ml/jax}.
\item \textbf{FineWeb-derived training data}. The training runs use
\texttt{VisionTheta/fineweb-1B}, a FineWeb-derived 1B-token dataset stream.
The underlying FineWeb dataset~\citep{penedo2024fineweb} is distributed under
the Open Data Commons Attribution License v1.0 (ODC-By) at
\texttt{https://huggingface.co/datasets/HuggingFaceFW/fineweb}; use is also
subject to CommonCrawl's Terms of Use.
\item \textbf{GPT-2}, \textbf{Pythia}, \textbf{Llama 3.2}, \textbf{SmolLM2}, \textbf{Qwen2.5}, and \textbf{Granite 3.1} model architectures and configurations: GPT-2 (MIT), Pythia (Apache-2.0), Llama 3.2 (Llama 3.2 Community License), SmolLM2 (Apache-2.0), Qwen2.5 (Apache-2.0), and Granite 3.1 (Apache-2.0).
\item \textbf{Sophia} reference implementation~\citep{liu2024sophia}
(\texttt{https://github.com/Liuhong99/Sophia}). The repository did not
display a top-level LICENSE file at commit \texttt{738c32b}; we used
the code only as a research baseline. The repository documents that
its training code is derived from nanoGPT (MIT) and levanter (Apache-2.0).
\item \textbf{GradMem} reference HVP kernels~\citep{kuratov2026gradmem},
evaluated at commit \texttt{3f40f11}. Used solely as a research
baseline under the terms of the upstream repository.
\item \textbf{FlashBack}~\citep{engstrom2024flashback}
(\texttt{https://github.com/lengstrom/flashback}), evaluated at commit
\texttt{62a0e0e}. The repository did not display a top-level LICENSE
file at this commit; we used the code only as a research baseline.
\item \textbf{NVIDIA CUDA Toolkit and drivers}: used under the NVIDIA
CUDA Toolkit End User License Agreement.
\end{itemize}

\section{Full BoB derivation}
\label{app:derivation}

This appendix derives every BoB identity in \S\ref{sec:affine}
from reverse-mode automatic differentiation applied to the
first-backward graph of softmax attention. The derivation is
deliberately step-by-step: each node is processed in turn, every local
Jacobian is computed from scratch from its defining formula, and every
chain-rule sum is written explicitly before being simplified. Readers
familiar with the chain rule and the multi-child rule have all the
prerequisites.

\subsection*{Notation summary}

Throughout this appendix we use the following two conventions in
parallel. For any node $X$ in the first-backward graph, $dX = \partial
L / \partial X$ is its first-order gradient (already computed by the
ordinary attention backward), and $\td{X} = \partial \Phi / \partial
X$ is its BoB gradient (what we are computing here). The first-order
gradients $dQ, dK, dV, dP, dS, dO$ together with the saved row scalar
$D$ are inputs to BoB; the BoB gradients
$\td{Q}, \td{K}, \td{V}, \td{dO}, \td{P}, \td{S}, \td{dP},
\td{dS}$ are intermediates or outputs.

The first-backward outputs $dQ, dK, dV$ are the entry points of the
BoB walk; reverse-mode begins with their BoB gradients
\begin{equation}
\label{eq:bob-seeds}
\td{dQ} = U_Q, \qquad \td{dK} = U_K, \qquad \td{dV} = U_V,
\end{equation}
which follow directly from differentiating the BoB scalar
$\Phi = \langle dQ, U_Q\rangle + \langle dK, U_K\rangle + \langle dV, U_V\rangle$
of \eqref{eq:bob-obj}.

For matrices $A, B$ of the same shape,
$\langle A, B\rangle = \sum_{i,j} A_{ij} B_{ij}$ is the Frobenius inner
product. Lowercase letters denote row slices: $q_i, k_j, v_j, do_i,
u_{Q,i}, u_{K,j}, u_{V,j} \in \R^d$. Whenever a row vector
$D, \alpha, E \in \R^N$ appears inside an $N \times N$ expression, it
is broadcast across columns.

\subsection*{Reverse-mode chain rule}

The only fact used in this appendix is the chain rule for a node $X$
with children $Y^{(1)}, \ldots, Y^{(m)}$ in the first-backward graph
(meaning that each $Y^{(r)}$ is computed from $X$ during the
first-backward pass): the BoB gradient at $X$ is the sum of
contributions through each child, with each contribution given by the
local Jacobian acting on the BoB gradient at the child. In components,
\begin{equation}
\label{eq:multichild}
\td{X_{ij}} \;=\; \sum_{r=1}^m \sum_{u, v} \td{Y^{(r)}_{uv}}\, \frac{\partial Y^{(r)}_{uv}}{\partial X_{ij}}.
\end{equation}
Every equation that follows is one application of \eqref{eq:multichild}.

\subsection*{Step 1: \texorpdfstring{$F = \td{dS}$}{F = tilde dS}}

The node $dS$ has two children in the first-backward graph,
$dQ = \tau\, dS\, K$ and $dK = \tau\, dS^\Tr Q$. Applying
\eqref{eq:multichild},
\begin{equation}
\td{dS_{ij}}
\;=\; \sum_a \td{dQ_{ia}}\, \frac{\partial dQ_{ia}}{\partial dS_{ij}}
\;+\; \sum_a \td{dK_{ja}}\, \frac{\partial dK_{ja}}{\partial dS_{ij}}.
\end{equation}
The local Jacobians come from the formulas $dQ_{ia} = \tau \sum_j dS_{ij} K_{ja}$
and $dK_{ja} = \tau \sum_i dS_{ij} Q_{ia}$, giving
\begin{equation}
\frac{\partial dQ_{ia}}{\partial dS_{ij}} = \tau\, K_{ja},
\qquad
\frac{\partial dK_{ja}}{\partial dS_{ij}} = \tau\, Q_{ia}.
\end{equation}
Substituting and using $\td{dQ_{ia}} = U_{Q,ia}$ and
$\td{dK_{ja}} = U_{K,ja}$ from \eqref{eq:bob-seeds},
\begin{equation}
\td{dS_{ij}}
\;=\; \tau \sum_a U_{Q,ia}\, K_{ja} \;+\; \tau \sum_a U_{K,ja}\, Q_{ia}
\;=\; \tau\, u_{Q,i}^\Tr k_j \;+\; \tau\, q_i^\Tr u_{K,j}.
\end{equation}
We name this quantity $F$. In matrix form,
\begin{equation}
\label{eq:F-app}
F \;:=\; \td{dS} \;=\; \tau\bigl(U_Q\, K^\Tr + Q\, U_K^\Tr\bigr).
\end{equation}

\subsection*{Step 2: \texorpdfstring{$\td{dP}$ and the row scalars $\alpha$}{tilde dP and the row scalars alpha}}

The node $dP$ has children in the first-backward graph through the
formula $dS_{ik} = P_{ik}(dP_{ik} - D_i)$ together with
$D_i = \sum_r P_{ir}\, dP_{ir}$. Holding $P$ fixed (its BoB gradient is
handled separately in Steps 4 and 5), $dP_{ij}$ enters $dS_{ik}$
through two routes: a direct route ($dP_{ij}$ appears as $dP_{ik}$ when
$k = j$) and an indirect route through the row scalar $D_i$ (every $k$
in row $i$ uses $D_i$, which depends on $dP_{ij}$). Both contribute to
the chain rule, so we compute the local Jacobian carefully:
\begin{equation}
\frac{\partial dS_{ik}}{\partial dP_{ij}}
\;=\; P_{ik}\!\left(\frac{\partial dP_{ik}}{\partial dP_{ij}} - \frac{\partial D_i}{\partial dP_{ij}}\right).
\end{equation}
The two pieces are
\begin{equation}
\frac{\partial dP_{ik}}{\partial dP_{ij}} = \delta_{kj},
\qquad
\frac{\partial D_i}{\partial dP_{ij}} = \frac{\partial}{\partial dP_{ij}} \sum_r P_{ir}\, dP_{ir} = P_{ij},
\end{equation}
so
\begin{equation}
\label{eq:dS-by-dP}
\frac{\partial dS_{ik}}{\partial dP_{ij}} \;=\; P_{ik}(\delta_{kj} - P_{ij}).
\end{equation}
Applying \eqref{eq:multichild} with the single child $dS$ and
$\td{dS_{ik}} = F_{ik}$,
\begin{equation}
\td{dP_{ij}}
\;=\; \sum_k F_{ik}\, P_{ik}(\delta_{kj} - P_{ij})
\;=\; \underbrace{F_{ij}\, P_{ij}}_{k = j \text{ term}}
\;-\; \underbrace{P_{ij} \sum_k P_{ik} F_{ik}}_{\text{row-sum term}}.
\end{equation}
The row-sum factor $\sum_k P_{ik} F_{ik}$ is a row reduction; we name
it $\alpha_i$:
\begin{equation}
\label{eq:alpha-app}
\alpha_i \;:=\; \sum_{k \in \mathcal{V}_i} P_{ik}\, F_{ik}.
\end{equation}
Substituting,
\begin{equation}
\label{eq:dPtilde-app}
\td{dP_{ij}} \;=\; P_{ij}(F_{ij} - \alpha_i).
\end{equation}

\subsection*{Step 3: \texorpdfstring{$\td{V}$ and the first contribution to $\td{dO}$}{tilde V and the first contribution to tilde dO}}

The node $V$ has only one child in the first-backward graph,
$dP = dO\, V^\Tr$; that is, $dP_{ij} = \sum_a dO_{ia}\, V_{ja}$. Applying
\eqref{eq:multichild} for the BoB gradient at $V_{ja}$,
\begin{equation}
\td{V_{ja}}
\;=\; \sum_i \td{dP_{ij}}\, \frac{\partial dP_{ij}}{\partial V_{ja}}
\;=\; \sum_i \td{dP_{ij}}\, dO_{ia}
\;=\; \sum_i P_{ij}(F_{ij} - \alpha_i)\, dO_{ia}.
\end{equation}
In matrix form, with $h := F - \alpha\, \mathbf{1}^\Tr$ broadcast over
columns,
\begin{equation}
\label{eq:Vtilde-app}
\td{V} \;=\; (P \odot h)^\Tr\, dO.
\end{equation}

The same node $dP = dO\, V^\Tr$ also sends one of two contributions
into $\td{dO}$. For the BoB gradient at $dO_{ia}$,
\begin{equation}
\frac{\partial dP_{ij}}{\partial dO_{ia}} = V_{ja},
\qquad\text{so the contribution from $dP$ is}\quad
\sum_j \td{dP_{ij}}\, V_{ja} = \sum_j P_{ij}(F_{ij} - \alpha_i)\, V_{ja}.
\end{equation}
We will combine this with the second contribution to $\td{dO}$ in
Step 4.

\subsection*{Step 4: \texorpdfstring{$C$, the contribution to $\td{P}$ from the $dV$ child, and the rest of $\td{dO}$}{C, the contribution to tilde P from the dV child, and the rest of tilde dO}}

The node $P$ has two children in the first-backward graph, $dV$ and
$dS$. We process the $dV$ child here (it is algebraically simple) and
defer the $dS$ child to Step 5 (where the second row scalar will
appear).

The first-backward formula is $dV_{ja} = \sum_i P_{ij}\, dO_{ia}$. The
local Jacobian is $\partial dV_{ja}/\partial P_{ij} = dO_{ia}$, and the
BoB gradient at $dV$ is $\td{dV_{ja}} = U_{V,ja}$ from
\eqref{eq:bob-seeds}. The $dV$-child contribution to $\td{P_{ij}}$ is
\begin{equation}
\sum_a U_{V,ja}\, dO_{ia} \;=\; dO_i^\Tr u_{V,j}.
\end{equation}
We name this quantity $C_{ij}$. In matrix form,
\begin{equation}
\label{eq:C-app}
C \;:=\; dO\, U_V^\Tr.
\end{equation}

The same node $dV = P^\Tr dO$ also sends the second contribution into
$\td{dO}$. For the BoB gradient at $dO_{ia}$,
\begin{equation}
\frac{\partial dV_{ja}}{\partial dO_{ia}} = P_{ij},
\qquad\text{so the contribution from $dV$ is}\quad
\sum_j P_{ij}\, U_{V,ja}.
\end{equation}
Combining with the contribution from Step 3,
\begin{equation}
\label{eq:dOtilde-app}
\td{dO_{ia}}
\;=\; \sum_j P_{ij}\, U_{V,ja} \;+\; \sum_j P_{ij}(F_{ij} - \alpha_i)\, V_{ja},
\quad\text{or}\quad
\td{dO} = P\, U_V + (P \odot h)\, V.
\end{equation}

\subsection*{Step 5: contribution to \texorpdfstring{$\td{P}$ from the $dS$ child}{tilde P from the dS child}}

This is the second contribution to $\td{P}$. The variable $P_{ij}$
reaches $dS_{ik}$ along two routes through the formula
$dS_{ik} = P_{ik}(dP_{ik} - D_i)$: the outer factor $P_{ik}$ (when
$k = j$) and the row scalar $D_i$ (which depends on every $P_{ir}$).
Computing the local Jacobian by the product rule, with $dP$ now held
fixed,
\begin{align}
\frac{\partial dS_{ik}}{\partial P_{ij}}
\;&=\; \frac{\partial P_{ik}}{\partial P_{ij}}(dP_{ik} - D_i)
       \;+\; P_{ik}\, \frac{\partial(dP_{ik} - D_i)}{\partial P_{ij}}\\
\;&=\; \delta_{kj}(dP_{ik} - D_i) \;-\; P_{ik}\, dP_{ij},
\end{align}
where in the second line we used $\partial P_{ik}/\partial P_{ij} = \delta_{kj}$
and $\partial D_i/\partial P_{ij} = dP_{ij}$. Multiplying by
$\td{dS_{ik}} = F_{ik}$ and summing over $k$,
\begin{align}
\sum_k F_{ik}\, \frac{\partial dS_{ik}}{\partial P_{ij}}
\;&=\; \underbrace{\sum_k F_{ik}\, \delta_{kj}(dP_{ij} - D_i)}_{\text{outer-factor route}}
       \;-\; \underbrace{\sum_k F_{ik}\, P_{ik}\, dP_{ij}}_{\text{row-scalar route through }D_i}\\
\;&=\; F_{ij}(dP_{ij} - D_i) \;-\; \alpha_i\, dP_{ij},
\end{align}
where the last step used the $\alpha_i$ definition from
\eqref{eq:alpha-app}. Adding the $C_{ij}$ contribution from
\eqref{eq:C-app} closes the BoB gradient at $P$:
\begin{equation}
\label{eq:Ptilde-app}
\td{P_{ij}}
\;=\; \underbrace{C_{ij} + F_{ij}(dP_{ij} - D_i)}_{\td{P_{ij}}^{\circ}}
\;-\; \alpha_i\, dP_{ij}.
\end{equation}
Equation~\eqref{eq:Ptilde-app} also defines the $\alpha$-free part
$\td{P_{ij}}^{\circ}$ of $\td{P_{ij}}$.

\subsection*{Step 6: \texorpdfstring{$\td{S}$ and the row scalars $E$}{tilde $S$ and the row scalars $E$}}

The node $S$ has one child, $P = \softmax(S)$. We first derive the
softmax Jacobian from scratch by the quotient rule, then apply
\eqref{eq:multichild}.

For row $i$, $P_{ik} = \exp(S_{ik}) / Z_i$ with $Z_i = \sum_r \exp(S_{ir})$.
By the quotient rule,
\begin{align}
\frac{\partial P_{ik}}{\partial S_{ij}}
\;&=\; \frac{[\partial \exp(S_{ik}) / \partial S_{ij}]\, Z_i \;-\; \exp(S_{ik})\, [\partial Z_i / \partial S_{ij}]}{Z_i^2}.
\end{align}
The two numerator pieces are
$\partial \exp(S_{ik})/\partial S_{ij} = \delta_{kj} \exp(S_{ik})$ and
$\partial Z_i / \partial S_{ij} = \exp(S_{ij})$, so
\begin{equation}
\label{eq:softmax-jac}
\frac{\partial P_{ik}}{\partial S_{ij}}
\;=\; \frac{\delta_{kj} \exp(S_{ik})}{Z_i} \;-\; \frac{\exp(S_{ik}) \exp(S_{ij})}{Z_i^2}
\;=\; P_{ik}(\delta_{kj} - P_{ij}).
\end{equation}
Applying \eqref{eq:multichild} at $S_{ij}$ with the single child $P$,
\begin{equation}
\td{S_{ij}}
\;=\; \sum_k \td{P_{ik}}\, \frac{\partial P_{ik}}{\partial S_{ij}}
\;=\; \sum_k \td{P_{ik}}\, P_{ik}(\delta_{kj} - P_{ij})
\;=\; P_{ij}\, \td{P_{ij}} \;-\; P_{ij} \sum_k P_{ik}\, \td{P_{ik}}.
\end{equation}
The row-sum factor is the second row scalar:
\begin{equation}
\label{eq:E-app}
E_i \;:=\; \sum_{k \in \mathcal{V}_i} P_{ik}\, \td{P_{ik}}.
\end{equation}
Substituting,
\begin{equation}
\label{eq:Stilde-app}
\td{S_{ij}} \;=\; P_{ij}(\td{P_{ij}} - E_i),
\end{equation}
the same shape as the first-backward formula
$dS_{ij} = P_{ij}(dP_{ij} - D_i)$.

\subsection*{Step 7: \texorpdfstring{$\td{Q}$ and $\td{K}$}{tilde Q and tilde K}}

The node $Q$ has two children in the first-backward graph: $S$
(through the indirect path $Q \to S \to P \to \cdots$) and $dK$
(through the direct path, since $dK = \tau\, dS^\Tr Q$ uses $Q$
directly). The contribution from each child is computed below.

\paragraph{Indirect contribution through $S$.} The forward
formula is $S_{ij} = \tau\, q_i^\Tr k_j + (B_{\text{mask}})_{ij}$, so the local
Jacobian is $\partial S_{ij}/\partial Q_{ia} = \tau\, K_{ja}$. The
contribution to $\td{Q_{ia}}$ from $S$ is
\begin{equation}
\sum_j \td{S_{ij}}\, \tau\, K_{ja} \;=\; \tau \sum_j \td{S_{ij}}\, K_{ja},
\quad\text{or in matrix form,}\quad
\tau\, \td{S}\, K.
\end{equation}

\paragraph{Direct contribution through $dK$.} The first-backward
formula is $dK_{ja} = \tau \sum_i dS_{ij}\, Q_{ia}$, so the local
Jacobian is $\partial dK_{ja}/\partial Q_{ib} = \tau\, dS_{ij}\, \delta_{ab}$.
The contribution to $\td{Q_{ib}}$ from $dK$ is
\begin{equation}
\sum_{j, a} \td{dK_{ja}}\, \tau\, dS_{ij}\, \delta_{ab}
\;=\; \tau \sum_j dS_{ij}\, U_{K,jb},
\quad\text{or in matrix form,}\quad
\tau\, dS\, U_K.
\end{equation}

Adding the two contributions,
\begin{equation}
\label{eq:Qtilde-app}
\td{Q} \;=\; \tau\bigl(\td{S}\, K + dS\, U_K\bigr).
\end{equation}
The same calculation with $K$ in place of $Q$ gives
\begin{equation}
\label{eq:Ktilde-app}
\td{K} \;=\; \tau\bigl(\td{S}^\Tr Q + dS^\Tr U_Q\bigr).
\end{equation}

\subsection*{Final assembled identities}

Collecting Equations~\eqref{eq:F-app}, \eqref{eq:alpha-app},
\eqref{eq:C-app}, \eqref{eq:Ptilde-app}, \eqref{eq:E-app},
\eqref{eq:Stilde-app}, \eqref{eq:Vtilde-app}, \eqref{eq:dOtilde-app},
\eqref{eq:Qtilde-app}, and \eqref{eq:Ktilde-app},
\begin{align*}
F &= \tau(U_Q\, K^\Tr + Q\, U_K^\Tr), &
\alpha_i &= \textstyle\sum_j P_{ij}\, F_{ij}, &
C &= dO\, U_V^\Tr, \\
\td{P_{ij}}^{\circ} &= C_{ij} + F_{ij}(dP_{ij} - D_i), &
\td{P_{ij}} &= \td{P_{ij}}^{\circ} - \alpha_i\, dP_{ij}, &
E_i &= \textstyle\sum_j P_{ij}\, \td{P_{ij}}, \\
\td{S} &= P \odot (\td{P} - E\, \mathbf{1}^\Tr), &
\td{Q} &= \tau(\td{S}\, K + dS\, U_K), &
\td{K} &= \tau(\td{S}^\Tr Q + dS^\Tr U_Q), \\
\td{V} &= (P \odot h)^\Tr dO, &
\td{dO} &= P\, U_V + (P \odot h)\, V, &
h &= F - \alpha\, \mathbf{1}^\Tr.
\end{align*}
These are the closed-form BoB identities used in Section~\ref{sec:affine}.

\section{Proofs of Proposition~\ref{prop:affine} and Theorem~\ref{thm:correct}}
\label{app:affine-proof}
\label{app:correctness-proof}
 
\begin{proof}[Proof of Proposition~\ref{prop:affine}]
For~\eqref{eq:fusion}, substitute
$\td{P_{ij}} = \td{P_{ij}}^{\circ} - \alpha_i\, dP_{ij}$ into
$E_i = \sum_{j \in \mathcal{V}_i} P_{ij}\, \td{P_{ij}}$:
\[
  E_i = \sum_{j \in \mathcal{V}_i} P_{ij}\, \td{P_{ij}}^{\circ}
        - \alpha_i \sum_{j \in \mathcal{V}_i} P_{ij}\, dP_{ij}
      = E^{\circ}_i - \alpha_i D_i.
\]
For~\eqref{eq:fusion}, expand $h_{ij} = F_{ij} - \alpha_i$ in
$\td{dO}_i = \sum_{j \in \mathcal{V}_i} P_{ij}\, u_{V,j}
            + \sum_{j \in \mathcal{V}_i} P_{ij}\, h_{ij}\, v_j$:
\[
  \td{dO}_i
  = \underbrace{\sum_{j \in \mathcal{V}_i} P_{ij}\, u_{V,j}
                + \sum_{j \in \mathcal{V}_i} P_{ij}\, F_{ij}\, v_j}_{\td{dO}^{\circ}_i}
  - \alpha_i \underbrace{\sum_{j \in \mathcal{V}_i} P_{ij}\, v_j}_{o_i}.
\]
For~\eqref{eq:fusion}, substitute
$\td{P_{ij}} = \td{P_{ij}}^{\circ} - \alpha_i dP_{ij}$ and
$dS_{ij} = P_{ij}(dP_{ij} - D_i)$ into
$\td{Q}_i = \tau \sum_{j \in \mathcal{V}_i}
  (P_{ij}(\td{P_{ij}} - E_i) k_j + dS_{ij} u_{K,j})$:
\[
  \td{Q}_i = \tau\Bigl(
    \underbrace{\sum_{j \in \mathcal{V}_i} P_{ij}\, \td{P_{ij}}^{\circ} k_j
                + \sum_{j \in \mathcal{V}_i} P_{ij}(dP_{ij} - D_i) u_{K,j}}_{\Omega_i}
    - \alpha_i \underbrace{\sum_{j \in \mathcal{V}_i} P_{ij}\, dP_{ij}\, k_j}_{R_i}
    - E_i \underbrace{\sum_{j \in \mathcal{V}_i} P_{ij}\, k_j}_{B_i}\Bigr). \qedhere
\]
\end{proof}
 
\begin{proof}[Proof of Theorem~\ref{thm:correct}]
By Proposition~\ref{prop:affine}, the row-side outputs decompose as
$E_i = E^{\circ}_i - \alpha_i D_i$,
$\td{dO}_i = \td{dO}^{\circ}_i - \alpha_i o_i$, and
$\td{Q}_i = \tau(\Omega_i - \alpha_i R_i - E_i B_i)$, where each of $\alpha$, $E^{\circ}$, $\td{dO}^{\circ}$,
$\Omega$, $B$, and $R$ can be accumulated while streaming the
columns without first knowing $\alpha_i$; here
``$\alpha$-free'' means independent of the completed value of
$\alpha$ (Proposition~\ref{prop:affine}). The inner loop of Pass~1 is the
blockwise realization of these column reductions, visiting each
$(i, j)$ tile exactly once per row. $D_i$ is loaded from the saved
row state $(L, D)$. Pass~2 loads $(\alpha_i, E_i)$ from HBM and
every tilewise intermediate is then the direct blockwise
realization of its defining equation from
Section~\ref{sec:bob-formulation}. The column accumulations
$\td{V}_j = \sum_i (P_{ij} \odot h_{ij})^\Tr dO_i$ and
$\td{K}_j = \tau \sum_i (\td{S_{ij}}^\Tr Q_i + dS_{ij}^\Tr U_{Q,i})$
are the exact tile-sum forms
of~\eqref{eq:QKtilde} and~\eqref{eq:VOtilde}.
 
For the FLOPs and memory bounds: Pass~1 forms the four
$B_r^{(1)} \times B_c^{(1)}$ tiles
$dP_{ij}, F_{ij}, C_{ij}, \td{P_{ij}}^{\circ}$ via matrix products of $N \times d$ operands and accumulates six column reductions of total output dimension $O(B_r^{(1)} d)$. Summed over all
$\lceil N / B_r^{(1)} \rceil \cdot \lceil N / B_c^{(1)} \rceil$ tile
visits, the matrix-multiply work is
$\Theta(N^2 d)$ and the elementwise softmax work is $O(N^2)$.
Pass~2 has the same per-tile work in column-major orientation, also
totalling $\Theta(N^2 d)$. Auxiliary HBM beyond inputs, outputs, the
saved row state $(L, D)$, and the forward output $O$ consists only of
the two $N$-vectors $(\alpha, E)$, hence $O(N)$. The on-chip footprint per tile visit is bounded by the tile-budget
analysis of Section~\ref{sec:algo}. These identities establish
algebraic correctness; floating-point reduction order can differ
between implementations.
\end{proof}

\section{Large-cache lower bound}
\label{app:lb-proof}

This section proves Theorem~\ref{thm:bob-lb}. We use the standard
two-level I/O model from Section~\ref{sec:prelim}: SRAM has capacity
\(M\) scalar elements, HBM is unbounded, and I/O counts scalar transfers
between them. The lower bound is stated for the standard
FlashAttention-style score-recomputation model for exact tiled
attention. In this model, the full score matrix
\(S=\tau QK^\Tr\) and probability matrix \(P=\softmax(S)\) are not
given as HBM inputs. Score blocks are recomputed from the bilinear forms
\[
S_{ij}
=
\tau q_i^\Tr k_j
=
\tau\sum_{\ell=1}^{d}Q_{i\ell}K_{j\ell},
\]
probability blocks are obtained by exact row-wise softmax normalization,
and the algorithm may receive the usual \(O(N)\) saved row statistics,
such as the row log-normalizer \(L_i\). This is the execution model used
by exact FlashAttention-style kernels and by the attention I/O lower
bounds invoked below.

The proof has one inherited component and one new component. The
inherited component is the large-cache I/O lower bound for exact forward
softmax attention under the standard score-recomputation model. In particular,
the large-cache lower bound of \citep[Lem.~3.4]{saha2024ioattention} shows that, for
\(M=\Omega(d^2)\), exact attention evaluated through the standard
matrix-product score computation has I/O complexity
\(\Omega(N^2d^2/M)\). This is where the \(N^2d^2/M\) term enters the
argument. The new component in this paper is the embedding: we show that
a special BoB instance outputs exactly
\(\softmax(\tau QK^\Tr)W\). Therefore, an asymptotically faster exact
BoB algorithm in the same score-recomputation model would imply an
asymptotically faster exact forward-attention algorithm, contradicting
the inherited lower bound.

Intuitively, the factor \(N^2d^2/M\) arises as follows. The attention
score computation contains \(\Omega(N^2d)\) scalar multiplication
vertices of the form \(Q_{i\ell}K_{j\ell}\). In the large-cache regime,
\citep[Lem.~3.3]{saha2024ioattention} bound the number of such level-1
score vertices in each \(M\)-partition part by \(O(M^2/d)\). Hence any
\(M\)-partition needs \(\Omega(N^2d^2/M^2)\) parts, and the red-blue
pebble partition lemma converts this into
\(\Omega(N^2d^2/M)\) HBM transfers.

\begin{lemma}[Forward-attention lower bound with row statistics]
\label{lem:fwd-lb-saved-stats}
Under the standard FlashAttention-style score-recomputation model,
exact forward softmax attention
\[
Y=\softmax(\tau QK^\Tr)W
\]
with \(Q,K,W\in\R^{N\times d}\), \(M=\Omega(d^2)\), \(N\ge d\),
\(M\lesssim Nd\), and \(\Theta(N^2)\) valid query-key interactions
requires \(\Omega(N^2d^2/M)\) HBM transfers.
\end{lemma}

\begin{proof}
We invoke the large-cache score-computation I/O lower bound of
\citep[Lem.~3.4]{saha2024ioattention}. The cited lemma concerns
unmasked attention. The causal case contains an unmasked rectangular
subproblem: every key in the first half of the sequence is visible to
every query in the second half. This gives \(\Theta(N^2)\) score
interactions even after masking, so the score-vertex partition argument
applies to that subproblem.

The scope of the reduction is important. We require score tiles to be
computed through the standard bilinear products of \(Q\) and \(K\);
full score or probability matrices are neither given as inputs nor
stored in HBM. The usual \(O(N)\) row normalizers may be supplied, but the
execution model still requires score tiles to be recomputed through
the stated products. In particular, this is not a lower bound for
arbitrary algorithms given an unrestricted, precomputed attention
matrix or other nonlinear checkpoints. The inherited partition bound therefore gives
\(\Omega(N^2d^2/M)\) HBM transfers for \(M=\Omega(d^2)\),
which dominates the \(O(N)\) row-state reads when \(M\lesssim Nd\).
This establishes the claimed lower bound within the stated
score-recomputation model.
\end{proof}

\begin{lemma}[Forward-attention embedding in BoB]
\label{lem:bob-forward-embed}
For any \(Q,K,W\in\R^{N\times d}\), there is a valid BoB instance whose
output \(\td{dO}\) is exactly
\[
\td{dO}
=
\softmax(\tau QK^\Tr)W.
\]
\end{lemma}

\begin{proof}
Choose
\[
V=0,
\qquad
dO=0,
\qquad
U_Q=0,
\qquad
U_K=0,
\qquad
U_V=W.
\]
The required row state is valid: \(O=PV=0\), \(D_i=dO_i^\Tr O_i=0\),
and the row normalizers are determined by \(Q\) and \(K\). On this
instance,
\[
F
=
\tau(U_QK^\Tr+QU_K^\Tr)
=
0,
\qquad
\alpha_i
=
\sum_j P_{ij}F_{ij}
=
0,
\qquad
h_{ij}
=
F_{ij}-\alpha_i
=
0.
\]
Using the BoB identity for the upstream-output derivative,
\[
\td{dO}
=
PU_V+(P\odot h)V,
\]
we obtain
\[
\td{dO}
=
PW.
\]
Thus this BoB instance computes exact forward softmax attention with
value matrix \(W\).
\end{proof}

\begin{proof}[Proof of Theorem~\ref{thm:bob-lb}]
The argument uses the same score-recomputation model as
Lemma~\ref{lem:fwd-lb-saved-stats}: the algorithm is not given \(S\) or
\(P\) as HBM inputs, does not store them as \(N\times N\) HBM tensors, and
evaluates the score subgraph through the standard bilinear products
\(S_{ij}=\tau q_i^\Tr k_j\). Suppose, for contradiction, that there were
an exact BoB algorithm with
\[
o\!\left(\frac{N^2d^2}{M}\right)
\]
HBM transfers under this model. Given an
arbitrary forward-attention instance \((Q,K,W)\), construct the BoB
instance from Lemma~\ref{lem:bob-forward-embed} by setting
\(V=0\), \(dO=0\), \(U_Q=U_K=0\), and \(U_V=W\). Running the assumed
BoB algorithm and reading its \(\td{dO}\) output would compute
\[
\softmax(\tau QK^\Tr)W.
\]
This construction only relabels the forward-attention value matrix
\(W\) as the BoB upstream matrix \(U_V\), while \(V,dO,U_Q,U_K\) are
fixed zero inputs and the required row state is determined by \(Q\) and
\(K\). Thus it gives an exact forward-attention algorithm in the same
score-recomputation model with asymptotically smaller HBM traffic than
Lemma~\ref{lem:fwd-lb-saved-stats} allows. This contradiction proves
that any exact BoB algorithm in the stated score-recomputation model
requires
\[\Omega\!\left(\frac{N^2d^2}{M}\right)\]
HBM transfers.
\end{proof}

\section{FlashBoB: Fully Annotated Pseudocode}
\label{app:full-algo}
 
Algorithm~\ref{alg:flashbob-full} reproduces the schedule of
Algorithm~\ref{alg:flashbob} with explicit tile shapes and inline
commentary on the role of each quantity. Shapes are given for the divisible dense case; boundary tiles are
padded and masked as in FlashAttention. For causal attention, masked
tile pairs are skipped and invalid elements within boundary tiles
have zero probability. All intermediates
\(
P_{ij},\ dP_{ij},\ F_{ij},\ C_{ij},\
\td{P}^{\circ}_{ij},\ \td{P}_{ij},\ \td{S}_{ij}
\)
are formed only as \(B_r\times B_c\) tiles in SRAM/registers and are
discarded after the tile update; no full \(N\times N\) intermediate is
materialized in HBM.
 
\begin{algorithm}[t]
\caption{\textsc{FlashBoB}: fully annotated schedule.}
\label{alg:flashbob-full}
\small
\begin{algorithmic}[1]
\Require $Q, K, V, dO, U_Q, U_K, U_V, O \in \R^{N \times d}$; row state $L, D \in \R^N$ in HBM, SRAM size $M$, $\tau = 1/\sqrt{d}$
\Ensure $\td{Q}, \td{K}, \td{V}, \td{dO} \in \R^{N \times d}$ and $\alpha, E \in \R^N$
\State $b \gets \lfloor M / (c_{\textsc{BoB}} d) \rfloor$;\;
       $B_r^{(1)}, B_c^{(2)} \gets b$;\;
       $B_c^{(1)}, B_r^{(2)} \gets \min(b, d)$
\Statex \colorbox{green!10}{\parbox{0.97\linewidth}{\textsc{Pass 1 (row-major).} Finalize $(\alpha, E, \td{Q}, \td{dO})$ in a single sweep via affine fusion.}}
\For{$i = 1$ \textbf{to} $\lceil N / B_r^{(1)} \rceil$}
    \State Load $Q_i, dO_i, U_{Q,i}, O_i \in \R^{B_r^{(1)} \times d}$ and saved row scalars $L_i, D_i \in \R^{B_r^{(1)}}$
           \Comment{active row block}
    \State $\alpha_i, E^{\circ}_i \gets 0 \in \R^{B_r^{(1)}}$;\;
           $\td{dO}^{\circ}_i, \Omega_i, B_i, R_i \gets 0 \in \R^{B_r^{(1)} \times d}$
           \Comment{six row accumulators; five are $\alpha$-free}
    \For{$j = 1$ \textbf{to} $\lceil N / B_c^{(1)} \rceil$}
        \State Load $K_j, V_j, U_{K,j}, U_{V,j} \in \R^{B_c^{(1)} \times d}$
               \Comment{streamed column block}
        \State Recompute $S_{ij}, P_{ij} \in \R^{B_r^{(1)} \times B_c^{(1)}}$ from $(Q_i, K_j, L_i)$
               \Comment{tilewise softmax}
        \State $dP_{ij} \gets dO_i\, V_j^\Tr$
               \Comment{upstream to $P$ via $dV$ path}
        \State $F_{ij} \gets \tau(U_{Q,i} K_j^\Tr + Q_i U_{K,j}^\Tr)$
               \Comment{BoB gradient at $dS$, eq.~\eqref{eq:F-app}}
        \State $C_{ij} \gets dO_i\, U_{V,j}^\Tr$
               \Comment{contribution from $dV$ path}
        \State $\td{P_{ij}}^{\circ} \gets C_{ij} + F_{ij} \odot (dP_{ij} - D_i)$
               \Comment{$\alpha$-free base of $\td{P}$, eq.~\eqref{eq:Ptilde-app}}
        \State $\alpha_i \mathrel{+}= \rowsum(P_{ij} \odot F_{ij})$
               \Comment{row scalar for affine update}
        \State $E^{\circ}_i \mathrel{+}= \rowsum(P_{ij} \odot \td{P_{ij}}^{\circ})$
               \Comment{$\alpha$-free component of $E_i$}
        \State $\td{dO}^{\circ}_i \mathrel{+}=
                P_{ij} U_{V,j} + (P_{ij} \odot F_{ij}) V_j$
               \Comment{$\alpha$-free component of $\td{dO}_i$}
        \State $\Omega_i \mathrel{+}=
                (P_{ij} \odot \td{P_{ij}}^{\circ}) K_j
                + (P_{ij} \odot (dP_{ij} - D_i)) U_{K,j}$
               \Comment{$\alpha$-free component of $\td{Q}_i$}
        \State $B_i \mathrel{+}= P_{ij} K_j$;\;
               $R_i \mathrel{+}= (P_{ij} \odot dP_{ij}) K_j$
               \Comment{coefficients for update}
    \EndFor
    \State $E_i \gets E^{\circ}_i - \alpha_i \odot D_i$
           \Comment{eq.~\eqref{eq:fusion}}
    \State $\td{dO}_i \gets \td{dO}^{\circ}_i - \alpha_i \odot O_i$
           \Comment{eq.~\eqref{eq:fusion}}
    \State $\td{Q}_i \gets \tau(\Omega_i - \alpha_i \odot R_i - E_i \odot B_i)$
           \Comment{eq.~\eqref{eq:fusion}}
    \State Write $\alpha_i, E_i, \td{Q}_i, \td{dO}_i$ to HBM
           \Comment{$(\alpha, E)$ consumed in Pass 2}
\EndFor
\Statex \colorbox{orange!12}{\parbox{0.97\linewidth}{\textsc{Pass 2 (column-major).} Consume saved $(\alpha, E)$; accumulate and write $(\td{K}, \td{V})$.}}
\For{$j = 1$ \textbf{to} $\lceil N / B_c^{(2)} \rceil$}
    \State Load $K_j, V_j, U_{K,j}, U_{V,j} \in \R^{B_c^{(2)} \times d}$
           \Comment{active column block}
    \State Initialize $\td{K}_j, \td{V}_j \gets 0 \in \R^{B_c^{(2)} \times d}$
    \For{$i = 1$ \textbf{to} $\lceil N / B_r^{(2)} \rceil$}
        \State Load $Q_i, dO_i, U_{Q,i} \in \R^{B_r^{(2)} \times d}$ and $L_i, D_i, \alpha_i, E_i \in \R^{B_r^{(2)}}$
               \Comment{streamed row block + saved scalars}
        \State Recompute $S_{ij}, P_{ij} \in \R^{B_r^{(2)} \times B_c^{(2)}}$ from $(Q_i, K_j, L_i)$
        \State $dP_{ij} \gets dO_i\, V_j^\Tr$;\;
               $F_{ij} \gets \tau(U_{Q,i} K_j^\Tr + Q_i U_{K,j}^\Tr)$;\;
               $C_{ij} \gets dO_i\, U_{V,j}^\Tr$
        \State $h_{ij} \gets F_{ij} - \alpha_i$;\;
               $dS_{ij} \gets P_{ij} \odot (dP_{ij} - D_i)$
               \Comment{tile-wise quantities}
        \State $\td{P_{ij}} \gets C_{ij} + F_{ij} \odot (dP_{ij} - D_i)
                 - \alpha_i \odot dP_{ij}$
               \Comment{full $\td{P}$, eq.~\eqref{eq:Ptilde-app}}
        \State $\td{S_{ij}} \gets P_{ij} \odot (\td{P_{ij}} - E_i)$
               \Comment{second softmax-backward}
        \State $\td{V}_j \mathrel{+}= (P_{ij} \odot h_{ij})^\Tr dO_i$
               \Comment{column update for $\td{V}_j$}
        \State $\td{K}_j \mathrel{+}= \tau(\td{S_{ij}}^\Tr Q_i
                + dS_{ij}^\Tr U_{Q,i})$
               \Comment{column update for $\td{K}_j$}
    \EndFor
    \State Write $\td{K}_j, \td{V}_j$ to HBM
\EndFor
\end{algorithmic}
\end{algorithm}

\section{Dependency graphs and schedule overlays}
\label{app:graphs}

This appendix collects the dependency graphs underlying the derivation. Three simple
graph ideas appear repeatedly. A \emph{fan-out} means one quantity is computed once and then sent to several later expressions. A \emph{reduction detour} means an $N\times N$ quantity is first compressed to one scalar per row and that row scalar is
then used later. A \emph{two-factor product} means a node is formed by multiplying information coming from two streamed operands. Figures~\ref{fig:fwd-graph} and~\ref{fig:bwd-graph} show these patterns in the forward and first-backward graphs,
and Figure~\ref{fig:bob-graph} shows how they reappear in BoB. Figures~\ref{fig:bob-3pass} and~\ref{fig:bob-2pass} then overlay the conservative three-pass and affine-fused
two-pass schedules on the BoB graph.

In all graphs, blue nodes are inputs, green nodes are intermediates, red nodes are outputs or incoming gradients from the next derivative level, and yellow circles are
one-scalar-per-row quantities. Dashed gray nodes are recomputed tilewise quantities that are never materialized in HBM.

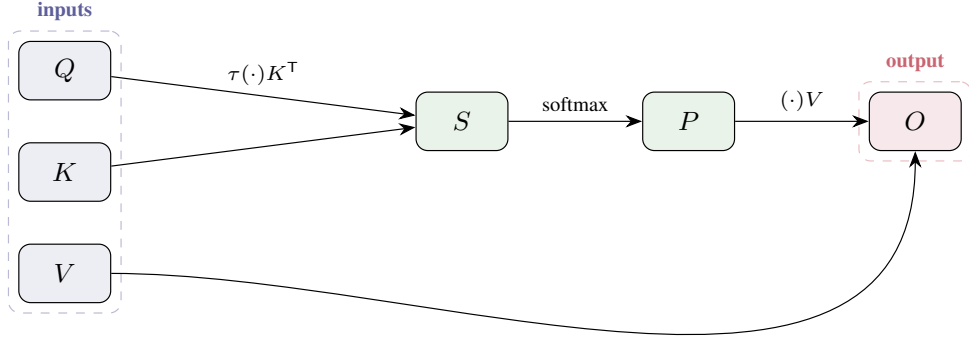
\begin{figure}[t]
\centering
\resizebox{0.92\linewidth}{!}{%
\begin{tikzpicture}[
    node distance=1.1cm and 1.4cm,
    every node/.style={font=\small},
    input/.style={draw, rounded corners, fill=graphinput!8, minimum width=1.1cm, minimum height=0.7cm},
    inter/.style={draw, rounded corners, fill=graphinter!10, minimum width=1.1cm, minimum height=0.7cm},
    output/.style={draw, rounded corners, fill=graphoutput!10, minimum width=1.1cm, minimum height=0.7cm},
    >={Stealth[length=2mm]}
]
\node[input] (Q) {$Q$};
\node[input, below=0.5cm of Q] (K) {$K$};
\node[input, below=0.5cm of K] (V) {$V$};
\node[inter, right=2cm of Q] (S) at ($(Q)!0.5!(K) + (2.2,0)$) {$S$};
\node[inter, right=1.6cm of S] (P) {$P$};
\node[output, right=1.6cm of P] (O) {$O$};
\draw[->] (Q) -- node[above,font=\scriptsize]{$\tau(\cdot)K^\Tr$} (S);
\draw[->] (K) -- (S);
\draw[->] (S) -- node[above,font=\scriptsize]{softmax} (P);
\draw[->] (P) -- node[above,font=\scriptsize]{$(\cdot)V$} (O);
\draw[->] (V) to[out=0,in=-90] (O);
\begin{pgfonlayer}{background}
\node[draw=graphinput!30, dashed, rounded corners, fit=(Q) (V), label={[graphinput!70]above:{\scriptsize\textbf{inputs}}}] {};
\node[draw=graphoutput!30, dashed, rounded corners, fit=(O), label={[graphoutput!70]above:{\scriptsize\textbf{output}}}] {};
\end{pgfonlayer}
\end{tikzpicture}
}
\caption{Forward dependency graph. Three inputs flow through two $N \times N$
intermediates ($S$ and $P$) to a single $N \times d$ output.}
\label{fig:fwd-graph}
\end{figure}

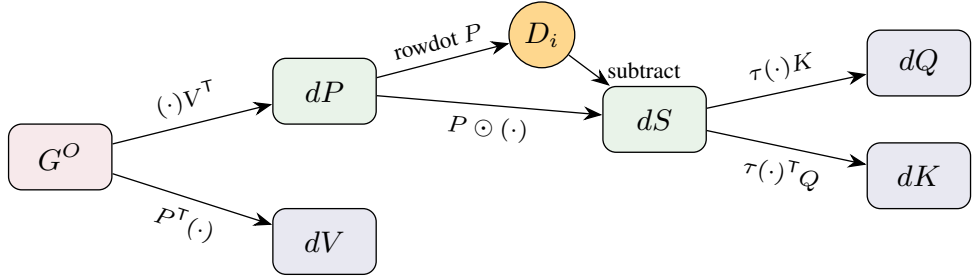
\begin{figure}[t]
\centering
\resizebox{0.92\linewidth}{!}{%
\begin{tikzpicture}[
    node distance=1.1cm and 1.5cm,
    every node/.style={font=\small},
    upstream/.style={draw, rounded corners, fill=graphoutput!10, minimum width=1.1cm, minimum height=0.7cm},
    inter/.style={draw, rounded corners, fill=graphinter!10, minimum width=1.1cm, minimum height=0.7cm},
    output/.style={draw, rounded corners, fill=graphinput!10, minimum width=1.1cm, minimum height=0.7cm},
    scalar/.style={draw, circle, fill=graphscalar!60, inner sep=1pt, minimum size=0.7cm},
    >={Stealth[length=2mm]}
]
\node[upstream] (G) {$dO$};
\node[inter, right=1.7cm of G, yshift=0.7cm] (dP) {$dP$};
\node[output, right=1.7cm of G, yshift=-0.9cm] (dV) {$dV$};
\node[scalar, right=1.4cm of dP, yshift=0.6cm] (D) {$D_i$};
\node[inter, right=2.4cm of dP, yshift=-0.3cm] (dS) {$dS$};
\node[output, right=1.7cm of dS, yshift=0.6cm] (dQ) {$dQ$};
\node[output, right=1.7cm of dS, yshift=-0.6cm] (dK) {$dK$};
\draw[->] (G) -- node[above,font=\scriptsize,sloped]{$(\cdot)V^\Tr$} (dP);
\draw[->] (G) -- node[below,font=\scriptsize,sloped]{$P^\Tr(\cdot)$} (dV);
\draw[->] (dP) -- node[above,font=\scriptsize,sloped]{rowdot $P$} (D);
\draw[->] (dP) -- node[below,font=\scriptsize,sloped]{$P \odot (\cdot)$} (dS);
\draw[->] (D) -- node[right,font=\scriptsize]{\;subtract} (dS);
\draw[->] (dS) -- node[above,font=\scriptsize,sloped]{$\tau(\cdot)K$} (dQ);
\draw[->] (dS) -- node[below,font=\scriptsize,sloped]{$\tau(\cdot)^\Tr Q$} (dK);
\end{tikzpicture}
}
\caption{First-backward dependency graph. One upstream ($dO$) produces three outputs
via two $N \times N$ intermediates, with a single row-scalar detour through $D_i$.}
\label{fig:bwd-graph}
\end{figure}

\begin{figure}[t]
\centering
\resizebox{0.96\linewidth}{!}{%
\begin{tikzpicture}[
    every node/.style={font=\scriptsize},
    upstream/.style={draw, rounded corners, fill=graphoutput!15, minimum width=0.9cm, minimum height=0.6cm},
    inter/.style={draw, rounded corners, fill=graphinter!15, minimum width=0.9cm, minimum height=0.6cm},
    recomp/.style={draw, dashed, rounded corners, fill=gray!10, text=gray!40!black, minimum width=0.9cm, minimum height=0.6cm},
    output/.style={draw, rounded corners, fill=graphinput!15, minimum width=0.9cm, minimum height=0.6cm},
    scalar/.style={draw, circle, fill=graphscalar!60, inner sep=1pt, minimum size=0.6cm},
    >={Stealth[length=1.6mm]}
]
\node[upstream] (bdQ) at (0,0) {$U_Q$};
\node[upstream] (bdK) at (2,0) {$U_K$};
\node[upstream] (bdV) at (10,0) {$U_V$};
\node[inter] (F) at (1,-1.6) {$F$};
\node[recomp] (dP) at (6,-1.6) {$dP$};
\node[inter] (C) at (10,-1.6) {$C$};
\node[scalar] (alpha) at (-0.5,-3.2) {$\alpha_i$};
\node[inter] (h) at (2.3,-3.2) {$h$};
\node[inter] (Pbase) at (7.8,-3.2) {$\td{P}^{\circ}$};
\node[inter] (Pbar) at (5.5,-4.8) {$\td{P}$};
\node[recomp] (dS) at (-0.5,-5.8) {$dS$};
\node[scalar] (E) at (3.5,-6.8) {$E_i$};
\node[inter] (Sbar) at (7,-6.8) {$\td{S}$};
\node[output] (bQ) at (0,-8.6) {$\td{Q}$};
\node[output] (bK) at (3,-8.6) {$\td{K}$};
\node[output] (bV) at (6.5,-8.6) {$\td{V}$};
\node[output] (bG) at (10,-8.6) {$\td{dO}$};
\draw[->] (bdQ) -- node[left,font=\tiny]{$\tau(\cdot)K^\Tr$} (F);
\draw[->] (bdK) -- node[right,font=\tiny]{$\tau Q(\cdot)^\Tr$} (F);
\draw[->] (bdV) -- node[right,font=\tiny]{$d{O}(\cdot)^\Tr$} (C);
\draw[->] (F) to[out=200,in=70] node[left,font=\tiny,pos=0.45]{$\rowsum(P \odot \cdot)$} (alpha);
\draw[->] (F) -- (h);
\draw[->] (alpha) -- (h);
\draw[->] (C) -- (Pbase);
\draw[->] (F.east) to[out=-10,in=160] (Pbase.west);
\draw[->] (Pbase) -- (Pbar);
\draw[->] (alpha) to[out=-60,in=180] (Pbar);
\draw[->] (Pbar) -- (E);
\draw[->] (Pbar) -- (Sbar);
\draw[->] (E) -- (Sbar);
\draw[->] (Sbar) to[out=225,in=50] node[above,sloped,font=\tiny,pos=0.15]{$\tau(\cdot)K$} (bQ);
\draw[->] (Sbar) to[out=250,in=60] node[above,sloped,font=\tiny,pos=0.2]{$\tau(\cdot)^\Tr Q$} (bK);
\draw[->] (h.south) to[out=-75,in=135] node[above,sloped,font=\tiny,pos=0.4]{$(d{O})^\Tr(P \odot \cdot)$} (bV);
\draw[->] (h.east) to[out=-10,in=150] node[above,sloped,font=\tiny,pos=0.85]{$(P \odot \cdot)V$} (bG);
\draw[->] (bdV.east) to[out=-15,in=90] node[right,font=\tiny]{$P(\cdot)$} (bG);
\draw[->, gray!60, dashed] (dP) -- (Pbase);
\draw[->, gray!60, dashed] (dP) to[out=-70,in=60] (Pbar);
\draw[->, gray!60, dashed] (dS) -- node[right,font=\tiny,pos=0.55]{$(\cdot)\, U_K$} (bQ);
\draw[->, gray!60, dashed] (dS) to[out=-30,in=140] node[above,sloped,font=\tiny,pos=0.35]{$(\cdot)^\Tr U_Q$} (bK);
\end{tikzpicture}
}
\caption{BoB dependency graph. Three upstream matrices produce four outputs through five $N \times N$ intermediates and two row scalars. The softmax backward identity appears twice (at $\alpha_i \to h$ and at $E_i \to \td{S}$).}
\label{fig:bob-graph}
\end{figure}

\begin{figure}[t]
\centering
\resizebox{0.96\linewidth}{!}{%
\begin{tikzpicture}[
    every node/.style={font=\scriptsize},
    upstream/.style={draw, rounded corners, fill=graphoutput!15, minimum width=0.9cm, minimum height=0.6cm},
    inter/.style={draw, rounded corners, fill=graphinter!15, minimum width=0.9cm, minimum height=0.6cm},
    recomp/.style={draw, dashed, rounded corners, fill=gray!10, text=gray!40!black, minimum width=0.9cm, minimum height=0.6cm},
    output/.style={draw, rounded corners, fill=graphinput!15, minimum width=0.9cm, minimum height=0.6cm},
    scalar/.style={draw, circle, fill=graphscalar!60, inner sep=1pt, minimum size=0.6cm},
    >={Stealth[length=1.6mm]}
]
\begin{scope}[on background layer]
\fill[graphinput!8, rounded corners=4pt] (-1.7,0.7) rectangle (11.1,-2.2);
\fill[graphinput!12, rounded corners=4pt] (-1.7,-2.35) rectangle (11.1,-5.7);
\fill[graphoutput!6, rounded corners=4pt] (-1.7,-5.9) rectangle (11.1,-9.3);
\end{scope}
\node[graphinput!60!black, anchor=west, font=\small\bfseries, align=left] at (11.2,-0.7) {Pass 1\\[-1pt]\scriptsize row-major\\[-1pt]\scriptsize compute $\alpha$};
\node[graphinput!70!black, anchor=west, font=\small\bfseries, align=left] at (11.2,-4.0) {Pass 2\\[-1pt]\scriptsize row-major\\[-1pt]\scriptsize compute $E,\td{Q},\td{dO}$};
\node[graphoutput!70!black, anchor=west, font=\small\bfseries, align=left] at (11.2,-7.5) {Pass 3\\[-1pt]\scriptsize column-major\\[-1pt]\scriptsize compute $\td{K},\td{V}$};
\node[upstream] (bdQ) at (0,0) {$U_Q$};
\node[upstream] (bdK) at (2,0) {$U_K$};
\node[upstream] (bdV) at (10,0) {$U_V$};
\node[inter] (F) at (1,-1.2) {$F$};
\node[scalar] (alpha) at (-0.4,-1.95) {$\alpha_i$};
\node[inter] (C) at (10,-3.0) {$C$};
\node[recomp] (dP2) at (6,-3.0) {$dP$};
\node[inter] (h2) at (2.3,-3.0) {$h$};
\node[inter] (Pbase2) at (7.8,-3.0) {$\td{P}^{\circ}$};
\node[inter] (Pbar2) at (5.5,-4.25) {$\td{P}$};
\node[scalar] (E2) at (3.0,-5.1) {$E_i$};
\node[recomp] (dS2row) at (0.8,-4.55) {$dS$};
\node[output] (bQ2) at (-0.4,-5.1) {$\td{Q}$};
\node[output] (bG2) at (9.8,-5.1) {$\td{dO}$};
\node[recomp] (dS3) at (0.8,-7.0) {$dS$};
\node[inter] (Sbar3) at (5.5,-7.0) {$\td{S}$};
\node[output] (bK3) at (2.5,-8.5) {$\td{K}$};
\node[output] (bV3) at (7.5,-8.5) {$\td{V}$};
\draw[->] (bdQ) -- (F);
\draw[->] (bdK) -- (F);
\draw[->] (F) -- (alpha);
\draw[->] (bdV) -- (C);
\draw[->] (alpha) -- (h2);
\draw[->] (F.east) to[out=-20,in=150] (h2.west);
\draw[->] (C) -- (Pbase2);
\draw[->] (F.east) to[out=-5,in=160] (Pbase2.west);
\draw[->] (Pbase2) -- (Pbar2);
\draw[->] (alpha) to[out=-65,in=180] (Pbar2);
\draw[->] (Pbar2) to[out=210,in=35] (E2);
\draw[->] (alpha) to[out=-110,in=120] (bQ2);
\draw[->] (h2) to[out=-75,in=90] (bG2);
\draw[->,gray!60,dashed] (dP2) -- (Pbase2);
\draw[->,gray!60,dashed] (dP2) to[out=-65,in=35] (Pbar2);
\draw[->,gray!60,dashed] (dS2row) to[bend right=15] (bQ2);
\draw[->] (Pbar2) -- (Sbar3);
\draw[->] (E2) to[out=-70,in=130] (Sbar3);
\draw[->] (Sbar3) to[out=220,in=40] (bK3);
\draw[->] (h2.east) to[out=-45,in=120] (bV3);
\draw[->] (bdV) to[out=-90,in=60] (bV3);
\draw[->,gray!60,dashed] (dS3) to[bend right=15] (bK3);
\end{tikzpicture}
}
\caption{Three-pass schedule overlaid on the BoB graph. Pass~1 exists only to accumulate $\alpha_i$. Proposition~\ref{prop:affine} collapses Passes~1 and~2 into one.}
\label{fig:bob-3pass}
\end{figure}
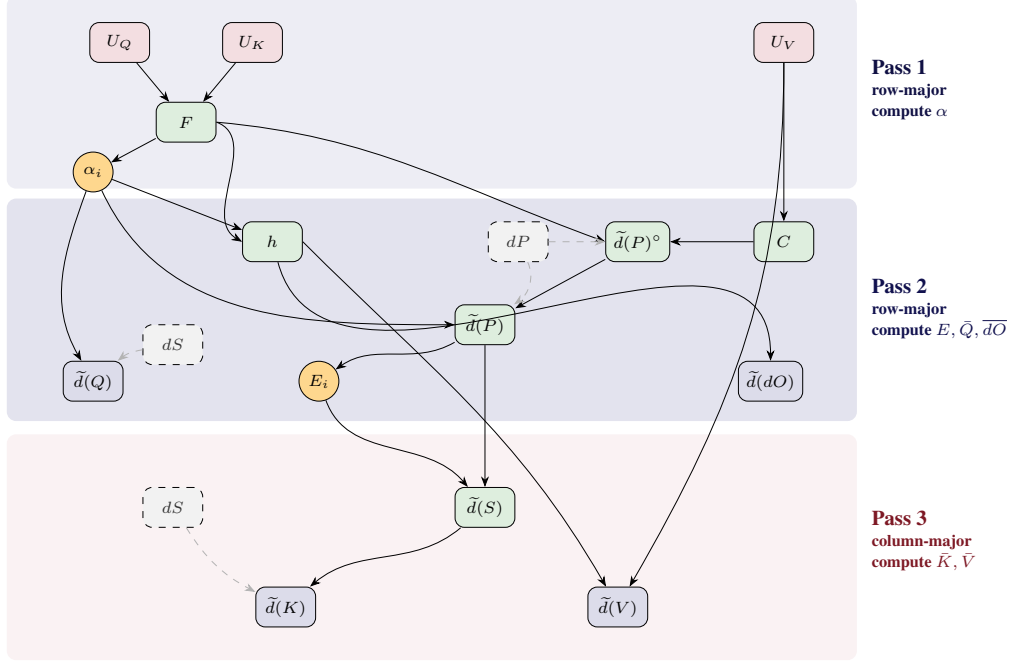

\begin{figure}[t]
\centering
\resizebox{0.96\linewidth}{!}{%
\begin{tikzpicture}[
    every node/.style={font=\scriptsize},
    upstream/.style={draw, rounded corners, fill=graphoutput!15, minimum width=0.9cm, minimum height=0.6cm},
    inter/.style={draw, rounded corners, fill=graphinter!15, minimum width=0.9cm, minimum height=0.6cm},
    recomp/.style={draw, dashed, rounded corners, fill=gray!10, text=gray!40!black, minimum width=0.9cm, minimum height=0.6cm},
    output/.style={draw, rounded corners, fill=graphinput!15, minimum width=0.9cm, minimum height=0.6cm},
    scalar/.style={draw, circle, fill=graphscalar!60, inner sep=1pt, minimum size=0.6cm},
    >={Stealth[length=1.6mm]}
]
\begin{scope}[on background layer]
\fill[graphinput!8, rounded corners=4pt] (-1.7,0.7) rectangle (11.1,-5.55);
\end{scope}
\node[graphinput!60!black, anchor=west, font=\small\bfseries, align=left] at (11.2,-2.4) {Pass 1\\[-1pt]\scriptsize row-major};
\begin{scope}[on background layer]
\fill[graphoutput!6, rounded corners=4pt] (-1.7,-5.75) rectangle (11.1,-9.3);
\end{scope}
\node[graphoutput!70!black, anchor=west, font=\small\bfseries, align=left] at (11.2,-7.5) {Pass 2\\[-1pt]\scriptsize column-major};
\node[upstream] (bdQ) at (0,0) {$U_Q$};
\node[upstream] (bdK) at (2,0) {$U_K$};
\node[upstream] (bdV) at (10,0) {$U_V$};
\node[inter] (F) at (1,-1.4) {$F$};
\node[inter] (C) at (10,-1.4) {$C$};
\node[recomp] (dP) at (6,-1.4) {$dP$};
\node[scalar] (alpha) at (-0.5,-2.8) {$\alpha_i$};
\node[inter] (h) at (2.3,-2.8) {$h$};
\node[inter] (Pbase) at (7.8,-2.8) {$\td{P}^{\circ}$};
\node[inter] (Pbar) at (5.5,-4.2) {$\td{P}$};
\node[scalar] (E) at (3,-5.2) {$E_i$};
\node[output] (bQ) at (-0.4,-5.2) {$\td{Q}$};
\node[output] (bG) at (9.8,-5.2) {$\td{dO}$};
\node[inter] (Sbar) at (5.5,-7.0) {$\td{S}$};
\node[recomp] (dS1) at (0.8,-4.6) {$dS$};
\node[recomp] (dS2) at (0.8,-7.0) {$dS$};
\node[output] (bK) at (2.5,-8.5) {$\td{K}$};
\node[output] (bV) at (7.5,-8.5) {$\td{V}$};
\draw[->] (bdQ) -- (F);
\draw[->] (bdK) -- (F);
\draw[->] (bdV) -- (C);
\draw[->] (F) to[out=210,in=60] (alpha);
\draw[->] (F) -- (h);
\draw[->] (alpha) -- (h);
\draw[->] (C) -- (Pbase);
\draw[->] (F.east) to[out=-10,in=160] (Pbase.west);
\draw[->] (Pbase) -- (Pbar);
\draw[->] (alpha) to[out=-60,in=180] (Pbar);
\draw[->] (Pbar) to[out=210,in=40] (E);
\draw[->] (Pbar) -- (Sbar);
\draw[->] (E) to[out=-70,in=130] (Sbar);
\draw[->, gray!60, dashed] (dP) -- (Pbase);
\draw[->, gray!60, dashed] (dP) to[out=-70,in=40] (Pbar);
\draw[->] (alpha) to[out=-90,in=130] (bQ);
\draw[->] (h.south) to[out=-70,in=90] (bG);
\draw[->, gray!60, dashed] (dS1) to[bend right=15] (bQ);
\draw[->, gray!60, dashed] (dS2) to[bend right=15] (bK);
\draw[->] (Sbar) to[out=220,in=40] (bK);
\draw[->] (h.east) to[out=-45,in=120] (bV);
\draw[->] (bdV) to[out=-90,in=60] (bV);
\end{tikzpicture}
}
\caption{Affine-fused two-pass schedule (\textsc{FlashBoB}) overlaid on the BoB graph. One
horizontal cut separates row-side nodes (above) from column-side nodes (below).
Pass~1 finalizes $\td{Q}$ and $\td{dO}$ by accumulating the six row
accumulators from Proposition~\ref{prop:affine} and applying scalar
corrections at the end of the sweep; Pass~2 consumes the saved
$(\alpha,E)$ and finalizes $\td{K}$ and $\td{V}$ as natural column
reductions.
}
\label{fig:bob-2pass}
\end{figure}
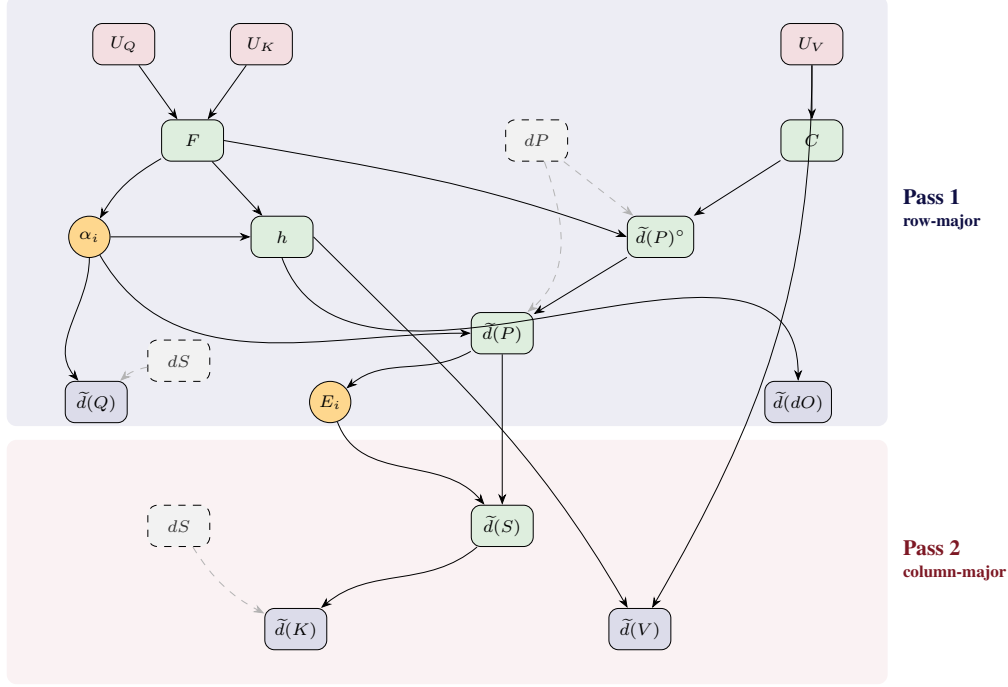

\subsection{Three-pass versus two-pass: HBM savings and empirical comparison}
\label{app:two-pass-proof}

Throughout this appendix, the \emph{FlashAttention execution model} means
an exact tiled attention execution with bounded \(O(M)\) on-chip working
sets, no global atomic adds on output tensors, no \(N\times N\) HBM
intermediates, and probability tiles recomputed from saved row state
\((L,D)\) rather than read from HBM.
 
The four BoB outputs do not finish in the same traversal direction
(\S\ref{sec:affine}), forcing one row-major and one column-major
sweep. A single pass cannot satisfy the FlashAttention execution
model: the column-side outputs $(\td K, \td V)$ reduce over rows for
fixed column, so accumulating them during a row-major traversal
would require materializing an $N \times N$ intermediate in HBM,
using global atomics for column-indexed outputs, or holding
$\Theta(Nd)$ partial column accumulators on chip. Each path gives up
the discipline that made FlashAttention fast.
 
The three-pass schedule (Figure~\ref{fig:bob-3pass}) is itself the
result of one level of affine fusion: its second pass uses the
delayed $E$-correction to finalize $\td Q$ in a single row-major
sweep. The two-pass \textsc{FlashBoB} schedule
(Figure~\ref{fig:bob-2pass}) extends this with a second level of
affine fusion that eliminates the dedicated $\alpha$-only sweep
(Pass~1 of the three-pass schedule). The HBM saving is bounded by
the streamed reads of the eliminated sweep, which loads exactly two
$B \times d$ tensors per tile pair ($K_j$ and $U_{K,j}$, the operands of $F_{ij}$). This is $2dN^2/B$ out of the three-pass total of $(10d + 4)N^2/B$ streamed traffic, about $2d/(10d+4) \approx 19\%$ of the streamed budget at $d{=}64$. With compute unchanged across schedules, the wall-time advantage scales the streamed-traffic ratio by the memory-bound fraction of the kernel.
 
\begin{table}[t]
\centering\small
\caption{Direct comparison between the three-pass baseline schedule
and the affine-fused two-pass kernel on NVIDIA A100 (GPT-2 Small
shape, batch 4, BF16 inputs with FP32 reductions).}
\label{tab:2pass-vs-3pass}
\begin{tabular}{@{}r rr r@{}}
\toprule
$N$ & \textsc{3-pass bob} (ms) & \textsc{FlashBoB} (ms) & Ratio \\
\midrule
  1\,024 &        2.70 &        2.57 & $1.05\times$ \\
  4\,096 &       27.30 &       25.80 & $1.06\times$ \\
 16\,384 &      389.14 &      365.73 & $1.06\times$ \\
 65\,536 &   6\,227.59 &   5\,843.14 & $1.07\times$ \\
262\,144 & 100\,696.16 &  94\,853.72 & $1.06\times$ \\
\bottomrule
\end{tabular}
\end{table}
 
The measured 5--7\% advantage at long context is consistent with the
predicted streamed-traffic ratio. We view this as additional evidence for the schedule choice; the formal optimality argument is the large-cache I/O bound of Theorem~\ref{thm:bob-lb}.

\section{Per-pass I/O ledger and proof of Theorem~\ref{thm:io}}
\label{app:io-ledger}

\begin{center}
\small
\begin{tabular}{@{}llll@{}}
\toprule
Pass & Stationary loads & Streamed loads & Writes \\
\midrule
1 (row) & $Q_i, U_{Q,i}, do_i, O_i, L_i, D_i$ & $K_j, U_{K,j}, V_j, U_{V,j}$ & $\td{Q}_i, \td{dO}_i, \alpha_i, E_i$ \\
2 (col) & $K_j, V_j, U_{K,j}, U_{V,j}$ & $Q_i, U_{Q,i}, do_i, L_i, D_i, \alpha_i, E_i$ & $\td{K}_j, \td{V}_j$ \\
\bottomrule
\end{tabular}
\end{center}

\medskip
\noindent\textit{Pass 1 element counts.} Stationary reads:
$\frac{N}{B_r^{(1)}}(4B_r^{(1)}d+2B_r^{(1)})=4Nd+2N$. Streamed reads:
$\frac{N}{B_r^{(1)}}\frac{N}{B_c^{(1)}} \cdot 4B_c^{(1)} d = \frac{4dN^2}{B_r^{(1)}}$.
Writes: $\frac{N}{B_r^{(1)}}(2B_r^{(1)}d+2B_r^{(1)}) = 2Nd+2N$.

\textit{Pass 2 element counts.} Stationary reads:
$\frac{N}{B_c^{(2)}} \cdot 4B_c^{(2)} d = 4Nd$. Streamed reads:
$\frac{N}{B_c^{(2)}}\frac{N}{B_r^{(2)}}(3B_r^{(2)} d + 4B_r^{(2)}) = \frac{(3d+4)N^2}{B_c^{(2)}}$.
Writes: $\frac{N}{B_c^{(2)}} \cdot 2B_c^{(2)} d = 2Nd$.

Summing all reads and writes,
\begin{align*}
  Q_{\textsc{FlashBoB}} &= (4Nd + 2N) + \frac{4dN^2}{B_r^{(1)}} + (2Nd + 2N) + 4Nd + \frac{(3d+4)N^2}{B_c^{(2)}} + 2Nd \\
  &= 12Nd + 4N + \frac{4dN^2}{B_r^{(1)}} + \frac{(3d+4)N^2}{B_c^{(2)}},
\end{align*}
proving Theorem~\ref{thm:io}. Setting $B_r^{(1)} = B_c^{(2)} = b$ gives
$Q_{\textsc{FlashBoB}} = 12Nd + 4N + (7d+4) N^2 / b$. \qed

The constant $c_{\textsc{BoB}}$ in
$b=\lfloor M/(c_{\textsc{BoB}}d)\rfloor$ is fixed by the on-chip
footprint of one tile visit. The footprint contains a constant number
of row-vector blocks, column-vector blocks, interaction-shaped
\(B_r\times B_c\) tiles, and row/column scalar buffers. We can write the
requirement as
\[
M
\;\ge\;
c_1 B_r d
+
c_2 B_c d
+
c_3 B_r B_c
+
c_4(B_r+B_c),
\]
for implementation-dependent constants \(c_1,c_2,c_3,c_4>0\). The
\(c_3B_rB_c\) term accounts for the constant number of
interaction-shaped intermediates formed during a tile visit, such as
probability-like, score-adjoint-like, and temporary product tiles. This
constant-factor accounting does not change the asymptotic choice
\(B_r,B_c=\Theta(M/d)\) in the large-cache regime used by
Theorem~\ref{thm:io}; it only makes explicit that the SRAM budget hides
more than one interaction tile.
 
Under causal masking, each pass visits only tile pairs that
intersect the lower-triangular attention pattern. When row and column
tiles have the same size and \(T_r=T_c=T\), there are
\(T(T+1)/2\) visited pairs instead of \(T^2\); the resulting
fraction is \(\rho=(1+1/T)/2\). For the unequal row and column
tile sizes used by Algorithm~\ref{alg:flashbob}, the boundary-tile
count depends on both tile counts, but the leading quadratic streamed
traffic approaches one half of the dense count at fixed tile sizes as
\(N\) grows. The stationary and output-write terms of
Theorem~\ref{thm:io} are unchanged.

\section{Extensions}
\label{app:extensions}

The derivation of Section~\ref{sec:flashbob} extends cleanly to three settings that
appear in practice: variable-length packed sequences, dropout (derived analytically below; the current code release does not enable dropout or arbitrary user-provided attention bias), and mixed precision.

Variable-length packing and sparse masks require only that each row has a
well-defined set of allowed columns and that masked entries have zero probability.
For packed variable-length attention with total length $N$, define $\mathcal{V}$ as
the set of block pairs with at least one unmasked entry and restrict the inner streamed
loop over column blocks to those $j$ with $(i,j) \in \mathcal{V}$. Under a causal mask,
$|\mathcal{V}|$ is block-triangular, so roughly half of all tile visits are skipped.
For each pass $p$, let $\mathcal{T}^{(p)}$ denote the visited tile pairs under its row/column tiling. The dense tile-pair count $T_r^{(p)}T_c^{(p)}$ is replaced by $|\mathcal{T}^{(p)}|$; traffic also depends on the streamed tile widths and boundary tiles.

\paragraph{Analytical dropout extension.}
For dropout after the softmax, let $R_{ij}\in\{0,1\}$ be a fixed keep
mask and $g_{ij}=R_{ij}/(1-p_{\mathrm{drop}})$. The forward probability
used for values is $\widehat P=g\odot P$, so
$O=\widehat P V$ and $dV=\widehat P^\Tr dO$. The upstream
of the \emph{undropped} softmax is
$A=g\odot(dO V^\Tr)$; hence the first-backward row statistic is
$D_i=\sum_j P_{ij}A_{ij}=dO_i^\Tr o_i$. The BoB row summaries
and intermediate probabilities then satisfy
\begin{align*}
\alpha_i &= \sum_j P_{ij}F_{ij}, &
C &= g\odot(dO U_V^\Tr),\\
\td{P_{ij}} &= C_{ij}+F_{ij}(A_{ij}-D_i)-\alpha_i A_{ij}, &
E_i &= \sum_j P_{ij}\td{P_{ij}},\\
\td{S_{ij}} &=P_{ij}(\td{P_{ij}}-E_i), &
h_{ij} &=F_{ij}-\alpha_i.
\end{align*}
The column- and row-owned outputs involving the value pathway become
\begin{align*}
\td{V} &= (g\odot P\odot h)^\Tr dO,\\
\td{dO} &= (g\odot P)U_V+(g\odot P\odot h)V.
\end{align*}
Thus the outer softmax Jacobian still uses the undropped $P$, while
paths through the dropout operation are multiplied by $g$. These are
analytical identities. The released implementation does not support
dropout; an implementation would need to regenerate the same mask in
both passes, for example using a counter-based random-number generator.

The element counts in Theorem~\ref{thm:io} are precision-agnostic, so mixed-precision
execution requires no algorithmic change. A standard implementation stores $Q, K, V, dO, U_Q, U_K, U_V$ in BF16 or FP16 and accumulates all reductions ($L, D, \alpha, E, E^\circ$) in FP32, matching the
FlashAttention backward~\citep{dao2023flashattention2}.

\section{Additional numerical agreement data}
\label{app:errors}

We evaluate the maximum absolute and relative discrepancy of each BoB output against the PyTorch \textsc{math} autograd reference across sequence lengths $N \in \{128, 256, 512, 1024, 2048, 4096, 8192\}$. The \textsc{math} reference differentiates PyTorch's SDPA \texttt{math} backend twice via \texttt{torch.autograd.grad} with $N \times N$ intermediates materialized in HBM and appears in Table~\ref{tab:errors} with zero discrepancy as the baseline. We also verified that evaluating the closed-form identities of \S\ref{sec:bob-formulation} directly in PyTorch matches autograd double backward to floating-point precision. We compare four non-materializing implementations against this reference: the two-pass \textsc{FlashBoB} schedule (Algorithm~\ref{alg:flashbob}), the intermediate three-pass schedule (\textsc{3-pass bob}; Appendix~\ref{app:two-pass-proof}), and the two GradMem-derived HVP kernels (\textsc{hvp-m} and \textsc{hvp-s}). All runs use BF16 inputs with FP32 accumulators on a single NVIDIA A6000 in the GPT-2 Small attention configuration (batch 4, 12 heads, $d=64$, causal mask). We also verified numerical agreement for $d=128$ and native GQA shapes ($H_Q>H_{KV}$); Table~\ref{tab:errors} reports the $d=64$ sweep.

Across the entries in Table~\ref{tab:errors}, the largest absolute
discrepancy is $2.38\times10^{-5}$ and the largest reported relative
discrepancy is $4.76\times10^{-6}$. The differences are consistent with
different floating-point reduction orders. These values describe
numerical agreement with the autograd oracle; they do not imply
bitwise agreement.

\begin{table}[t]
\centering\scriptsize
\caption{Maximum absolute and relative discrepancy of each BoB output
against the PyTorch \textsc{math} autograd double-backward reference, for $N \in \{128, \ldots, 8192\}$. Math
reference shown for completeness as the all-zero baseline. In addition
to the main methods from Table~\ref{tab:attn-kernel}, we include a
3-pass BoB implementation as an auxiliary correctness check. Hardware:
NVIDIA A6000 48GB; precision: BF16 inputs with FP32 reductions.}
\label{tab:errors}
\setlength{\tabcolsep}{4pt}
\begin{tabular}{@{}rl rrrr rrrr@{}}
\toprule
& & \multicolumn{4}{c}{Max absolute discrepancy} & \multicolumn{4}{c}{Max relative discrepancy} \\
\cmidrule(lr){3-6}\cmidrule(lr){7-10}
$N$ & Backend & $\td{Q}$ & $\td{K}$ & $\td{V}$ & $\td{dO}$ & $\td{Q}$ & $\td{K}$ & $\td{V}$ & $\td{dO}$ \\
\midrule
\multirow{5}{*}{128}
  & \textsc{math}         & 0       & 0       & 0       & 0       & 0       & 0       & 0       & 0       \\
  & \textsc{hvp-m}        & 2.38e-6 & 5.25e-6 & 1.43e-6 & 9.54e-7 & 5.74e-7 & 8.08e-7 & 3.31e-7 & 1.97e-7 \\
  & \textsc{hvp-s}        & 2.38e-6 & 5.25e-6 & 1.43e-6 & 9.54e-7 & 5.74e-7 & 8.08e-7 & 3.31e-7 & 1.97e-7 \\
  & \textsc{FlashBoB}    & 3.34e-6 & 3.58e-6 & 1.91e-6 & 2.38e-6 & 8.03e-7 & 5.51e-7 & 4.42e-7 & 4.92e-7 \\
  & \textsc{3-pass bob}   & 2.62e-6 & 5.01e-6 & 1.91e-6 & 1.67e-6 & 6.31e-7 & 7.72e-7 & 4.42e-7 & 3.44e-7 \\
\midrule
\multirow{5}{*}{256}
  & \textsc{math}         & 0       & 0       & 0       & 0       & 0       & 0       & 0       & 0       \\
  & \textsc{hvp-m}        & 1.31e-6 & 3.93e-6 & 1.43e-6 & 1.19e-6 & 3.74e-7 & 1.02e-6 & 3.84e-7 & 2.96e-7 \\
  & \textsc{hvp-s}        & 1.31e-6 & 3.93e-6 & 1.43e-6 & 1.19e-6 & 3.74e-7 & 1.02e-6 & 3.84e-7 & 2.96e-7 \\
  & \textsc{FlashBoB}    & 2.15e-6 & 4.53e-6 & 1.43e-6 & 1.91e-6 & 6.12e-7 & 1.18e-6 & 3.84e-7 & 4.74e-7 \\
  & \textsc{3-pass bob}   & 2.15e-6 & 3.58e-6 & 1.43e-6 & 1.19e-6 & 6.12e-7 & 9.31e-7 & 3.84e-7 & 2.96e-7 \\
\midrule
\multirow{5}{*}{512}
  & \textsc{math}         & 0       & 0       & 0       & 0       & 0       & 0       & 0       & 0       \\
  & \textsc{hvp-m}        & 1.31e-6 & 6.20e-6 & 1.43e-6 & 7.15e-7 & 3.51e-7 & 1.21e-6 & 3.80e-7 & 1.78e-7 \\
  & \textsc{hvp-s}        & 1.31e-6 & 6.20e-6 & 1.43e-6 & 7.15e-7 & 3.51e-7 & 1.21e-6 & 3.80e-7 & 1.78e-7 \\
  & \textsc{FlashBoB}   & 2.86e-6 & 5.96e-6 & 1.31e-6 & 1.91e-6 & 7.66e-7 & 1.17e-6 & 3.49e-7 & 4.74e-7 \\
  & \textsc{3-pass bob}   & 1.91e-6 & 5.96e-6 & 1.55e-6 & 1.37e-6 & 5.11e-7 & 1.17e-6 & 4.12e-7 & 3.40e-7 \\
\midrule
\multirow{5}{*}{1\,024}
  & \textsc{math}         & 0       & 0       & 0       & 0       & 0       & 0       & 0       & 0       \\
  & \textsc{hvp-m}        & 1.79e-6 & 1.12e-5 & 1.91e-6 & 7.45e-7 & 3.57e-7 & 2.44e-6 & 5.25e-7 & 2.18e-7 \\
  & \textsc{hvp-s}        & 1.79e-6 & 1.12e-5 & 1.91e-6 & 7.45e-7 & 3.57e-7 & 2.44e-6 & 5.25e-7 & 2.18e-7 \\
  & \textsc{FlashBoB}    & 2.38e-6 & 6.20e-6 & 1.67e-6 & 1.85e-6 & 4.75e-7 & 1.35e-6 & 4.60e-7 & 5.40e-7 \\
  & \textsc{3-pass bob}   & 1.91e-6 & 5.72e-6 & 1.67e-6 & 1.07e-6 & 3.80e-7 & 1.25e-6 & 4.60e-7 & 3.14e-7 \\
\midrule
\multirow{5}{*}{2\,048}
  & \textsc{math}         & 0       & 0       & 0       & 0       & 0       & 0       & 0       & 0       \\
  & \textsc{hvp-m}        & 2.15e-6 & 1.34e-5 & 2.15e-6 & 1.19e-6 & 5.29e-7 & 2.87e-6 & 5.74e-7 & 2.94e-7 \\
  & \textsc{hvp-s}        & 2.15e-6 & 1.34e-5 & 2.15e-6 & 1.19e-6 & 5.29e-7 & 2.87e-6 & 5.74e-7 & 2.94e-7 \\
  & \textsc{FlashBoB}    & 3.34e-6 & 1.12e-5 & 1.91e-6 & 2.15e-6 & 8.23e-7 & 2.41e-6 & 5.10e-7 & 5.29e-7 \\
  & \textsc{3-pass bob}   & 3.70e-6 & 1.03e-5 & 1.91e-6 & 1.43e-6 & 9.11e-7 & 2.20e-6 & 5.10e-7 & 3.52e-7 \\
\midrule
\multirow{5}{*}{4\,096}
  & \textsc{math}         & 0       & 0       & 0       & 0       & 0       & 0       & 0       & 0       \\
  & \textsc{hvp-m}        & 2.06e-6 & 1.72e-5 & 1.67e-6 & 7.15e-7 & 3.54e-7 & 3.29e-6 & 3.06e-7 & 1.47e-7 \\
  & \textsc{hvp-s}        & 2.06e-6 & 1.72e-5 & 1.67e-6 & 7.15e-7 & 3.54e-7 & 3.29e-6 & 3.06e-7 & 1.47e-7 \\
  & \textsc{FlashBoB}    & 3.34e-6 & 1.26e-5 & 2.15e-6 & 1.97e-6 & 5.75e-7 & 2.42e-6 & 3.94e-7 & 4.04e-7 \\
  & \textsc{3-pass bob}   & 2.15e-6 & 1.00e-5 & 2.15e-6 & 1.31e-6 & 3.69e-7 & 1.92e-6 & 3.94e-7 & 2.69e-7 \\
\midrule
\multirow{5}{*}{8\,192}
  & \textsc{math}         & 0       & 0       & 0       & 0       & 0       & 0       & 0       & 0       \\
  & \textsc{hvp-m}        & 2.33e-6 & 2.38e-5 & 1.67e-6 & 1.07e-6 & 6.33e-7 & 4.76e-6 & 4.12e-7 & 2.79e-7 \\
  & \textsc{hvp-s}        & 2.33e-6 & 2.38e-5 & 1.67e-6 & 1.07e-6 & 6.33e-7 & 4.76e-6 & 4.12e-7 & 2.79e-7 \\
  & \textsc{FlashBoB}    & 4.29e-6 & 2.11e-5 & 2.15e-6 & 2.41e-6 & 1.17e-6 & 4.22e-6 & 5.30e-7 & 6.28e-7 \\
  & \textsc{3-pass bob}   & 2.74e-6 & 2.00e-5 & 1.91e-6 & 1.43e-6 & 7.47e-7 & 4.00e-6 & 4.71e-7 & 3.72e-7 \\
\bottomrule
\end{tabular}
\end{table}

\section{Extended isolated-kernel scaling}
\label{app:extended-kernel-scaling}

Table~\ref{tab:attn-kernel-extended} reports the full isolated attention-layer BoB scaling sweep at $d=64$ on NVIDIA A100 80GB, corresponding to Table~\ref{tab:attn-kernel} (left panel). Once the materializing \textsc{math} backend ($N=8{,}192$) and the GradMem-derived HVP baselines ($N=16{,}384$) exhaust memory, \textsc{FlashBoB} continues to execute exact attention BoB at $N=262{,}144$ on a single A100 80GB GPU. This is an isolated attention-layer capacity measurement; it does not establish that a complete model can train at that sequence length.

\begin{table}[t]
\centering\small
\caption{Extended attention-layer BoB scaling on NVIDIA A100 80GB.
Shape: GPT-2 Small attention, batch 4, 12 heads, head dimension 64,
causal mask, BF16 inputs with FP32 reductions. OOM denotes out of memory;
best wall-clock time and lowest peak memory are bold.}
\label{tab:attn-kernel-extended}
\begin{tabular}{@{}r rrrr rrrr@{}}
\toprule
& \multicolumn{4}{c}{Wall-clock time (ms)} & \multicolumn{4}{c}{Peak memory (MiB)} \\
\cmidrule(lr){2-5}\cmidrule(lr){6-9}
$N$ & \textsc{math} & \textsc{hvp-m} & \textsc{hvp-s} & \textsc{FlashBoB}
& \textsc{math} & \textsc{hvp-m} & \textsc{hvp-s} & \textsc{FlashBoB} \\
\midrule
     256 &    1.82 &    1.58 &    1.66 & \textbf{0.94}
         &       212 &       114 &       116 &         \textbf{95} \\
     512 &    4.95 &    2.31 &    2.33 & \textbf{1.29}
         &       668 &       315 &       318 &        \textbf{171} \\
  1\,024 &   17.66 &    7.39 &    6.41 & \textbf{2.57}
         &    2\,375 &    1\,045 &    1\,051 &        \textbf{323} \\
  2\,048 &   68.80 &   28.34 &   22.73 & \textbf{7.53}
         &    8\,957 &    3\,805 &    3\,817 &        \textbf{630} \\
  4\,096 &  265.67 &  114.89 &   89.14 & \textbf{25.80}
         &   34\,794 &   14\,521 &   14\,546 &     \textbf{1\,243} \\
  8\,192 &     OOM &  468.93 &  356.84 & \textbf{95.11}
         &        -- &   55\,971 &   56\,020 &     \textbf{1\,702} \\
 16\,384 &     OOM &     OOM &     OOM & \textbf{365.73}
         &        -- &        -- &        -- &     \textbf{3\,372} \\
 32\,768 &     OOM &     OOM &     OOM & \textbf{1\,449.06}
         &        -- &        -- &        -- &     \textbf{6\,744} \\
 65\,536 &     OOM &     OOM &     OOM & \textbf{5\,843.14}
         &        -- &        -- &        -- &    \textbf{13\,488} \\
131\,072 &     OOM &     OOM &     OOM & \textbf{23\,528.87}
         &        -- &        -- &        -- &    \textbf{26\,976} \\
262\,144 &     OOM &     OOM &     OOM & \textbf{94\,853.72}
         &        -- &        -- &        -- &    \textbf{53\,952} \\
\bottomrule
\end{tabular}
\end{table}

\section{Full model benchmark sweeps: GPT-2, Pythia, and native GQA/MoE}
\label{app:gpt2-full}

This appendix reports per-$N$ sweeps of complete model-level
second-order steps. They measure the model step, including
non-attention operations, rather than the attention-BoB kernel
in isolation. The sweeps support Table~\ref{tab:gpt2-bench}
in \S\ref{sec:benchmarks}: GPT-2 Small, Medium, and Large on
NVIDIA B200
(Tables~\ref{tab:gpt2-small-full}--\ref{tab:gpt2-large-full}),
Pythia-160M and Pythia-410M on NVIDIA B300
(Table~\ref{tab:pythia}), and five released native GQA and
MoE models on NVIDIA B300 (Table~\ref{tab:gqa-models}).

\begin{table}[t]
\centering\scriptsize
\caption{GPT-2 Small full sweep, batch 4, BF16 inputs with FP32 reductions. Hardware: NVIDIA B200 (192 GiB HBM); OOM denotes exceeding the 192 GiB budget.}
\label{tab:gpt2-small-full}
\begin{tabular}{@{}r rrrr rrrr@{}}
\toprule
& \multicolumn{4}{c}{Wall time (ms)} & \multicolumn{4}{c}{Peak memory (MiB)} \\
\cmidrule(lr){2-5}\cmidrule(lr){6-9}
$N$ & \textsc{math} & \textsc{hvp-m} & \textsc{hvp-s} & \textsc{FlashBoB} & \textsc{math} & \textsc{hvp-m} & \textsc{hvp-s} & \textsc{FlashBoB} \\
\midrule
 512    &   46.5 &   48.0 &   49.5 &   43.0 &   4\,932 &   3\,956 &   3\,981 &   3\,983 \\
1\,024  &   58.7 &   58.6 &   57.5 &   48.6 &   9\,605 &   5\,913 &   5\,956 &   5\,957 \\
2\,048  &  134.9 &  146.3 &  114.6 &   84.4 &  24\,091 &   9\,830 &   9\,915 &   9\,924 \\
4\,096  &  397.7 &  417.5 &  304.0 &  190.9 &  79\,058 &  17\,649 &  17\,804 &  17\,804 \\
8\,192  &    OOM & 1537.5 & 1017.3 &  509.1 &       -- &  46\,965 &  47\,267 &  33\,604 \\
16\,384 &    OOM & 4974.7 & 3421.6 & 1576.4 &       -- & 147\,483 & 148\,082 &  65\,194 \\
32\,768 &    OOM &    OOM &    OOM & 5368.3 &       -- &       -- &       -- & 128\,392 \\
\bottomrule
\end{tabular}
\end{table}

\begin{table}[H]
\centering\scriptsize
\caption{GPT-2 Medium full sweep, batch 4, BF16 inputs with FP32 reductions. Hardware: NVIDIA B200 (192 GiB HBM); OOM denotes exceeding the 192 GiB budget.}
\label{tab:gpt2-medium-full}
\begin{tabular}{@{}r rrrr rrrr@{}}
\toprule
& \multicolumn{4}{c}{Wall time (ms)} & \multicolumn{4}{c}{Peak memory (MiB)} \\
\cmidrule(lr){2-5}\cmidrule(lr){6-9}
$N$ & \textsc{math} & \textsc{hvp-m} & \textsc{hvp-s} & \textsc{FlashBoB} & \textsc{math} & \textsc{hvp-m} & \textsc{hvp-s} & \textsc{FlashBoB} \\
\midrule
 512    &   88.4 &  102.8 &  104.0 &   81.9 &  10\,443 &   7\,851 &   7\,900 &   7\,900 \\
1\,024  &  125.6 &  125.2 &  120.9 &   98.8 &  20\,562 &  10\,769 &  10\,869 &  10\,869 \\
2\,048  &  312.6 &  343.0 &  255.9 &  177.2 &  55\,203 &  16\,612 &  16\,810 &  16\,810 \\
4\,096  &    OOM & 1028.5 &  726.1 &  420.9 &       -- &  30\,480 &  30\,876 &  28\,677 \\
8\,192  &    OOM & 3944.0 & 2557.1 & 1196.6 &       -- &  74\,523 &  75\,315 &  52\,429 \\
16\,384 &    OOM &    OOM &    OOM & 3878.9 &       -- &       -- &       -- &  99\,912 \\
\bottomrule
\end{tabular}
\end{table}

\begin{table}[H]
\centering\scriptsize
\caption{GPT-2 Large full sweep, batch 4, BF16 inputs with FP32 reductions. Hardware: NVIDIA B200 (192 GiB HBM); OOM denotes exceeding the 192 GiB budget.}
\label{tab:gpt2-large-full}
\begin{tabular}{@{}r rrrr rrrr@{}}
\toprule
& \multicolumn{4}{c}{Wall time (ms)} & \multicolumn{4}{c}{Peak memory (MiB)} \\
\cmidrule(lr){2-5}\cmidrule(lr){6-9}
$N$ & \textsc{math} & \textsc{hvp-m} & \textsc{hvp-s} & \textsc{FlashBoB} & \textsc{math} & \textsc{hvp-m} & \textsc{hvp-s} & \textsc{FlashBoB} \\
\midrule
 512    &  131.3 &  135.4 &  139.3 &  122.3 &  19\,720 &  14\,856 &  14\,956 &  14\,965 \\
1\,024  &  215.4 &  216.6 &  192.7 &  154.4 &  37\,496 &  19\,120 &  19\,335 &  19\,342 \\
2\,048  &  562.2 &  614.6 &  450.4 &  304.3 & 100\,554 &  27\,666 &  28\,040 &  28\,045 \\
4\,096  &    OOM & 1879.6 & 1313.4 &  740.3 &       -- &  49\,205 &  49\,968 &  45\,457 \\
8\,192  &    OOM & 7322.2 & 4721.8 & 2162.8 &       -- & 110\,927 & 112\,414 &  80\,303 \\
16\,384 &    OOM &    OOM &    OOM & 7106.5 &       -- &       -- &       -- & 150\,019 \\
\bottomrule
\end{tabular}
\end{table}

\begin{table}[H]
\centering\scriptsize
\caption{End-to-end second-order steps on Pythia (GPT-NeoX architecture) at $d=128$. Batch 4, BF16 with FP32 reductions, NVIDIA B300. Speedup is relative to the fastest baseline that fits at the same sequence length. $N$ restarts for each model block.}
\label{tab:pythia}
\begin{tabular}{@{}ll rrrr rr@{}}
\toprule
Model & $N$ & \textsc{math} (ms) & \textsc{hvp-m} (ms) & \textsc{hvp-s} (ms) & \textbf{\textsc{FlashBoB} (ms)} & \textsc{FlashBoB} memory (MiB) & Speedup \\
\midrule
\multirow{7}{*}{Pythia-160M}
  & 512      & 87.9  & 91.0  & 101.8 & \textbf{79.1}   & 5\,062   & $1.11\times$ \\
  & 1\,024   & 103.2 & 106.0 & 106.2 & \textbf{99.3}   & 8\,942   & $1.04\times$ \\
  & 2\,048   & 152.4 & 161.4 & 142.7 & \textbf{115.7}  & 16\,702  & $1.23\times$ \\
  & 4\,096   & 420.8 & 437.0 & 361.7 & \textbf{239.8}  & 32\,222  & $1.51\times$ \\
  & 8\,192   & OOM   & 1\,563.1 & 1\,195.0 & \textbf{659.0} & 63\,263 & $1.81\times$ \\
  & 16\,384  & OOM   & 5\,069.5 & 4\,088.9 & \textbf{2\,064.2} & 126\,273 & $1.98\times$ \\
  & 32\,768  & OOM   & OOM   & OOM   & \textbf{7\,114.6} & 250\,435 & -- \\
\midrule[1.0pt]
\multirow{6}{*}{Pythia-410M}
  & 512      & 93.7  & 91.5  & 92.2  & \textbf{81.7}   & 8\,763   & $1.12\times$ \\
  & 1\,024   & 161.1 & 158.5 & 148.4 & \textbf{126.3}  & 14\,502  & $1.17\times$ \\
  & 2\,048   & 368.0 & 393.8 & 328.7 & \textbf{242.9}  & 25\,992  & $1.35\times$ \\
  & 4\,096   & 1\,074.1 & 1\,125.0 & 918.4 & \textbf{583.7} & 48\,974 & $1.57\times$ \\
  & 8\,192   & OOM   & 4\,121.5 & 3\,121.7 & \textbf{1\,652.8} & 94\,937 & $1.89\times$ \\
  & 16\,384  & OOM   & OOM   & OOM   & \textbf{5\,280.2} & 186\,863 & -- \\
\bottomrule
\end{tabular}
\end{table}

\begin{table}[H]
\centering\scriptsize
\caption{Complete model-level second-order calculations on native GQA and MoE models. Batch size 1, BF16, one NVIDIA B300. Latencies are synchronized CUDA-event medians. Peak memory is PyTorch \texttt{max\_memory\_allocated}, measured identically for both methods.}
\label{tab:gqa-models}
\setlength{\tabcolsep}{4pt}
\begin{tabular}{@{}lrrrrrrrr@{}}
\toprule
Model ($H_Q/H_{KV}$) & $d$ & $N$ &
\textbf{\textsc{FlashBoB} (ms)} & \textsc{math} (ms) & Speedup &
\textbf{\textsc{FlashBoB} (GiB)} & \textsc{math} (GiB) & Reduction \\
\midrule
Llama 3.2 1B ($32/8$)
& 64 & 1\,024 & \textbf{68.681} & 96.664 & $1.41\times$
& \textbf{16.489} & 22.987 & 28.3\% \\
& 64 & 2\,048 & \textbf{116.920} & 220.275 & $1.88\times$
& \textbf{23.755} & 51.946 & 54.3\% \\
& 64 & 4\,096 & \textbf{247.073} & 653.269 & $2.64\times$
& \textbf{39.761} & 152.469 & 73.9\% \\
\midrule[0.8pt]
SmolLM2-135M ($9/3$)
& 64 & 1\,024 & \textbf{110.231} & 113.387 & $1.03\times$
& \textbf{4.211} & 7.753 & 45.7\% \\
& 64 & 2\,048 & \textbf{105.555} & 146.113 & $1.38\times$
& \textbf{7.601} & 21.576 & 64.8\% \\
& 64 & 4\,096 & \textbf{147.941} & 360.238 & $2.44\times$
& \textbf{14.381} & 69.893 & 79.4\% \\
& 64 & 8\,192 & \textbf{319.350} & 1\,208.060 & $3.78\times$
& \textbf{27.942} & 249.214 & 88.8\% \\
\midrule[0.8pt]
SmolLM2-360M ($15/5$)
& 64 & 1\,024 & \textbf{100.320} & 116.980 & $1.17\times$
& \textbf{7.712} & 13.981 & 44.8\% \\
& 64 & 2\,048 & \textbf{122.291} & 213.072 & $1.74\times$
& \textbf{13.298} & 38.021 & 65.0\% \\
& 64 & 4\,096 & \textbf{212.327} & 599.298 & $2.82\times$
& \textbf{24.470} & 122.664 & 80.1\% \\
\midrule[0.8pt]
Qwen2.5-1.5B ($12/2$)
& 128 & 1\,024 & \textbf{111.551} & 116.835 & $1.05\times$
& \textbf{21.780} & 25.944 & 16.0\% \\
& 128 & 2\,048 & \textbf{176.361} & 224.462 & $1.27\times$
& \textbf{34.482} & 51.476 & 33.0\% \\
& 128 & 4\,096 & \textbf{365.433} & 568.491 & $1.56\times$
& \textbf{59.886} & 128.555 & 53.4\% \\
\midrule[0.8pt]
Granite 3.1 1B-A400M ($16/8$, MoE)
& 64 & 1\,024 & \textbf{113.701} & 125.384 & $1.10\times$
& \textbf{15.517} & 20.222 & 23.3\% \\
& 64 & 2\,048 & \textbf{151.667} & 229.960 & $1.52\times$
& \textbf{23.167} & 42.935 & 46.0\% \\
& 64 & 4\,096 & \textbf{272.416} & 593.166 & $2.18\times$
& \textbf{38.729} & 118.548 & 67.3\% \\
\bottomrule
\end{tabular}
\end{table}

Power-of-two capacity sweeps extend the longest feasible sequence
length from $4$K to $16$K for Llama 3.2 1B and Qwen2.5-1.5B,
from $8$K to $64$K for SmolLM2-135M, and from $4$K to $32$K for
SmolLM2-360M and Granite 3.1 1B-A400M. These yield $4\times$
or $8\times$ capacity gains over the PyTorch math baseline.
SmolLM2 evaluations beyond its released $8{,}192$-token context use
rotary-position extrapolation and establish hardware feasibility only,
not validated long-context model quality. Table~\ref{tab:gqa-models}
reports matched-length timing and memory; capacity endpoints do not
establish equivalent model accuracy.

\section{Sliding-window attention}
\label{app:swa}

For the multi-head-attention configuration evaluated here,
sliding-window attention (SWA) is supported by restricting the inner-loop
column range: for each row tile, the streamed column tiles are restricted to those
overlapping the window. The affine-fusion structure of Proposition~\ref{prop:affine} is
unchanged ($\alpha_i$ and the five $\alpha$-free accumulators $E_i^{\circ}, B_i, \td{dO}_i^{\circ}, R_i, \Omega_i$ are accumulated
over the now-shorter rowwise supports), so the two-pass algorithm applies to the restricted attention pattern.
This experiment does not establish SWA support in the current GQA/MQA
release path.

Table~\ref{tab:swa} reports an end-to-end second-order step on GPT-2 Small with
window $w = N/4$, comparing \textsc{FlashBoB} against the PyTorch \textsc{math} reference. We
report only $N \geq 2048$: at shorter lengths, non-attention work and launch overhead
dominate the step, and the gains there are in the measurement noise floor.

\begin{table}[t]
\centering\scriptsize
\setlength{\tabcolsep}{3.5pt}
\caption{GPT-2 Small with sliding-window attention, $w = N/4$. Hardware: NVIDIA A100 80GB. Wall time reports the median $\pm$ half of the 20th-to-80th percentile spread. OOM denotes out of memory.}
\label{tab:swa}
\begin{tabular}{@{}rr cc c rr c@{}}
\toprule
& & \multicolumn{2}{c}{Wall time (ms)} & & \multicolumn{2}{c}{Peak memory (MiB)} & \\
\cmidrule(lr){3-4}\cmidrule(lr){6-7}
$N$ & $w$ & \textsc{math} & \textsc{FlashBoB} & Speedup & \textsc{math} & \textsc{FlashBoB} & Mem.\ ratio \\
\midrule
 2{,}048 &   512 & \(175.376 \pm 0.111\)  & \(136.254 \pm 0.125\)  & \(1.29\times\) &  \(8{,}366.84\) &  \(5{,}007.09\) & \(1.67\times\) \\
 4{,}096 & 1{,}024 & \(432.310 \pm 0.104\)  & \(218.646 \pm 0.113\)  & \(1.98\times\) & \(21{,}809.62\) &  \(9{,}019.05\) & \(2.42\times\) \\
 8{,}192 & 2{,}048 & \(1{,}330.003 \pm 0.089\) & \(457.850 \pm 0.061\) & \(2.90\times\) & \(66{,}801.47\) & \(17{,}040.55\) & \(3.92\times\) \\
16{,}384 & 4{,}096 & OOM & \(1{,}136.572 \pm 0.172\) & -- & OOM & \(33{,}104.94\) & -- \\
\bottomrule
\end{tabular}
\end{table}

The speedup and memory ratios grow with $N$, as in the dense setting.

\section{Head-dimension sensitivity and autotuned tile configurations}
\label{app:head-dim-sweep}

The A100 isolated-kernel sweep in Table~\ref{tab:attn-kernel}
(left) uses GPT-2-style attention heads with \(d=64\), while
Table~\ref{tab:d128-isolated} (right) evaluates \(d=128\) across
sequence lengths on B300. Since both the I/O analysis and the kernel
implementation are parametric in \(d\),
Table~\ref{tab:head-dim-sweep} additionally tests
\(d \in \{32,64,128\}\) at fixed \(N=4096\) on A100 80GB.
Increasing \(d\) makes the problem harder for a FlashAttention-style
schedule: each tile carries larger \(Q,K,V,dO\)-like blocks, the
feasible tile size decreases for a fixed on-chip memory budget,
and register/SRAM pressure increases.

Extending the optimized \textsc{FlashBoB} kernels to the larger
$d = 256$ regime is an engineering target rather than a
change to the algorithm: the two-pass schedule and the I/O analysis of
Section~\ref{sec:io} are parametric in $d$, and Theorem~\ref{thm:io}
applies unchanged. The binding constraints at $d = 256$ are
implementation-level. First, register pressure: the $B_r \times d$
row-resident tiles ($Q_i, dO_i, U_{Q,i}, O_i$) and the four
$B_r \times d$ row accumulators of Proposition~\ref{prop:affine}
($\td{dO}^{\circ}, \Omega, B, R$) compete for the same register file
inside Pass~1, and the per-thread register requirement grows with $d$ for a fixed tile
shape. Second, shared-memory pressure: the simultaneously live
interaction-shaped tiles ($P_{ij}, dP_{ij}, F_{ij}, C_{ij},
\td{P}_{ij}^{\circ}$) scale as $B_r B_c$ and force a smaller $B_c$ at
larger $d$, which raises the streamed-tile re-load count proportionally.
Third, tensor-core tiling: the native FlashAttention-3 pipelines on
Hopper assume specific $(B_r, B_c, d)$ shape tuples that do not transfer
verbatim to the BoB schedule, so a $d = 256$ port would require
re-deriving the warp specialization and asynchronous load schedules used
by FlashAttention-3 and FlashAttention-4. We view a $d = 256$ specialization as the
clearest follow-up systems work.

\begin{table}[t]
\centering\small
\caption{Head-dimension sensitivity for isolated attention BoB on
NVIDIA A100 80GB ($N=4{,}096$, batch 4, 12 heads, causal mask, BF16
inputs with FP32 reductions). Unlike the $N \le 4{,}096$ rows of
Table~\ref{tab:attn-kernel}, this sweep releases the \textsc{math}
reference outputs before timing, which accounts for the
$384$--$400$~MiB lower peak memory at $d=64$. Time and memory ratios
use \textsc{FlashBoB} at the same $d$ as the denominator.}
\label{tab:head-dim-sweep}
\setlength{\tabcolsep}{4pt}
\begin{tabular}{@{}r l r r r r@{}}
\toprule
$d$ & Method & Time (ms) & Peak mem. (MiB)
    & Time ratio & Mem. ratio \\
\midrule
\multirow{3}{*}{32}
  & \textsc{FlashBoB} & $\mathbf{14.214}$ & $\mathbf{423.0}$
  & $1.00\times$ & $1.00\times$ \\
  & \textsc{hvp-m} & $111.312$ & $14{,}005.4$
  & $7.83\times$ & $33.11\times$ \\
  & \textsc{hvp-s} & $82.393$ & $14{,}010.0$
  & $5.80\times$ & $33.12\times$ \\
\midrule
\multirow{3}{*}{64}
  & \textsc{FlashBoB} & $\mathbf{25.676}$ & $\mathbf{843.0}$
  & $1.00\times$ & $1.00\times$ \\
  & \textsc{hvp-m} & $114.726$ & $14{,}137.4$
  & $4.47\times$ & $16.77\times$ \\
  & \textsc{hvp-s} & $89.032$ & $14{,}154.0$
  & $3.47\times$ & $16.79\times$ \\
\midrule
\multirow{3}{*}{128}
  & \textsc{FlashBoB} & $\mathbf{81.694}$ & $\mathbf{1{,}683.0}$
  & $1.00\times$ & $1.00\times$ \\
  & \textsc{hvp-m} & $120.029$ & $14{,}401.4$
  & $1.47\times$ & $8.56\times$ \\
  & \textsc{hvp-s} & $103.995$ & $14{,}442.0$
  & $1.27\times$ & $8.58\times$ \\
\bottomrule
\end{tabular}
\end{table}

Table~\ref{tab:head-dim-sweep} shows that \textsc{FlashBoB}
remains the fastest and most memory-efficient method at every
tested head dimension on A100. The runtime advantage is largest
at smaller head dimensions: \textsc{FlashBoB} is
$7.83\times$ faster than \textsc{hvp-m} and $5.80\times$
faster than \textsc{hvp-s} at $d=32$. On A100 at $N=4096$, the speedup narrows at $d=128$ ($1.47\times$ over \textsc{hvp-m} and $1.27\times$ over \textsc{hvp-s}), whereas on B300 (Table~\ref{tab:d128-isolated}),
\textsc{FlashBoB} achieves $2.08$--$2.36\times$ speedups
across $N=2048$--$16{,}384$. The memory advantage remains
substantial: on A100 at $N=4096$, the HVP baselines allocate roughly $14$~GiB across the sweep, while \textsc{FlashBoB} uses $423$~MiB at $d=32$, $843$~MiB at $d=64$, and $1{,}683$~MiB at $d=128$, an $8.6\times$ reduction even at $d=128$.

\subsection{Triton autotuning search space and tile configurations by $d$}
\label{app:tile-configs}

Both passes autotune over a shared candidate family, fixing \texttt{num\_warps=4} and CTA swizzle group size $\in \{4,8\}$, and pruning candidates before compilation when either of two static hardware budgets is exceeded: (i) more than $16{,}384$ live FP32 accumulator elements per CTA ($128$ registers per thread at four warps), or (ii) more than $95$~KiB of estimated shared memory for staged operand tiles. Pass~1 holds four $B_r \times d$ vector accumulators and four staged column tiles, whereas Pass~2 holds two $B_c \times d$ accumulators and three staged row tiles. In addition, at $d \ge 128$, candidates with $B_r B_c > 512$ or Pass~2 $B_c > 32$ are pruned. Every surviving tile dimension is a multiple of $16$ and therefore eligible for Tensor Core MMA instructions when lowered via \texttt{tl.dot}. Table~\ref{tab:tile-configs} lists the surviving $(B_r, B_c)$ and pipeline stage configs for each $d$.

\begin{table}[h]
\centering\small
\caption{Autotuned Triton tile shapes $(B_r, B_c)$ and pipeline depths (\texttt{num\_stages}) surviving static register and shared-memory pruning for each head dimension $d$ (all with $4$ warps and CTA group size $\in \{4,8\}$).}
\label{tab:tile-configs}
\setlength{\tabcolsep}{5pt}
\begin{tabular}{@{}r l c l c@{}}
\toprule
 & \multicolumn{2}{c}{Pass 1 (row-major)} & \multicolumn{2}{c}{Pass 2 (column-major)} \\
\cmidrule(lr){2-3} \cmidrule(l){4-5}
$d$ & Tile shapes $(B_r, B_c)$ & Stages & Tile shapes $(B_r, B_c)$ & Stages \\
\midrule
32  & $(16,16), (32,16), (32,32),$ & $2, 3$ & $(16,16), (16,32), (32,32), (64,32)$ & $2, 3$ \\
    & $(64,16), (64,32), (128,32)$ &        &                                      &        \\
    & $(32,64)$                    & $2$    &                                      &        \\
\addlinespace[4pt]
64  & $(16,16), (32,16), (64,16)$  & $2, 3$ & $(16,16), (16,32), (32,32)$          & $2, 3$ \\
    & $(32,32), (64,32)$           & $2$    &                                      &        \\
\addlinespace[4pt]
128 & $(16,16), (32,16)$           & $2$    & $(16,16), (16,32)$                   & $2, 3$ \\
\bottomrule
\end{tabular}
\end{table}

\section{FineWeb training reproducibility}
\label{app:fineweb-training-repro}

This section reports the full configuration for the 1B-token FineWeb
training experiments in Section~\ref{sec:fineweb}: the main three-way comparison on $8{\times}$A6000 48GB (AdamW, Sophia-H with the PyTorch math reference, and Sophia-H with \textsc{FlashBoB}) and the six-run matched hybrid baseline sweep across Hessian intervals on $8{\times}$A100 80GB. Within each matched pair, the Sophia-H runs use the same model,
dataset stream, seed, optimizer hyperparameters, batch geometry,
precision, token budget, and Hessian-estimation interval. In the original A6000 comparison, the PyTorch math arm uses math attention on ordinary steps as well as Hessian steps. In the A100 hybrid comparison, both arms use fused attention on ordinary steps and differ only in the attention BoB implementation during Hessian estimation. The hybrid comparison is therefore the controlled estimate of the BoB substitution's contribution to training wall time.

We use the Hugging Face dataset identifier
\texttt{VisionTheta/fineweb-1B}, which is a 1B-token prepared subset
derived from the FineWeb corpus of~\citet{penedo2024fineweb}.
The license and source attribution therefore follow the upstream
FineWeb dataset, while the exact artifact used for these experiments is
the \texttt{VisionTheta/fineweb-1B} Hugging Face dataset. We report the
dataset identifier explicitly to make the data order and preprocessing
reproducible.

All FineWeb curves in Figure~\ref{fig:fineweb_parity} are single runs
with seed \(17\). The AdamW curve is included as a first-order optimizer
reference, not as a backend-equivalence experiment. To avoid comparing
against a weak AdamW baseline, we selected the plotted AdamW curve from
a small learning-rate sweep, described below.

\begin{table}[t]
\centering\small
\caption{Shared data, batching, and evaluation configuration for the
FineWeb GPT-2 Small training runs.}
\label{tab:fineweb-shared-config}
\begin{tabular}{@{}ll@{}}
\toprule
Setting & Value \\
\midrule
Dataset & \texttt{VisionTheta/fineweb-1B} \\
Dataset name & \texttt{default} \\
Dataset mode & streaming \\
Tokenizer & \texttt{gpt2} \\
Sequence length & \(2048\) \\
Training token budget & \(1{,}000{,}000{,}000\) tokens \\
Maximum optimizer steps & \(1908\) \\
Seed & \(17\) \\
Precision & bfloat16 \\
Number of GPUs & \(8\) ($8\times$A6000 48GB main; $8\times$A100 80GB hybrid baseline) \\
Per-GPU batch size & \(2\) sequences \\
Gradient accumulation steps & \(16\) \\
Effective global batch size & \(256\) sequences \\
Local tokens per optimizer step & \(65{,}536\) \\
Global tokens per optimizer step & \(524{,}288\) \\
Evaluation interval & every \(250\) optimizer steps \\
Evaluation batches & \(32\) \\
Evaluation batch size & \(2\) per GPU \\
Prefetch factor & \(2\) \\
\bottomrule
\end{tabular}
\end{table}

\begin{table}[t]
\centering\small
\caption{GPT-2 Small model configuration used for the FineWeb training
runs.}
\label{tab:fineweb-gpt2-config}
\begin{tabular}{@{}ll@{}}
\toprule
Setting & Value \\
\midrule
Architecture & GPT-2 Small \\
Layers & \(12\) \\
Attention heads & \(12\) \\
Embedding dimension & \(768\) \\
Context length & \(2048\) \\
Vocabulary size & \(50{,}257\) \\
Bias terms & enabled \\
Total parameters & \(125{,}226{,}240\) \\
Attention parameters & \(28{,}348{,}416\) \\
Non-attention parameters & \(96{,}877{,}824\) \\
Embedding parameters & \(40{,}170{,}240\) \\
Attention parameter fraction & \(22.64\%\) \\
\bottomrule
\end{tabular}
\end{table}

\begin{table}[t]
\centering\small
\caption{Sophia-H optimizer and Hessian-estimation configuration. These
settings are identical for the PyTorch math-backend, hybrid baseline, and
\textsc{FlashBoB} Sophia-H runs.}
\label{tab:fineweb-sophiah-optimizer}
\begin{tabular}{@{}ll@{}}
\toprule
Setting & Value \\
\midrule
Optimizer & Sophia-H \\
Learning rate & \(6.0{\times}10^{-4}\) \\
Minimum learning rate & \(3.0{\times}10^{-5}\) \\
\(\beta_1\) & \(0.96\) \\
\(\beta_2\) & \(0.99\) \\
\(\epsilon\) & \(10^{-12}\) \\
\(\gamma\) & \(0.01\) \\
Weight decay & \(0.2\) \\
Gradient clipping & \(1.0\) \\
Hessian-estimation interval & every \(10\) steps (main); every \(\{10, 5, 2\}\) steps (hybrid sweep) \\
Hutchinson samples & \(1\) \\
Hutchinson batch size & \(32\) \\
Warmup steps & \(2000\) \\
Maximum optimizer steps & \(1908\) \\
\bottomrule
\end{tabular}
\end{table}

The Sophia-H learning-rate schedule is intentionally kept as
used in the best observed Sophia-H configuration from our tuning runs.
Although the configured warmup length of \(2000\) steps exceeds the
\(1908\)-step 1B-token run, this setting gave the best observed
Sophia-H behavior and is held fixed across all Sophia-H arms.

\begin{table}[t]
\centering\small
\caption{Sophia-H backend arms compared in Section~\ref{sec:fineweb}.}
\label{tab:fineweb-sophiah-backends}
\begin{tabular}{@{}llll@{}}
\toprule
Configuration & Ordinary steps (fwd/bwd) & Hessian steps (BoB) & Hardware \\
\midrule
Sophia-H (math ref.) & PyTorch \texttt{math} SDPA & PyTorch \texttt{math} autograd & $8\times$A6000 48GB \\
Sophia-H (hybrid) & FlashAttention SDPA & PyTorch \texttt{math} autograd & $8\times$A100 80GB \\
Sophia-H (\textsc{FlashBoB}) & FlashAttention SDPA & \textsc{FlashBoB} two-pass kernel & $8\times$A6000 / $8\times$A100 \\
\bottomrule
\end{tabular}
\end{table}

\paragraph{AdamW reference.}
Figure~\ref{fig:fineweb_parity} also includes an AdamW reference curve.
This run is not used to test backend equivalence; it is included as a
familiar first-order optimizer reference for the same 1B-token FineWeb
setup. To avoid comparing against a poorly tuned AdamW baseline, we ran
a small learning-rate sweep over peak learning rates
\(\{3.0{\times}10^{-4}, 4.5{\times}10^{-4}, 6.0{\times}10^{-4},
8.0{\times}10^{-4}, 1.0{\times}10^{-3}\}\). Each sweep point used AdamW
with \(\beta_1=0.9\), \(\beta_2=0.95\), \(\epsilon=10^{-8}\), weight
decay \(0.1\), gradient clipping \(1.0\), linear warmup followed by
cosine decay, and minimum learning rate \(0.1\) times the peak learning
rate. The AdamW curve in Figure~\ref{fig:fineweb_parity} uses the
best-performing configuration from this sweep.

\begin{table}[t]
\centering\small
\caption{AdamW learning-rate sweep used to select the reference AdamW
curve in Figure~\ref{fig:fineweb_parity}. All runs use AdamW with
\(\beta_1=0.9\), \(\beta_2=0.95\), \(\epsilon=10^{-8}\), weight decay
\(0.1\), gradient clipping \(1.0\), linear warmup followed by cosine
decay, and minimum learning rate \(0.1\) times the peak learning rate.}
\label{tab:fineweb-adamw-lr-sweep}
\begin{tabular}{@{}rlll@{}}
\toprule
Run & Peak LR & Min LR & Role \\
\midrule
1 & \(3.0{\times}10^{-4}\) & \(3.0{\times}10^{-5}\) & Conservative baseline \\
2 & \(4.5{\times}10^{-4}\) & \(4.5{\times}10^{-5}\) & Stable intermediate setting \\
3 & \(6.0{\times}10^{-4}\) & \(6.0{\times}10^{-5}\) & Canonical anchor \\
4 & \(8.0{\times}10^{-4}\) & \(8.0{\times}10^{-5}\) & Aggressive setting \\
5 & \(1.0{\times}10^{-3}\) & \(1.0{\times}10^{-4}\) & Highest-LR stress test \\
\bottomrule
\end{tabular}
\end{table}

\begin{table}[t]
\centering\small
\caption{Selected AdamW reference-run configuration used for the
FineWeb plot. This was the best-performing run from the sweep in
Table~\ref{tab:fineweb-adamw-lr-sweep}.}
\label{tab:fineweb-adamw-reference}
\begin{tabular}{@{}ll@{}}
\toprule
Setting & Value \\
\midrule
Optimizer & AdamW \\
Attention backend & FlashAttention backend \\
Seed & \(17\) \\
Learning-rate schedule & linear warmup followed by cosine decay \\
Peak learning rate & \(1.0{\times}10^{-3}\) \\
Minimum learning rate & \(1.0{\times}10^{-4}\) \\
Warmup steps & \(150\) \\
Maximum optimizer steps & \(1908\) \\
\(\beta_1\) & \(0.9\) \\
\(\beta_2\) & \(0.95\) \\
\(\epsilon\) & \(10^{-8}\) \\
Weight decay & \(0.1\) \\
Gradient clipping & \(1.0\) \\
Precision & bfloat16 \\
Compilation & disabled \\
Per-GPU batch size & \(2\) sequences \\
Gradient accumulation steps & \(16\) \\
Effective global batch size & \(256\) sequences \\
Global tokens per optimizer step & \(524{,}288\) \\
Training token budget & \(1{,}000{,}000{,}000\) tokens \\
Evaluation interval & every \(250\) optimizer steps \\
Evaluation batches & \(32\) \\
\bottomrule
\end{tabular}
\end{table}

\section{Detailed comparison with FlashBack}
\label{app:flashback-comparison}

This appendix expands on the main-text comparison in
Table~\ref{tab:flashback-main}.
FlashBack~\citep{engstrom2024flashback} fuses the entire BoB
into a single GPU kernel parallelized over rows, runs two
sequential inner loops over column tiles per row block, and
accumulates the column-owned outputs $(\td{K},\td{V})$
through global atomic adds. \textsc{FlashBoB} instead splits
the work into a row-major Pass~1 that finalizes
$(\td{Q},\td{dO},\alpha,E)$ via affine fusion
(Proposition~\ref{prop:affine}), and a column-major Pass~2
that finalizes $(\td{K},\td{V})$ as a natural column reduction.
Both schedules recompute the probability tile $P$ twice
across the BoB; they differ in how that recomputation work
is distributed, how column-owned outputs are accumulated,
and how much state is simultaneously live on chip.

\begin{center}
\small
\begin{tabular}{@{}lll@{}}
\toprule
Aspect & FlashBack & \textsc{FlashBoB} \\
\midrule
Kernel launches & 1 (single fused) & 2 (row-major + column-major) \\
$\td{K}, \td{V}$ writes & \texttt{atomic\_add} & natural column reduction (no atomics) \\
Affine fusion of $\alpha$ & no & yes \\
Row-side work in first sweep & scalar summaries only & all row outputs $(\td{Q}, \td{dO}, E)$ \\
On-chip row scalars & kept on chip (no HBM round trip) & written and re-read via $(\alpha, E)$ \\
\bottomrule
\end{tabular}
\end{center}

\textbf{Configuration sweeps and measurement details.}
For both methods, we autotune over candidate kernel configurations and report the lowest steady-state latency at each sequence length $N$, excluding compilation and autotuning time. On NVIDIA B300 at $d=128$, we swept FlashBack over block sizes in $\{16,32\}$, pipeline stages in $\{1,2,3\}$, and warp counts in $\{4,8\}$; block size $32$ with $1$ stage and $4$ warps is fastest across all tested lengths. The B300 sweep ($d=128$) evaluates all powers of two from $N=256$ to $65{,}536$, whereas the A6000 sweep ($d=64$) evaluates the same range except $N=8{,}192$ (Table~\ref{tab:flashback-main} omits $N=8{,}192$ to align the two columns, and on B300 at $N=8{,}192$ FlashBack requires $619.00$~ms versus $68.23$~ms for \textsc{FlashBoB}). We did not run A6000 comparisons at $N=131{,}072$ or $N=262{,}144$; the longer-sequence \textsc{FlashBoB} results in Table~\ref{tab:attn-kernel-extended} were measured on A100 80GB.

We report shared-hardware wall time rather than FlashBack peak memory
in Table~\ref{tab:flashback-main} because the public FlashBack
implementation uses JAX/XLA, whose allocator reservation is not
directly comparable to PyTorch's
\texttt{torch.cuda.max\_memory\_allocated}.

\textbf{Matched MQA comparison.}
Because FlashBack supports native multi-query attention ($H_{KV}=1$), we also compared both methods directly at the saved-state second-order attention boundary on one NVIDIA B300 using batch size $1$, $H_Q/H_{KV}=8/1$, and matched BF16 inputs and outputs across $N=2{,}048$--$131{,}072$: \textsc{FlashBoB} is $3.96\times$--$5.04\times$ faster than FlashBack at $d=64$ and $5.87\times$--$8.90\times$ faster at $d=128$.

The complete A6000 and B300 timing sweeps appear in
Table~\ref{tab:flashback-main}.

\section{Broader long-context systems context}
\label{app:broader-systems-context}

Long-context systems address complementary bottlenecks at different
levels of execution. Sequence-parallel training partitions attention
across GPUs and coordinates communication between devices~\citep{jacobs2023deepspeedulysses,liu2024ringattention,li2024distflashattn,fang2024usp}.
Inference systems instead optimize KV-cache management and attention
execution during serving~\citep{kwon2023pagedattention,ye2025flashinfer}.
Other long-sequence kernels, such as FlashFFTConv~\citep{fu2024flashfftconv},
restructure distinct sequence operators to exploit the GPU memory
hierarchy.

\textsc{FlashBoB} addresses a different operation: exact
backward-over-backward through softmax attention on a single GPU.
Its two-pass schedule avoids $N\times N$ HBM intermediates and
accumulates row- and column-owned outputs without global atomics.
Because the kernel executes independently on each local attention head, standard data parallelism, tensor parallelism that shards attention heads, pipeline parallelism, and head-sharded sequence parallelism with all-to-all communication (such as DeepSpeed-Ulysses~\citep{jacobs2023deepspeedulysses}) leave \textsc{FlashBoB} unchanged.
Ring and context parallelism~\citep{liu2024ringattention}, which shard the sequence dimension across $P$ ranks, require two scalar ring all-reduces across the sequence group between Pass~1 and Pass~2 to sum the partial row scalars $(\alpha_i, E_i)$ ($4N/P$ FP32 scalars per rank, which can be overlapped with streaming the key-value blocks) before Pass~2 accumulates $(\td{K},\td{V})$; we have not implemented this distributed ring schedule or inference-oriented paged KV layouts.
Finally, PyTorch's Fully Sharded Data Parallel (\texttt{FSDP}) currently lacks native double-backward hooks for parameter resharding during a second backward pass, so our multi-GPU second-order training runs in Section~\ref{sec:fineweb} use Distributed Data Parallel (\texttt{DDP}).

\newpage
\input{checklist.tex}

\end{document}

%% file: checklist.tex
\section*{NeurIPS Paper Checklist}

\begin{enumerate}

\item {\bf Claims}
    \item[] Question: Do the main claims made in the abstract and introduction accurately reflect the paper's contributions and scope?
    \item[] Answer: \answerYes{} %
    \item[] Justification: Yes. All claims made in the abstract and introduction were taken from experimental results, which formed the basis of our contributions.
    \item[] Guidelines:
    \begin{itemize}
        \item The answer \answerNA{} means that the abstract and introduction do not include the claims made in the paper.
        \item The abstract and/or introduction should clearly state the claims made, including the contributions made in the paper and important assumptions and limitations. A \answerNo{} or \answerNA{} answer to this question will not be perceived well by the reviewers. 
        \item The claims made should match theoretical and experimental results, and reflect how much the results can be expected to generalize to other settings. 
        \item It is fine to include aspirational goals as motivation as long as it is clear that these goals are not attained by the paper. 
    \end{itemize}

\item {\bf Limitations}
    \item[] Question: Does the paper discuss the limitations of the work performed by the authors?
    \item[] Answer: \answerYes{} %
    \item[] Justification: We make sure to state all the current limitations of the project, so that all stakeholders can be privy of them. We believe that some of these limitations are exciting future work, so we ensured that we stated them exactly. 
    \item[] Guidelines:
    \begin{itemize}
        \item The answer \answerNA{} means that the paper has no limitation while the answer \answerNo{} means that the paper has limitations, but those are not discussed in the paper. 
        \item The authors are encouraged to create a separate ``Limitations'' section in their paper.
        \item The paper should point out any strong assumptions and how robust the results are to violations of these assumptions (e.g., independence assumptions, noiseless settings, model well-specification, asymptotic approximations only holding locally). The authors should reflect on how these assumptions might be violated in practice and what the implications would be.
        \item The authors should reflect on the scope of the claims made, e.g., if the approach was only tested on a few datasets or with a few runs. In general, empirical results often depend on implicit assumptions, which should be articulated.
        \item The authors should reflect on the factors that influence the performance of the approach. For example, a facial recognition algorithm may perform poorly when image resolution is low or images are taken in low lighting. Or a speech-to-text system might not be used reliably to provide closed captions for online lectures because it fails to handle technical jargon.
        \item The authors should discuss the computational efficiency of the proposed algorithms and how they scale with dataset size.
        \item If applicable, the authors should discuss possible limitations of their approach to address problems of privacy and fairness.
        \item While the authors might fear that complete honesty about limitations might be used by reviewers as grounds for rejection, a worse outcome might be that reviewers discover limitations that aren't acknowledged in the paper. The authors should use their best judgment and recognize that individual actions in favor of transparency play an important role in developing norms that preserve the integrity of the community. Reviewers will be specifically instructed to not penalize honesty concerning limitations.
    \end{itemize}

\item {\bf Theory assumptions and proofs}
    \item[] Question: For each theoretical result, does the paper provide the full set of assumptions and a complete (and correct) proof?
    \item[] Answer: \answerYes{} %
    \item[] Justification: We state all assumptions and ensure that we have a complete proof to accompany our theoretical results. We also ensure that the proofs are easy to read and self-contained.
    \item[] Guidelines:
    \begin{itemize}
        \item The answer \answerNA{} means that the paper does not include theoretical results. 
        \item All the theorems, formulas, and proofs in the paper should be numbered and cross-referenced.
        \item All assumptions should be clearly stated or referenced in the statement of any theorems.
        \item The proofs can either appear in the main paper or the supplemental material, but if they appear in the supplemental material, the authors are encouraged to provide a short proof sketch to provide intuition. 
        \item Inversely, any informal proof provided in the core of the paper should be complemented by formal proofs provided in appendix or supplemental material.
        \item Theorems and Lemmas that the proof relies upon should be properly referenced. 
    \end{itemize}

    \item {\bf Experimental result reproducibility}
    \item[] Question: Does the paper fully disclose all the information needed to reproduce the main experimental results of the paper to the extent that it affects the main claims and/or conclusions of the paper (regardless of whether the code and data are provided or not)?
    \item[] Answer: \answerYes{} %
    \item[] Justification: We specify all that is needed to reproduce all our main claims, including a very detailed pseudocode of our Triton kernel. We also specify the exact setup and software configuration needed to reproduce our results.  
    \item[] Guidelines:
    \begin{itemize}
        \item The answer \answerNA{} means that the paper does not include experiments.
        \item If the paper includes experiments, a \answerNo{} answer to this question will not be perceived well by the reviewers: Making the paper reproducible is important, regardless of whether the code and data are provided or not.
        \item If the contribution is a dataset and\slash or model, the authors should describe the steps taken to make their results reproducible or verifiable. 
        \item Depending on the contribution, reproducibility can be accomplished in various ways. For example, if the contribution is a novel architecture, describing the architecture fully might suffice, or if the contribution is a specific model and empirical evaluation, it may be necessary to either make it possible for others to replicate the model with the same dataset, or provide access to the model. In general. releasing code and data is often one good way to accomplish this, but reproducibility can also be provided via detailed instructions for how to replicate the results, access to a hosted model (e.g., in the case of a large language model), releasing of a model checkpoint, or other means that are appropriate to the research performed.
        \item While NeurIPS does not require releasing code, the conference does require all submissions to provide some reasonable avenue for reproducibility, which may depend on the nature of the contribution. For example
        \begin{enumerate}
            \item If the contribution is primarily a new algorithm, the paper should make it clear how to reproduce that algorithm.
            \item If the contribution is primarily a new model architecture, the paper should describe the architecture clearly and fully.
            \item If the contribution is a new model (e.g., a large language model), then there should either be a way to access this model for reproducing the results or a way to reproduce the model (e.g., with an open-source dataset or instructions for how to construct the dataset).
            \item We recognize that reproducibility may be tricky in some cases, in which case authors are welcome to describe the particular way they provide for reproducibility. In the case of closed-source models, it may be that access to the model is limited in some way (e.g., to registered users), but it should be possible for other researchers to have some path to reproducing or verifying the results.
        \end{enumerate}
    \end{itemize}

\item {\bf Open access to data and code}
    \item[] Question: Does the paper provide open access to the data and code, with sufficient instructions to faithfully reproduce the main experimental results, as described in supplemental material?
    \item[] Answer: \answerYes{} %
    \item[] Justification: Our training dataset is publicly available, and we release code artifacts that detail, verbatim, how to reproduce all experiments.
    \item[] Guidelines:
    \begin{itemize}
        \item The answer \answerNA{} means that paper does not include experiments requiring code.
        \item Please see the NeurIPS code and data submission guidelines (\url{https://neurips.cc/public/guides/CodeSubmissionPolicy}) for more details.
        \item While we encourage the release of code and data, we understand that this might not be possible, so \answerNo{} is an acceptable answer. Papers cannot be rejected simply for not including code, unless this is central to the contribution (e.g., for a new open-source benchmark).
        \item The instructions should contain the exact command and environment needed to run to reproduce the results. See the NeurIPS code and data submission guidelines (\url{https://neurips.cc/public/guides/CodeSubmissionPolicy}) for more details.
        \item The authors should provide instructions on data access and preparation, including how to access the raw data, preprocessed data, intermediate data, and generated data, etc.
        \item The authors should provide scripts to reproduce all experimental results for the new proposed method and baselines. If only a subset of experiments are reproducible, they should state which ones are omitted from the script and why.
        \item At submission time, to preserve anonymity, the authors should release anonymized versions (if applicable).
        \item Providing as much information as possible in supplemental material (appended to the paper) is recommended, but including URLs to data and code is permitted.
    \end{itemize}

\item {\bf Experimental setting/details}
    \item[] Question: Does the paper specify all the training and test details (e.g., data splits, hyperparameters, how they were chosen, type of optimizer) necessary to understand the results?
    \item[] Answer: \answerYes{} %
    \item[] Justification: We detail all that is needed to replicate our experiments, in addition to the actual experimentation code as well.
    \item[] Guidelines:
    \begin{itemize}
        \item The answer \answerNA{} means that the paper does not include experiments.
        \item The experimental setting should be presented in the core of the paper to a level of detail that is necessary to appreciate the results and make sense of them.
        \item The full details can be provided either with the code, in appendix, or as supplemental material.
    \end{itemize}

\item {\bf Experiment statistical significance}
    \item[] Question: Does the paper report error bars suitably and correctly defined or other appropriate information about the statistical significance of the experiments?
    \item[] Answer: \answerNo{} %
    \item[] Justification: We do not report error bars and/or statistical significance of our experiments, due to computational cost.
    \item[] Guidelines:
    \begin{itemize}
        \item The answer \answerNA{} means that the paper does not include experiments.
        \item The authors should answer \answerYes{} if the results are accompanied by error bars, confidence intervals, or statistical significance tests, at least for the experiments that support the main claims of the paper.
        \item The factors of variability that the error bars are capturing should be clearly stated (for example, train/test split, initialization, random drawing of some parameter, or overall run with given experimental conditions).
        \item The method for calculating the error bars should be explained (closed form formula, call to a library function, bootstrap, etc.)
        \item The assumptions made should be given (e.g., Normally distributed errors).
        \item It should be clear whether the error bar is the standard deviation or the standard error of the mean.
        \item It is OK to report 1-sigma error bars, but one should state it. The authors should preferably report a 2-sigma error bar than state that they have a 96\% CI, if the hypothesis of Normality of errors is not verified.
        \item For asymmetric distributions, the authors should be careful not to show in tables or figures symmetric error bars that would yield results that are out of range (e.g., negative error rates).
        \item If error bars are reported in tables or plots, the authors should explain in the text how they were calculated and reference the corresponding figures or tables in the text.
    \end{itemize}

\item {\bf Experiments compute resources}
    \item[] Question: For each experiment, does the paper provide sufficient information on the computer resources (type of compute workers, memory, time of execution) needed to reproduce the experiments?
    \item[] Answer: \answerYes{} %
    \item[] Justification: We stated, verbatim, all the GPUs we used, along with the exact software stack needed to replicate all our experiments. In addition, we provide code to make this process easier for all stakeholders.
    \item[] Guidelines:
    \begin{itemize}
        \item The answer \answerNA{} means that the paper does not include experiments.
        \item The paper should indicate the type of compute workers CPU or GPU, internal cluster, or cloud provider, including relevant memory and storage.
        \item The paper should provide the amount of compute required for each of the individual experimental runs as well as estimate the total compute. 
        \item The paper should disclose whether the full research project required more compute than the experiments reported in the paper (e.g., preliminary or failed experiments that didn't make it into the paper). 
    \end{itemize}
    
\item {\bf Code of ethics}
    \item[] Question: Does the research conducted in the paper conform, in every respect, with the NeurIPS Code of Ethics \url{https://neurips.cc/public/EthicsGuidelines}?
    \item[] Answer: \answerYes{} %
    \item[] Justification: To the best of our knowledge, the work uses standard public datasets and does not involve human subjects or other procedures that would conflict with NeurIPS Code of Ethics.
    \item[] Guidelines:
    \begin{itemize}
        \item The answer \answerNA{} means that the authors have not reviewed the NeurIPS Code of Ethics.
        \item If the authors answer \answerNo, they should explain the special circumstances that require a deviation from the Code of Ethics.
        \item The authors should make sure to preserve anonymity (e.g., if there is a special consideration due to laws or regulations in their jurisdiction).
    \end{itemize}

\item {\bf Broader impacts}
    \item[] Question: Does the paper discuss both potential positive societal impacts and negative societal impacts of the work performed?
    \item[] Answer: \answerNo{} %
    \item[] Justification: We do not discuss broader positive and/or negative societal impacts, since this work primarily focuses on optimizing throughput for neural network training.
    \item[] Guidelines:
    \begin{itemize}
        \item The answer \answerNA{} means that there is no societal impact of the work performed.
        \item If the authors answer \answerNA{} or \answerNo, they should explain why their work has no societal impact or why the paper does not address societal impact.
        \item Examples of negative societal impacts include potential malicious or unintended uses (e.g., disinformation, generating fake profiles, surveillance), fairness considerations (e.g., deployment of technologies that could make decisions that unfairly impact specific groups), privacy considerations, and security considerations.
        \item The conference expects that many papers will be foundational research and not tied to particular applications, let alone deployments. However, if there is a direct path to any negative applications, the authors should point it out. For example, it is legitimate to point out that an improvement in the quality of generative models could be used to generate Deepfakes for disinformation. On the other hand, it is not needed to point out that a generic algorithm for optimizing neural networks could enable people to train models that generate Deepfakes faster.
        \item The authors should consider possible harms that could arise when the technology is being used as intended and functioning correctly, harms that could arise when the technology is being used as intended but gives incorrect results, and harms following from (intentional or unintentional) misuse of the technology.
        \item If there are negative societal impacts, the authors could also discuss possible mitigation strategies (e.g., gated release of models, providing defenses in addition to attacks, mechanisms for monitoring misuse, mechanisms to monitor how a system learns from feedback over time, improving the efficiency and accessibility of ML).
    \end{itemize}
    
\item {\bf Safeguards}
    \item[] Question: Does the paper describe safeguards that have been put in place for responsible release of data or models that have a high risk for misuse (e.g., pre-trained language models, image generators, or scraped datasets)?
    \item[] Answer: \answerNA{} %
    \item[] Justification: We do not release new high-risk pretrained models, scraped datasets, or other assets that would require a responsible-release safeguard discussion.
    \item[] Guidelines:
    \begin{itemize}
        \item The answer \answerNA{} means that the paper poses no such risks.
        \item Released models that have a high risk for misuse or dual-use should be released with necessary safeguards to allow for controlled use of the model, for example by requiring that users adhere to usage guidelines or restrictions to access the model or implementing safety filters. 
        \item Datasets that have been scraped from the Internet could pose safety risks. The authors should describe how they avoided releasing unsafe images.
        \item We recognize that providing effective safeguards is challenging, and many papers do not require this, but we encourage authors to take this into account and make a best faith effort.
    \end{itemize}

\item {\bf Licenses for existing assets}
    \item[] Question: Are the creators or original owners of assets (e.g., code, data, models), used in the paper, properly credited and are the license and terms of use explicitly mentioned and properly respected?
    \item[] Answer: \answerYes{} %
    \item[] Justification: We cite all existing works that we use in our experiments, in order to preserve fairness. 
    \item[] Guidelines:
    \begin{itemize}
        \item The answer \answerNA{} means that the paper does not use existing assets.
        \item The authors should cite the original paper that produced the code package or dataset.
        \item The authors should state which version of the asset is used and, if possible, include a URL.
        \item The name of the license (e.g., CC-BY 4.0) should be included for each asset.
        \item For scraped data from a particular source (e.g., website), the copyright and terms of service of that source should be provided.
        \item If assets are released, the license, copyright information, and terms of use in the package should be provided. For popular datasets, \url{paperswithcode.com/datasets} has curated licenses for some datasets. Their licensing guide can help determine the license of a dataset.
        \item For existing datasets that are re-packaged, both the original license and the license of the derived asset (if it has changed) should be provided.
        \item If this information is not available online, the authors are encouraged to reach out to the asset's creators.
    \end{itemize}

\item {\bf New assets}
    \item[] Question: Are new assets introduced in the paper well documented and is the documentation provided alongside the assets?
    \item[] Answer: \answerYes{} %
    \item[] Justification: We will release fully documented code, open-sourcing our Triton kernels and experimentation suite. We ensured that the code is documented, making it easy for all stakeholders to interact with the codebase.
    \item[] Guidelines:
    \begin{itemize}
        \item The answer \answerNA{} means that the paper does not release new assets.
        \item Researchers should communicate the details of the dataset\slash code\slash model as part of their submissions via structured templates. This includes details about training, license, limitations, etc. 
        \item The paper should discuss whether and how consent was obtained from people whose asset is used.
        \item At submission time, remember to anonymize your assets (if applicable). You can either create an anonymized URL or include an anonymized zip file.
    \end{itemize}

\item {\bf Crowdsourcing and research with human subjects}
    \item[] Question: For crowdsourcing experiments and research with human subjects, does the paper include the full text of instructions given to participants and screenshots, if applicable, as well as details about compensation (if any)? 
    \item[] Answer: \answerNA{} %
    \item[] Justification: Our work neither involves crowdsourcing nor human-subject research.
    \item[] Guidelines:
    \begin{itemize}
        \item The answer \answerNA{} means that the paper does not involve crowdsourcing nor research with human subjects.
        \item Including this information in the supplemental material is fine, but if the main contribution of the paper involves human subjects, then as much detail as possible should be included in the main paper. 
        \item According to the NeurIPS Code of Ethics, workers involved in data collection, curation, or other labor should be paid at least the minimum wage in the country of the data collector. 
    \end{itemize}

\item {\bf Institutional review board (IRB) approvals or equivalent for research with human subjects}
    \item[] Question: Does the paper describe potential risks incurred by study participants, whether such risks were disclosed to the subjects, and whether Institutional Review Board (IRB) approvals (or an equivalent approval/review based on the requirements of your country or institution) were obtained?
    \item[] Answer: \answerNA{} %
    \item[] Justification: Our work does not involve crowdsourcing or human-subject research, so IRB approval is not applicable.
    \item[] Guidelines:
    \begin{itemize}
        \item The answer \answerNA{} means that the paper does not involve crowdsourcing nor research with human subjects.
        \item Depending on the country in which research is conducted, IRB approval (or equivalent) may be required for any human subjects research. If you obtained IRB approval, you should clearly state this in the paper. 
        \item We recognize that the procedures for this may vary significantly between institutions and locations, and we expect authors to adhere to the NeurIPS Code of Ethics and the guidelines for their institution. 
        \item For initial submissions, do not include any information that would break anonymity (if applicable), such as the institution conducting the review.
    \end{itemize}

\item {\bf Declaration of LLM usage}
    \item[] Question: Does the paper describe the usage of LLMs if it is an important, original, or non-standard component of the core methods in this research? Note that if the LLM is used only for writing, editing, or formatting purposes and does \emph{not} impact the core methodology, scientific rigor, or originality of the research, declaration is not required.
    \item[] Answer: \answerNA{} %
    \item[] Justification: LLMs are not an important or non-standard component of our proposed method, or experimental methodology. We only use LLMs for writing, editing, and/or formatting purposes.
    \item[] Guidelines:
    \begin{itemize}
        \item The answer \answerNA{} means that the core method development in this research does not involve LLMs as any important, original, or non-standard components.
        \item Please refer to our LLM policy in the NeurIPS handbook for what should or should not be described.
    \end{itemize}

\end{enumerate}